\documentclass[12pt]{article}
\usepackage[margin=1in]{geometry}
\usepackage{times}

\usepackage{amsmath}
\usepackage{amssymb}
\usepackage{latexsym}
\usepackage{verbatim}
\usepackage[final]{graphicx}
\usepackage{bbm}

\usepackage{amsthm}

\newtheorem{theorem}{Theorem}
\newtheorem{proposition}{Proposition}
\newtheorem{lemma}{Lemma}
\newtheorem{corollary}{Corollary}
\newtheorem{remark}{Remark}

\newtheorem{conjecture}{Conjecture}
\newtheorem{definition}{Definition}

{\begin{list}               
    {$\bullet$ \hfill}{
        \setlength{\leftmargin}{\parindent}
        \setlength{\parsep}{0.04\baselineskip}
        \setlength{\itemsep}{0.5\parsep}
        \setlength{\labelwidth}{\leftmargin}
        \setlength{\labelsep}{0em}}
    }
{\end{list}}

\providecommand{\cref}[1]{Chapter~\ref{chap:#1}}
\providecommand{\sref}[1]{Section~\ref{sec:#1}}

\providecommand{\thref}[1]{Theorem~\ref{thm:#1}}
\providecommand{\lref}[1]{Lemma~\ref{lem:#1}}
\providecommand{\corref}[1]{Corollary~\ref{cor:#1}}

\providecommand{\R}{\ensuremath{\mathbb{R}}}

\providecommand{\norm}[1]{\lVert#1\rVert}

\providecommand{\inprod}[1]{\langle#1\rangle}
\providecommand{\set}[1]{\left\{#1\right\}}

\providecommand{\bydef}{\coloneqq}

\newcommand{\E}{\mathbb{E}}
\renewcommand{\P}{\mathbb{P}}

\newcommand{\Eb}[2]{\E_{#1}\left[#2\right]}

\newcommand{\Op}{\mathcal{O}_\prec}

\DeclareMathOperator{\tr}{tr}
\DeclareMathOperator{\diag}{diag}

\DeclareMathOperator*{\argmin}{arg\,min}

\providecommand{\eqd}{\overset{(d)}{=}}

\usepackage{commath}
\usepackage{mathtools}
\usepackage{bm}
\usepackage{enumitem}
\mathtoolsset{showonlyrefs}
\usepackage[dvipsnames]{xcolor}
\usepackage{subcaption}
\usepackage{placeins}

\newcommand{\Wh}{\widehat W}
\newcommand{\muh}{\widehat \mu}

\newcommand{\WLOO}[1][h]{\widehat W^{\setminus #1}}
\newcommand{\muLOO}[1][h]{\widehat \mu^{\setminus #1}}

\newcommand{\Wfresh}{\widehat{W}_\text{fresh}}

\DeclareMathOperator{\var}{Var}

\DeclareMathOperator{\vecop}{vec}

\DeclareMathOperator{\TAM}{TAM}

\DeclareMathOperator{\Prox}{Prox}

\usepackage{hyperref}
\hypersetup{
    colorlinks=true,
    linkcolor=MidnightBlue,
    citecolor=MidnightBlue,
    urlcolor=MidnightBlue,
    pdftitle={Sharp Capacity Thresholds in Linear Associative Memory: From Top-1 Retrieval to Tail-Average Learning},
    pdfauthor={Nicholas Barnfield, Juno Kim, Eshaan Nichani, Jason D. Lee, and Yue M. Lu}
}

\usepackage{natbib}

\title{\Large Sharp Capacity Thresholds in Linear Associative Memory:\\
From Top-1 Retrieval to Tail-Average Learning}
\author{
Nicholas Barnfield$^{1}$ \quad
Juno Kim$^{2}$ \quad
Eshaan Nichani$^{3}$\\[0.25em]
Jason D. Lee$^{2}$ \quad
Yue M. Lu$^{1}$\thanks{Correspondence: \texttt{yuelu@seas.harvard.edu}.}\\[0.5em]
{\small $^{1}$Harvard University \quad
$^{2}$University of California, Berkeley \quad
$^{3}$Princeton University}
}
\date{August 19, 2026}

\ifdefined\usebigfont

\usepackage{times}
\usepackage[fontsize=13pt]{scrextend}
\makeatletter
\@ifpackageloaded{geometry}{\AtBeginDocument{\newgeometry{letterpaper,left=1.56in,right=1.56in,top=1.71in,bottom=1.77in}}}{\usepackage[letterpaper,left=1.56in,right=1.56in,top=1.71in,bottom=1.77in]{geometry}}
\AtBeginDocument{\newgeometry{letterpaper,left=1.56in,right=1.56in,top=1.71in,bottom=1.77in}}
\fi
\begin{document}

\maketitle

\begin{abstract}
How many key-value associations can a $d\times d$ linear
memory store? The answer depends not only on the $d^2$ degrees of freedom in
the memory matrix, but also on the retrieval criterion. Under isotropic
Gaussian embeddings, we prove a sharp threshold for top-1 retrieval, where
every signal must beat its largest distractor: the critical value of
$d^2/(n\log n)$ is $2$. Above the threshold, we explicitly construct a
linear memory that
retrieves all $n$ associations with high probability; below it, no
data-dependent linear memory can do so. The $\log n$ factor is therefore
the unavoidable extreme-value cost of winner-take-all decoding.

Without the logarithmic factor---that is, when
$n/d^2\to\alpha\in(0,\infty)$---simultaneous top-1 retrieval is
impossible. The matched target can nevertheless remain near the top of the
ranking. We capture this weaker retrieval goal with the Tail-Average Margin
(TAM), which, for list size $k$, compares each signal with the average
of its $k$ strongest competitors; a positive TAM margin certifies that the
target belongs to the top-$k$ candidate list. When $k/n\to r\in(0,1)$, we
learn the memory by empirical risk minimization with a smoothed TAM objective
and derive an exact high-dimensional characterization through a two-parameter
scalar variational problem. The result gives limiting laws for signal and
competitor scores, margins, and percentile ranks. Sending the ridge parameter
to zero after the high-dimensional limit yields a closed-form critical load
$\alpha_c(r)$ separating vanishing from positive average loss.

The analysis in this work also develops a coupled leave-one-out method for
matrix-valued empirical risk problems in which each sample enters many
dependent comparisons, a tool that may be useful beyond associative
memory.
\end{abstract}

{\setlength{\parskip}{0pt plus 1pt}
\tableofcontents
}

\section{Introduction}
\label{sec:intro}

Associative memory refers broadly to the ability of a system to store and retrieve relationships between objects. Given a collection of input-output pairs, one would like the system to recover the correct target when presented with a corresponding cue. This basic mechanism appears in many forms, ranging from the first neural network models for pattern learning \citep{Kohonen1972CorrelationMM,amari1972learning,anderson1972,hopfield1982neural,hopfield1984neurons}, to modern key-value retrieval systems \citep{vaswani2017attention,ramsauer2021hopfieldnetworksneed,gershman2025keyvaluememorybrain}. Associative memory has also recently attracted attention as a model of factual recall in large language models \citep{cabannes2024scaling,cabannes2024gd,jiang2024llmsdreamelephantswhen,nichani2024understanding} as trained transformer layers have been observed to function as key-value memories \citep{geva2021transformerfeedforwardlayerskeyvalue,meng2023locatingeditingfactualassociations,bietti2023birth}.

A central theoretical question in associative memory is: given $n$
associations, how large must the memory model be in order to retrieve them
under a prescribed criterion?
In this paper, we study this question in the simple
but already nontrivial setting of \emph{linear associative memory}.
Given pairs of key vectors (embeddings)~$v_i$ and associated targets
(unembeddings)~$u_i$, we consider a matrix memory model
$W \in \mathbb{R}^{d \times d}$ and define the score matrix
$S=(s_{j,i})_{1\le i,j\le n}$ as
\begin{equation}
s_{j,i} = u_j^\top W v_i, \qquad \{(u_i,v_i)\}_{i=1}^n \subset \mathbb{R}^d \times \mathbb{R}^d, \qquad W \in \mathbb{R}^{d \times d}.
\end{equation}
This score matrix encodes the predictions for the target when decoding to a fixed vocabulary of size $n$.
Rather than requiring the output $Wv_i$ of the memory model on an input
$v_i$ to equal its associated target $u_i$, as in classical models of
associative memory~\citep{hopfield1982neural}, we compare the score of
the matched target, $s_{i,i}$, with its competing scores
$(s_{j,i})_{j \ne i}$. The capacity of a linear associative memory
depends sharply on how retrieval is defined. We first study the
strongest natural criterion, which requires the matched target to
outrank every competitor.

This criterion is known as \emph{top-1 retrieval}, or exact retrieval
\citep{Kohonen1972CorrelationMM,cabannes2024scaling,cabannes2024gd,nichani2024understanding}.
Under this criterion, each key must retrieve its matched target with
the largest score:
\begin{equation}\label{eq:top1}
s_{i,i} > s_{j,i}, \qquad \forall j \neq i, \quad \forall i \in [n].
\end{equation}
This is the natural criterion for exact winner-take-all
decoding. It is also the hard-attention (low-temperature) limit of softmax attention mechanisms in transformer-based architectures \citep{vaswani2017attention}.

To derive a sharp high-dimensional characterization of top-1 capacity,
we study the isotropic Gaussian model
\begin{equation}\label{eq:gaussian_model}
u_i, v_i \stackrel{\mathrm{i.i.d.}}{\sim} \mathcal{N}(0,I_d/d),
\qquad i\in[n],
\end{equation}
where all vectors in the two collections are mutually independent.
This model isolates the effect of high-dimensional
interference without imposing additional structure on the stored
associations.
As one of the main results of this work, Theorem~\ref{thm:PR-main}
shows that, under~\eqref{eq:gaussian_model}, the asymptotic ratio $d^2/(n\log n)$ has
the sharp critical value~$2$. Specifically, for every fixed
$\epsilon>0$,
\begin{equation}\label{eq:logd}
\begin{aligned}
d^2\ge(2+\epsilon)n\log n
&\quad\Longrightarrow\quad
\text{top-1 retrieval is achievable with high probability},\\
d^2\le(2-\epsilon)n\log n
&\quad\Longrightarrow\quad
\text{top-1 retrieval is impossible with high probability}.
\end{aligned}
\end{equation}
Here, achievability means that some (data-dependent) linear memory~$W\in\R^{d\times d}$
satisfies~\eqref{eq:top1}, whereas impossibility means \emph{no} linear
memory can satisfy~\eqref{eq:top1}. The
critical value~$2$ was conjectured in an earlier version of this paper
and independently in the concurrent work of
\citet{giorlandino2026factual}.

The logarithmic factor in~\eqref{eq:logd} comes from an extreme value.
To see this heuristically, fix a query~$i$ and suppose that its competing
scores $\{s_{j,i}\}_{j\neq i}$ behave as independent centered Gaussians.
Under the Gaussian model~\eqref{eq:gaussian_model}, the matched score has
natural scale one, while an individual competing score has standard
deviation $\sigma\approx\sqrt{n}/d$. The largest of the $n-1$
competitors is therefore about
$\sigma\sqrt{2\log n}\approx\sqrt{2n\log n}/d$. Equating this extreme
competitor scale with the order-one matched score gives
$\sqrt{2n\log n}/d\approx1$, or $d^2\approx2n\log n$.
Theorem~\ref{thm:PR-main} makes this prediction rigorous.

The logarithmic penalty is specific to the requirement that every
target beat its largest competitor. Removing this factor brings us to
the critical load $n/d^2\to\alpha\in(0,\infty)$, where a linear memory
has only a constant number of parameters per association and top-1
retrieval is impossible. Even so, the matched target need not disappear
from the ranking: it may still outrank all but a fraction of
its competitors. This residual ranking information is useful
when a memory access is an intermediate computation whose output is
refined by later layers, heads, or rerankers. We thus ask:

\begin{center}
\emph{At critical load, how close to the top can a linear memory still
place every matched target?}
\end{center}

We call this coarser notion of retrieval \emph{listwise retrieval}.
Our second main result
introduces a convex certificate for listwise retrieval and gives a
sharp high-dimensional theory for learning a memory under this
criterion.

For a fixed query~$i$, order the $n-1$ competing scores as
$s_{(1),i}\ge s_{(2),i}\ge\cdots\ge s_{(n-1),i}$, and, for a tail
fraction $r\in(0,1]$, let $k=\lceil r(n-1)\rceil$. Direct top-$k$
retrieval asks that $s_{i,i}>s_{(k),i}$, so that at most $k-1$
competitors outrank the matched target. Except when $k=1$, this rank
constraint is nonconvex. Yet the average of the $k$ largest competing
scores is a convex upper bound on $s_{(k),i}$
\citep{rockafellar2000,rockafellar2002cvar}. This motivates us to propose the
\emph{Tail-Average Margin} (TAM) criterion:
\begin{equation}\label{eq:intro_tam}
    s_{i,i}
    >
    \TAM_r((s_{j,i})_{j\neq i})
    :=
    \frac{1}{k}\sum_{\ell=1}^k s_{(\ell),i},
    \qquad \forall i\in[n].
\end{equation}
Because $\TAM_r((s_{j,i})_{j\neq i})\ge s_{(k),i}$,
\eqref{eq:intro_tam} certifies that the matched target has rank at most
$k$, and hence belongs to the top-$k$ candidate list. The tail fraction
$r$ controls the resolution of retrieval. At the finest endpoint,
$r=1/(n-1)$, TAM is the largest competing score and the criterion
recovers the top-1 condition~\eqref{eq:top1}; at the coarsest endpoint,
$r=1$, it is the mean competing score. For
fixed $r\in(0,1)$, TAM therefore controls the collective strength of the
top $r$-fraction of competitors rather than a single extreme score.

This fractional-rank perspective is in the same spirit as list decoding in
coding theory, where the decoder is allowed to retain a controlled list
of plausible messages rather than make a unique decision
\citep{elias1957listdecoding,guruswami2007algorithmic}. It is also
aligned with listwise ranking losses in information retrieval, which
optimize criteria defined on ranked candidate lists rather than on a
single pairwise comparison \citep{cao2007learningtorank}. For fixed
$r$, the certified list contains a fixed fraction of the vocabulary;
the guarantee is therefore one of percentile rank rather than a
sublinear shortlist.

\subsection{Summary of main results}

\paragraph{Sharp top-1 capacity.}
Our first main result is Theorem~\ref{thm:PR-main} in
Section~\ref{sec:top1}, which identifies the exact capacity of linear
associative memory under top-1 retrieval as $d^2/(n\log n)=2$. This critical threshold applies to the
entire class of data-dependent linear memories, rather than to a
particular construction or learning rule.

The achievability proof begins with a classical benchmark model,
called the correlation matrix memory (CMM), which stores the
associations as a superposition of outer products
\citep{Kohonen1972CorrelationMM,elhage2022toymodelssuperposition}.
Previous work of \citet{nichani2024understanding} showed that CMM
succeeds when $d^2\ge Cn(\log n)^4$ for a sufficiently large constant
$C$. We sharpen this benchmark by proving the sufficient condition
$\liminf d^2/(n\log n)>8$
(Proposition~\ref{prop:sup-achievability}). We then introduce an
adaptive sparse repair, which uses CMM on a large coordinate block to store most associations, then corrects wrong (or small-margin) associations using a small reserve block of size~$o(d)$. This reserve-and-repair principle is related in
spirit to multiscale constructions for random perceptrons
\citep{kim1998covering,abbe2022binary,ding2025capacity}.
Our construction succeeds throughout the regime
$\liminf d^2/(n\log n)>2$ (Proposition~\ref{prop:optimal}); see
Sections~\ref{sec:superposition}--\ref{sec:top1-achievability} for details.
Conversely, a dual
certificate rules out every linear memory with high probability when
$\limsup d^2/(n\log n)<2$
(Theorem~\ref{thm:PR-main}\ref{item:impossibility}).
Moreover, Section~\ref{sec:top1-finite-size} presents
Conjecture~\ref{conj:top1-finite-size}, suggested by a simpler surrogate
model: a $-3n\log\log n$ second-order correction that better captures
finite-size behavior.

\paragraph{Tail-Average Margin.}
Our second main result gives an exact high-dimensional theory for a
convex learning rule based on the listwise criterion~\eqref{eq:intro_tam}.
With logistic loss $\ell(t)=\log(1+e^{-t})$, ridge parameter
$\lambda\ge0$, we consider the estimator
\begin{equation}\label{eq:intro_tam_erm}
    \widehat W
    \in
    \argmin_{W\in\mathbb{R}^{d\times d}}
    \left\{
        \frac1n\sum_{i=1}^n
        \ell\left(
            s_{i,i}(W)
            -
            \TAM_{r,\beta}\bigl((s_{j,i}(W))_{j\ne i}\bigr)
        \right)
        +
        \frac{\lambda}{2d^2}\norm{W}_{\mathsf F}^2
    \right\}.
\end{equation}
Here $\TAM_{r,\beta}$ is an exponentially smoothed version of the TAM
functional $\TAM_r$ in~\eqref{eq:intro_tam}, with $\beta>0$ the
smoothing parameter. For every fixed
score vector $x$, $\TAM_{r,\beta}(x)\to\TAM_r(x)$ as
$\beta\to\infty$. Its precise definition is given in
Section~\ref{sec:tam-functional}, specifically in~\eqref{eq:sTAM}, and
the corresponding convex learning formulation is developed in
Section~\ref{sec:tam-learning}; see~\eqref{eq:tam-learning}.
The main goal of the TAM theory is to understand the
asymptotic behavior of this ERM in the regime
\begin{equation}
    n,d\to\infty,
    \qquad
    \frac{n}{d^2}\to\alpha\in(0,\infty),
    \qquad
    r\in(0,1)\ \text{fixed}.
\end{equation}
Our main asymptotic conclusions are as follows:
\begin{enumerate}[leftmargin=*, itemsep=2pt]
    \item \emph{Scalar reduction.} The high-dimensional ERM admits a scalar variational
    description involving only two parameters, stated in
    Theorem~\ref{thm:tam_scalar} in Section~\ref{sec:tam_selfconsistent};
    \item \emph{Score and margin laws.} The scalar description gives limiting predictions for the
    fitting loss, the laws of the true and competing scores, and the
    induced margin and percentile profiles, developed in
    \sref{tam_theory} and \sref{tam_phase_transitions};
    \item \emph{Ridgeless phase transition.} In the ridgeless limit, the scalar theory yields a sharp phase
    transition at a critical load
    \begin{equation}\label{eq:fix_tail_threshold}
    \alpha_c(r)
    =
    \frac{1}{
        (1+\kappa_r^2)\Phi(\kappa_r)+\kappa_r\varphi(\kappa_r)
    },
    \qquad
    \kappa_r
    =
    \frac{\varphi(\Phi^{-1}(1-r))}{r},
\end{equation}
where $\varphi$ and $\Phi$ denote the standard normal density and
distribution function. When $\alpha<\alpha_c(r)$, the limiting TAM
fitting loss vanishes in the ridgeless limit; when
$\alpha>\alpha_c(r)$, it remains positive. Thus $\alpha_c(r)$ is the
sharp threshold for asymptotically perfect TAM recovery---vanishing
average TAM loss and hence successful TAM retrieval for all but a
vanishing fraction of keys.
\end{enumerate}

\paragraph{A coupled leave-one-out method.}
Our third contribution is methodological: we develop a new coupled
leave-one-out method for matrix-valued empirical-risk problems with
shared samples. This method forms the technical core of our TAM theory.
Classical
leave-one-out analyses of high-dimensional $M$-estimators remove one
independent loss term, leaving an optimizer that is nearly independent
of the held-out feature
\citep{elkaroui2013robust,lei2018asymptotics}. That basic decoupling
fails here: a stored pair supplies its own signal score, while its
target also appears as a competitor for every other query. Removing
one pair therefore perturbs order $n$ loss terms and their
data-dependent tail thresholds.

We develop a coupled leave-one-out expansion that tracks this global
response. Its key conclusion is that, after all induced changes are
accounted for, the leading perturbation of the optimizer is the inverse
leave-one-out Hessian applied to the single rank-one feature
$u_hv_h^\top$; the remaining terms are lower order. Iterating this
expansion and combining it with new spectral and self-averaging
estimates yields the Gaussian competitor law and the scalar
description in Theorem~\ref{thm:tam_scalar}. This method may be of
independent interest for other matrix-valued or listwise ERMs in which
each observation participates in many dependent comparisons rather
than one separable loss term. Section~\ref{sec:tam_loo} develops the
coupled reduction, and the supporting estimates are proved in
Appendices~\ref{app:loo_stability_proofs}--\ref{app:bilinear_coupling}.

\subsection{Related work}
\label{sec:related_work}

\paragraph{Neural associative memory.} Associative memory dates back to
early neural models that store patterns through correlations in the
weights \citep{amari1972learning,anderson1972,nakano1972associatron,Kohonen1972CorrelationMM,hopfield1982neural,hopfield1984neurons}.
The Hopfield network \citep{hopfield1982neural} is the symmetric case,
where the goal is to memorize and retrieve the patterns themselves.
We focus instead on associative recall: a key should retrieve its
paired target, and the two need not coincide \citep{amari1972learning}.
In particular,
\citet{Kohonen1972CorrelationMM} define correlation matrix memory as
a sum of rank-one key-target outer products, the construction analyzed
in Section~\ref{sec:superposition}. Bidirectional associative memory
is another paired-pattern model based on outer-product updates
\citep{kosko1988}.

Modern versions of neural networks with built-in associative memory capabilities include the neural Turing machine \citep{graves2014neuralturingmachines}, memory networks \citep{weston2015memorynetworks}, dense associative memory \citep{krotov2016denseassociativememorypattern,Lucibello2024}, MCHN \citep{ramsauer2021hopfieldnetworksneed}, MCSDM \citep{Kanerva1988SparseDistributedMemory,bricken2022attentionapproximatessparsedistributed}, and universal Hopfield networks \citep{millidge2022universalhopfieldnetworksgeneral}. Binary variants of matrix memory have also been studied in the literature \citep{Willshaw1969NonHolographicAM,gardner1987,wreo2211}.

\paragraph{Transformers and factual recall.} Transformer models
trained on factual recall tasks have been empirically observed to
implement associative memory through their weight matrices. Feedforward as well as attention layers have been shown to behave like
key-value memories \citep{geva2021transformerfeedforwardlayerskeyvalue,bietti2023birth},
and the linear memory perspective has been used to study factual editing,
induction head tasks, and latent space retrieval
\citep{meng2023locatingeditingfactualassociations,bietti2023birth,jiang2024llmsdreamelephantswhen}.
On the theory side, \citet{nichani2024understanding} study a
sequence-based factual recall task where one token functions as the
key, and show that a one-layer transformer solves the task with capacity scaling linearly with parameter count, up to logarithmic factors. This storage scaling has also been empirically observed in large language models \citep{allenzhu2024physicslanguagemodels33}. Related theoretical analyses include optimization dynamics for
a single-layer transformer on a token recall task \citep{vural2026learningrecalltransformersorthogonal},
and a measure-theoretic formulation of transformers with implicit associative
memory mechanisms
\citep{kawata2026transformersmeasuretheoreticassociativememory}.

\paragraph{Storage capacity.}
\citet{nichani2024understanding} study the capacity of linear and MLP
associative memories with random spherical embeddings, showing in
particular that correlation matrix memory succeeds when
$d^2\gg n(\log n)^4$. \citet{cabannes2024scaling} propose the
exact retrieval criterion and study correlation matrix memory under a
weighted data distribution. From an optimization perspective, \citet{cabannes2024gd} study gradient descent under a multiclass logistic formulation, and \citet{kim2026sharp} compare the capacity of memory matrices obtained by different optimizers under Gaussian embeddings. Compared with these works, our top-1 result
identifies the exact statistical threshold $d^2=2n\log n$ for arbitrary linear memories. We
also introduce the TAM criterion and, to our knowledge, give the first
capacity theory for linear associative memory under a listwise retrieval rule.

We also briefly mention that storage capacity has been studied in detail for the Hopfield network and its variants. The capacity $d\asymp n$ of Hopfield networks is classical, see e.g., \citet{mceliece1987,Folli2017MaximumStorageCapacity} and the works cited within, while modern extensions such as dense associative memory have been shown to achieve exponential capacity \citep{demircigil2017model, ramsauer2021hopfieldnetworksneed,Lucibello2024}. The scaling in these works differ from our results since we decode to a fixed vocabulary of size $n$, while the classical Hopfield formulation decodes to the hypercube.

\paragraph{Concurrent work.}
The sharp top-1 threshold was conjectured in an earlier version of this
paper and independently in the concurrent work of
\citet{giorlandino2026factual}. The latter studies the capacity of linear
associative memory under the top-1 retrieval criterion through a \textit{decoupled}
model. In this simplified formulation,
the output embeddings $\{u_j\}_{j \le n}$ are \textit{resampled} across
constraints, thus removing the statistical dependence structure among
the retrieval inequalities in \eqref{eq:top1}. Based on numerical
evidence comparing training and test errors, and score distributions, they conjecture that
the original and decoupled problems have the same critical capacity.
For the decoupled model, they carry out a Gardner-type statistical
physics calculation and predict the critical threshold
$\frac{n\log n}{d^2}=\frac12$ for top-1 retrieval, which we rigorously prove in \sref{top1}. They further extend both the
decoupling conjecture and the capacity calculation to rank-constrained
memories, corresponding to two-layer linear MLPs.

\paragraph{Classification and leave-one-out methods.}
The TAM learning problem is also connected to high-dimensional
classification through its logistic loss. The classical work of
\citet{cover1965geometrical} computed the separability capacity of
random patterns, and the perceptron capacity was later studied through
statistical physics methods, including margin-constrained variants
\citep{gardner1988space}. In ordinary unregularized logistic
regression, the existence of the maximum likelihood estimator is
equivalent to nonseparability of the labeled data; equivalently, the
failure of existence is a linear separability, or capacity, event.
\citet{sur2019modern} established a high-dimensional phase transition
for this event under Gaussian covariates. A broader line of work in
statistical physics
studies optimal errors and phase transitions in high-dimensional
generalized linear models via the approximate message passing algorithm
\citep{barbier2019optimal,aubin2020generalization,mignacco2020role}.
From a technical point of view, our TAM analysis uses a leave-one-out approach, which has played an important role in precise asymptotics for high-dimensional $M$-estimators
\citep{elkaroui2013robust,elkaroui2018impact,lei2018asymptotics}.
At a conceptual level, this is also the cavity method from spin glass theory \citep{mezard1987spin}.

The present problem differs from standard high-dimensional regression
and classification in several ways. The unknown is a matrix rather
than a vector; the effective covariates are rank-one tensors
$u_jv_i^\top$ with strong dependencies; each sample
contributes both a signal score and a full column of competitor
scores; and the TAM threshold is endogenous, selected by the learned
scores themselves. These couplings prevent existing leave-one-out
theories from applying as black boxes; the resulting cavity
calculation is matrix-valued, listwise, and self-consistent in a way
that does not arise in standard high-dimensional regression.

\subsection*{Notation}

We write
\begin{equation}\label{eq:UV}
  U = [u_1 \ \cdots \ u_n] \in \R^{d \times n}
  \qquad \text{and} \qquad
  V = [v_1 \ \cdots \ v_n] \in \R^{d \times n}
\end{equation}
for the target and key matrices, respectively. Under the Gaussian
model~\eqref{eq:gaussian_model}, $\sqrt d\,U$ and $\sqrt d\,V$ are
independent standard Gaussian matrices in $\R^{d\times n}$.
$\vecop(\cdot)$ denotes column-wise vectorization.

We write $X\eqd Y$ when two random
variables, vectors, or matrices have the same distribution. For matrices,
$\inprod{A,B}$ is understood as the matrix inner product.

We use $X_n=O_{\P}(Y_n)$ and $X_n=o_{\P}(Y_n)$ in their standard
meanings: $X_n/Y_n$ is tight and converges to zero in probability,
respectively. These notations record behavior at the displayed scale. In situations where we need tail control, especially uniformly over
polynomial-size index sets, we use stochastic
domination~\citep{erdHos2017dynamical}. Let
$\{X_\alpha^{(n)}\}_{\alpha\in\mathcal I_n}$ be nonnegative random
variables, where $\mathcal I_n$ may depend on $n$ and has polynomial
cardinality, and let $Y_n\ge0$ be deterministic. We write
\begin{equation}\label{eq:Oprec}
    X_\alpha^{(n)} = \Op(Y_n)
    \qquad \text{uniformly in } \alpha\in\mathcal I_n
\end{equation}
if, for every $\varepsilon>0$ and $D>0$,
\begin{equation}
    \P\left(
        \max_{\alpha\in\mathcal I_n} X_\alpha^{(n)}
        > n^\varepsilon Y_n
    \right)
    \le n^{-D}
\end{equation}
for all sufficiently large $n$. When $\mathcal I_n$ is a singleton, we
suppress the uniformity qualifier.
Thus, \eqref{eq:Oprec} permits an $n^\varepsilon$ slack in
exchange for $n^{-D}$ tail control, uniformly over the indicated
polynomial-size index set.

\section{Top-1 retrieval: the sharp threshold}
\label{sec:top1}

Throughout this section, we work under the Gaussian model
\eqref{eq:gaussian_model} and the top-1 retrieval criterion
\eqref{eq:top1}. For a realization $(U,V)$ and a matrix
$W\in\R^{d\times d}$, define the perfect-retrieval condition
\begin{equation}\label{eq:PR-condition}
    \mathsf{PR}(W;U,V)
    \quad\Longleftrightarrow\quad
    u_i^\top Wv_i
    >
    \max_{j\ne i}u_j^\top Wv_i
    \quad\text{for every }i\in[n].
\end{equation}
We study the random feasibility event
\begin{equation}\label{eq:PR-exists}
    \mathsf{PR}_n
    \bydef
    \left\{
        (U,V):
        \text{there exists }W\in\R^{d\times d}
        \text{ such that }\mathsf{PR}(W;U,V)\text{ occurs}
    \right\}.
\end{equation}
The matrix $W$ in this definition may depend arbitrarily on the
realized data $(U,V)$. Accordingly, $\mathsf{PR}_n$ asks whether a
random instance is retrievable by any linear memory, without imposing
a particular learning rule.

\begin{theorem}[Sharp top-1 threshold]
\label{thm:PR-main}
Along any sequence $n,d\to\infty$:
\begin{enumerate}[label=\textup{(\roman*)}]
    \item\label{item:achievability}
    If $\liminf \frac{d^2}{n\log n}>2$, then
    $\P(\mathsf{PR}_n)\to1$.

    \item\label{item:impossibility}
    If $\limsup \frac{d^2}{n\log n}<2$, then
    $\P(\mathsf{PR}_n)\to0$.
\end{enumerate}
\end{theorem}

The sharp threshold at $d^2=2n\log n$ was conjectured in an earlier
version of this paper and independently in the concurrent work of
\citet{giorlandino2026factual}.
We establish the achievability and impossibility directions in
\sref{top1-achievability} and \sref{top1-impossibility}, respectively.
Before turning to these proofs, \sref{superposition} introduces the
correlation matrix memory (CMM), a classical outer-product construction.
Its analysis exposes
the extreme-value origin of the logarithmic factor and motivates the
sparse repair used in the achievability proof. We note that an order-wise
$\Omega(n\log n)$ lower bound is also suggested in \citet[Theorem~7]{nichani2024understanding}
via an information-theoretic argument, but it does not identify the sharp
constant and applies only to finite-bit learners. Finally,
\sref{top1-finite-size} studies finite-size scaling and proposes a
second-order conjecture.

\subsection{Correlation matrix memory and the logarithmic factor}
\label{sec:superposition}

The simplest linear memory stores each association by its outer
product and superposes the results:
\begin{equation}\label{eq:cmm}
    W_{\mathsf{CMM}}
    \bydef
    UV^\top
    =\sum_{k=1}^n u_kv_k^\top.
\end{equation}
This is the classical \emph{correlation matrix memory} (CMM)
\citep{Kohonen1972CorrelationMM} and is the direct analogue of the
classical self-organizing network \citep{amari1972learning} for
asymmetric key and target pairs. When the key $v_i$ is presented, the
$k$th summand contributes
$u_k\langle v_k,v_i\rangle$: the desired vector $u_i$ appears with
coefficient $\|v_i\|^2$, while the remaining terms create cross-talk.
The rule is transparent and admits a rank-one update whenever an
association is added or removed. It also illustrates
\emph{superposition}: the stored vectors need not be mutually
orthogonal for retrieval to succeed
\citep{elhage2022toymodelssuperposition}.

To understand the performance of this particular, data-dependent
memory, write
\[
    \alpha_n\bydef\frac{n}{d^2}
\]
and denote the CMM scores by
$s_{j,i}=u_j^\top W_{\mathsf{CMM}}v_i$. Expanding \eqref{eq:cmm} gives
\begin{align}
    s_{i,i}
    &={\|u_i\|^2\|v_i\|^2}
      +\sum_{k\ne i}
      \langle u_i,u_k\rangle\langle v_k,v_i\rangle,
      \label{eq:cmm-signal-decomp}\\
    s_{j,i}
    &=\langle u_j,u_i\rangle\|v_i\|^2
      +\|u_j\|^2\langle v_j,v_i\rangle
      +\sum_{k\notin\{i,j\}}
      \langle u_j,u_k\rangle\langle v_k,v_i\rangle,
      \qquad j\ne i.
      \label{eq:cmm-competitor-decomp}
\end{align}
Under the Gaussian model,
$\|u_i\|^2=1+O_{\P}(d^{-1/2})$ and
$\|v_i\|^2=1+O_{\P}(d^{-1/2})$, so the first term in
\eqref{eq:cmm-signal-decomp} is $1+O_{\P}(d^{-1/2})$. Conditional on
$(u_i,v_i)$, the summands in the cross-talk term of
\eqref{eq:cmm-signal-decomp} are independent and centered, each with
variance $\|u_i\|^2\|v_i\|^2/d^2$. Their sum therefore has variance
approximately $n/d^2=\alpha_n$ and is approximately Gaussian by the
central limit theorem.

Turning to a competitor score, the same conditional central limit
argument applies to the cross-talk sum over $k\notin\{i,j\}$ in
\eqref{eq:cmm-competitor-decomp}. The remaining $k=i$ and $k=j$
contributions,
$\langle u_j,u_i\rangle\|v_i\|^2$ and
$\|u_j\|^2\langle v_j,v_i\rangle$, respectively, are only
$O_{\P}(d^{-1/2})$. At the scale $d^2\asymp n\log n$, they and the
norm fluctuations above are therefore
$O_{\P}(d^{-1/2})=o_{\P}(\sqrt{\alpha_n})$.

Combining these estimates suggests the following leading-order
Gaussian picture for the scores associated with a fixed query:
\begin{equation}\label{eq:cmm-score-heuristic}
    s_{i,i}\approx 1+\sqrt{\alpha_n}\,Z_i,
    \qquad
    s_{j,i}\approx\sqrt{\alpha_n}\,Z_{j,i},
    \quad j\ne i.
\end{equation}
Here the variables on the right are heuristically treated as
independent standard Gaussians. The exact one-column representation
in
\ifdefined\sectiontwoworking
Appendix~A.1
\else
Lemma~\ref{lem:single-coord-law}
\fi
provides a finite-$(n,d)$ description underlying this picture.
For a fixed ordered pair $i\ne j$,
\eqref{eq:cmm-score-heuristic} predicts
\[
    s_{i,i}-s_{j,i}
    \approx
    1+\sqrt{2\alpha_n}\,G_{j,i},
    \qquad G_{j,i}\sim\mathcal N(0,1),
\]
and hence, by standard Gaussian tail bounds,
\[
    \P(s_{i,i}\le s_{j,i})
    \approx
    \exp\!\left(-\frac{1}{4\alpha_n}\right).
\]
Perfect retrieval requires all $n(n-1)$ such margins to be positive.
A union bound therefore predicts
\begin{equation}\label{eq:cmm-eight-heuristic}
    \P\big(\text{$W_{\mathsf{CMM}}$ fails perfect retrieval}\big)
    \lesssim
    n^2\exp\!\left(-\frac{1}{4\alpha_n}\right)
    =\exp\!\left(2\log n-\frac{d^2}{4n}\right),
\end{equation}
which vanishes whenever
$\liminf d^2/(n\log n)>8$. This also reveals the origin of the
logarithmic factor: cross-talk has size $\sqrt n/d$, but the worst of
quadratically many Gaussian-like margins incurs an additional
$\sqrt{\log n}$ extreme-value factor.

Previously, \citet[Theorem~1]{nichani2024understanding} proved that
CMM succeeds with high probability when
$d^2\ge Cn(\log n)^4$ for a sufficiently large constant $C$. The
calculation above sharpens this sufficient condition to the
$n\log n$ scale and predicts the leading constant $8$. The following
proposition makes this improvement rigorous.

\begin{proposition}[CMM succeeds above constant $8$]
\label{prop:sup-achievability}
Along any sequence $n,d\to\infty$,
\[
    \liminf\frac{d^2}{n\log n}>8
    \quad\Longrightarrow\quad
    \P\big(\mathsf{PR}(W_{\mathsf{CMM}};U,V)\big)\to1.
\]
\end{proposition}

\begin{proof}
\ifdefined\sectiontwoworking
See the appendix.
\else
See Appendix~\ref{app:perfect_retrieval}.
\fi
\end{proof}

This sufficient constant $8$ is nevertheless four times the optimal
constant $2$ in Theorem~\ref{thm:PR-main}  (and one can show that this is unimprovable for CMM). Where does this gap come
from? The calculation
leading to \eqref{eq:cmm-eight-heuristic}
examines one pairwise margin at a time and then protects all
$n(n-1)$ margins by a union bound. This does not require the margins
to be independent, but it discards their strong dependence. For a
fixed query $v_i$, the signal score $s_{i,i}$ is realized only once
and is shared by every comparison, whereas, at leading order, the
competitor scores $\{s_{j,i}:j\ne i\}$ behave as independent
Gaussians. The exact distributional representation behind this claim
is given in
\ifdefined\sectiontwoworking
Appendix~A.1.
\else
Appendix~\ref{app:lemma_1_proof}.
\fi
The natural query-wise object is therefore
\[
    s_{i,i}-\max_{j\ne i}s_{j,i}
    \approx
    1+\sqrt{\alpha_n}Z_i-\sqrt{2\alpha_n\log n}.
\]
Here $\sqrt{2\alpha_n\log n}$ comes from the standard estimate for
the maximum of $n$ independent Gaussian variables with variance
$\alpha_n$.
The constant $8$ arises in the all-pairs argument because it
repeatedly pays for fluctuations of the same signal score as though
each comparison were a separate bad event; the query-wise picture
pays for the maximum of the competitors only once.

If $d^2/(n\log n)>2$ by a fixed amount, the display above predicts a
positive margin for any \textit{typical} query. Perfect retrieval, however,
requires \textit{all} $n$ queries to succeed, and a typical-query calculation
cannot rule out a small set with unusually low margins. Thus CMM
correctly handles the bulk at constant $2$, but by itself does not
provide the required all-query guarantee. This naturally suggests a
sparse correction: retain the CMM solution for most queries and
modify it only to repair the exceptional ones. Section~\ref{sec:top1-achievability}
makes this idea precise.

\subsection{Achievability by adaptive sparse repair}
\label{sec:top1-achievability}

We reserve a vanishing fraction of the coordinates for repair and use
the rest to form a base CMM. Independence of the reserved block then
allows us to repair only the deficient queries. Specifically, we
construct a block-diagonal memory
\[
    W^\star\bydef\diag(W^0,W^1),
\]
Here $W^0\in\R^{d_0\times d_0}$ is the base memory and
$W^1\in\R^{d_1\times d_1}$ is the repair; the block dimensions
$d_0,d_1$ are specified below. Split every vector into independent coordinate blocks
$u_i=(u_i^0,u_i^1)$ and $v_i=(v_i^0,v_i^1)$ of dimensions
\[
    d_1=\left\lfloor d(\log n)^{-1/8}\right\rfloor,
    \qquad d_0=d-d_1.
\]
The large block stores all associations using the rescaled CMM
\begin{align*}
    W^0\bydef \frac{d^2}{d_0^2}
      \sum_{k=1}^n u_k^0(v_k^0)^\top,\qquad
    s_{j,i}
    \bydef (u_j^0)^\top W^0v_i^0.
\end{align*}
Its retrieval margin for query $i$ is
\[
    m_i
    \bydef s_{i,i}-\max_{j\ne i}s_{j,i}.
\]
For a tolerance $\gamma>0$ (to be chosen later), call a query
\emph{deficient} when its
margin is at most $\gamma$, and collect these queries in
\[
    \mathcal B
    \bydef\{i:m_i\le\gamma\}.
\]
For each $i\in\mathcal B$, the coefficient $2\gamma-m_i$ is the
amount by which we aim to raise its margin. We therefore use the
reserved block to add the targeted rank-one corrections
\[
    W^1
    \bydef
    \frac{d^2}{d_1^2}
    \sum_{i\in\mathcal B}(2\gamma-m_i)u_i^1(v_i^1)^\top.
\]
Both $\mathcal B$ and the correction coefficients depend only on the
large block and are thus independent of the coordinates used for
repair.

\begin{proposition}[Sparse repair succeeds above constant $2$]
\label{prop:optimal}
Suppose $\liminf_{n\to\infty}d^2/(n\log n)>2$.
Choose constants $c$ and $\gamma\in(0,1/2)$ such that
\[
    2<c<\liminf_{n\to\infty}\frac{d^2}{n\log n},
    \qquad
    c(1-2\gamma)^2>2,
\]
and apply the construction above with this value of $\gamma$. Then
\[
    \P\big(\mathsf{PR}(W^\star;U,V)\big)\to1.
\]
Consequently, Theorem~\ref{thm:PR-main}\ref{item:achievability}
holds.
\end{proposition}

\begin{proof}
Fix $i\ne j$ and set $\alpha = d_0/d\in (0,1)$. Conditioning on $x=u_j^0$ and $y=v_i^0$, write
\begin{align*}
s_{j,i} = \frac{1}{\alpha^2} \left[\|x\|^2 (v_j^0)^\top y + (x^\top u_i^0)\|y\|^2 + \sum_{k\ne i,j} (x^\top u_k^0) (v_k^0)^\top y\right]
\end{align*}
where all summands are conditionally independent. Then the log moment generating function of $s_{j,i}$ is
\begin{align*}
\log\E[e^{\lambda s_{j,i}} \mid x,y] = \frac{\lambda^2}{2\alpha^4 d} \|x\|^2\|y\|^2(\|x\|^2 + \|y\|^2) - \frac{n-2}{2} \log\left(1-\frac{\lambda^2}{\alpha^4 d^2}\|x\|^2\|y\|^2\right).
\end{align*}
Taking $\lambda =\alpha^2 cz\log n$ and noting that $\|x\|^2,\|y\|^2 = \alpha(1+o(1))$ with probability at least $1-n^{-C}$ for any constant $C>0$, the first term is $o(1)$ and the second is
\begin{align*}
\frac{n-2}{2} \frac{\lambda^2}{\alpha^4 d^2}\|x\|^2\|y\|^2 (1+o(1)) = \frac{\alpha^2 c^2z^2 n(\log n)^2}{2d^2} (1+o(1)) < \frac12 \alpha^2 cz^2\log n (1+o(1)).
\end{align*}
By the Chernoff bound, we thus have
\begin{align*}
\log\Pr(s_{j,i}\ge z\mid x,y) &\le -\lambda z + \log\E[e^{\lambda s_{j,i}} \mid x,y] \\
&\le -\frac12 \alpha^2 cz^2\log n + o(\log n),
\end{align*}
and so
\begin{align}\label{eq:mgf-1}
\Pr(s_{j,i}\ge z) &\le n^{- \frac12\alpha^2 cz^2 + o(1)} + n^{-C}.
\end{align}
Next, a straightforward computation gives
\begin{align*}
\E[s_{i,i}] = 1, \qquad \var(s_{i,i}) = \frac{4}{\alpha d} + \frac{n+3}{\alpha^2 d^2} = O\left(\frac{1}{\log n}\right).
\end{align*}
Writing
\begin{align*}
q_i := s_{i,i} - \alpha^{-2} \|u_i^0\|^2 \|v_i^0\|^2,
\end{align*}
a similar computation as above gives
\begin{align}\label{eq:mgf-2}
\Pr(|q_i| \ge z) \le 2n^{-\frac12 \alpha^2 cz^2 + o(1)} + n^{-C}.
\end{align}
Plugging $z=2$ in~\eqref{eq:mgf-1} and $z=1$ in~\eqref{eq:mgf-2} and choosing $C>2$, by union bounding, it follows that
\begin{align*}
\min_i s_{i,i} \ge -1, \quad \max_{j\ne i} s_{j,i} \le 2
\end{align*}
with probability tending to~$1$.

Now recall the definition of deficient queries $\mathcal B=\{i:m_i\le\gamma\}$. If $i\in \mathcal B$ and $s_{i,i}>1-\gamma$, then $\max_{j\ne i}s_{j,i} \ge 1-2\gamma$. Therefore
\begin{align*}
\Pr(i\in\mathcal B) &\le \Pr(s_{i,i} \le 1-\gamma) + \sum_{j\ne i} \Pr(s_{j,i} \ge 1-2\gamma) \\
&\le \frac{\var(s_{i,i})}{\gamma^2} + n^{1-\frac12 c(1-2\gamma)^2 + o(1)} = O\left(\frac{1}{\log n}\right).
\end{align*}
Here, we have used $c(1-2\gamma)^2>2$ and Chebyshev’s inequality. It follows from Markov's inequality that
\begin{align*}
\Pr\left(|\mathcal B| > \frac{n}{\sqrt{\log n}}\right) = o(1).
\end{align*}
Write the repair block as
\begin{align*}
W^1 = (1-\alpha)^{-2} \sum_{i=1}^n w_i u_i^1(v_i^1)^\top, \qquad w_i = (2\gamma-m_i) 1\{i\in \mathcal B\}, \quad t_{j,i} := (u_j^1)^\top W^1 v_i^1.
\end{align*}
Conditional on the~$d_0$ block and the above events, the weights~$w_i$ are fixed and satisfy
\begin{align*}
0\le w_i\le 4, \qquad \|w\|_2 \le 4\sqrt{|\mathcal B|} \le 4n^{1/2}(\log n)^{-1/4}.
\end{align*}
Then by chi-square concentration and Bernstein's inequality for subexponential random variables, it holds with probability $1-o(1)$ that
\begin{align*}
\max_i |t_{i,i} - w_i| \vee \max_{j\ne i} |t_{j,i}| &= O\left(\|w\|_\infty \sqrt{\frac{\log n}{d_1}} + \frac{\|w\|_\infty\log n}{d_1} + \frac{\|w\|_2 \sqrt{\log n}}{d_1}\right) \\
&= O\left(n^{-1/4}(\log n)^{5/16} + n^{-1/2} (\log n)^{5/8} + (\log n)^{-1/8}\right).
\end{align*}
Choose~$n$ sufficiently large such that the right-hand side is less than $\gamma/2$. Then for every~$i$ and~$j\ne i$,
\begin{align*}
u_i^\top W^\star v_i - u_j^\top W^\star v_i &=  s_{i,i}-s_{j,i} + t_{i,i}-t_{j,i} \\
&\ge m_i + w_i - \gamma.
\end{align*}
If $i\notin \mathcal B$, then $m_i>\gamma$ and $w_i=0$; if $i\in \mathcal B$, then $m_i+w_i=2\gamma$. Thus all margins become positive under~$W^\star$ and hence $\mathsf{PR}(W^\star;U,V)$ is satisfied with probability tending to one, as was to be shown.
\end{proof}

\begin{remark}[Relation to perceptron constructions]
The sparse-repair argument in our proof is related in spirit to
multiscale majority constructions for random perceptrons, where
successive fresh coordinate blocks are used to improve the constraints
with the weakest current margins
\citep{kim1998covering,abbe2022binary}. In the related construction of
\citet{ding2025capacity}, the deficits from a main construction are also
filled using a small coordinate block. In contrast, our repair is
one-shot: a single independent block targets the vanishing set of
CMM-deficient queries through weighted rank-one corrections.
\end{remark}

\subsection{Impossibility via a dual certificate}
\label{sec:top1-impossibility}

We now prove the impossibility direction of
Theorem~\ref{thm:PR-main}. The argument has a useful geometric
interpretation.
For each ordered comparison $i\ne j$, define the matrix
\[
    A_{ij}\bydef (u_i-u_j)v_i^\top.
\]
A memory $W$ retrieves every association if and only if
$\inprod{W,A_{ij}}>0$ for every $i\ne j$. Thus to show impossibility, it is
enough to find nonnegative coefficients $\lambda=(\lambda_{ij})$, not all zero, such that
\[
    \sum_{i\ne j}\lambda_{ij}A_{ij}=0.
\]
Indeed, taking the Frobenius inner product with such a $W$ would make the
left-hand side both zero and strictly positive, a contradiction. Such a nontrivial nonnegative linear
dependence among these comparison matrices is our \emph{dual certificate}.
The same principle underlies
classical results on random half-spaces and linear separability
\citep{wendel1962problem,cover1965geometrical}.

The proof constructs the certificate in two stages. First, we choose
positive coefficients for which the above weighted sum is already close to zero; the threshold $2$ ensures
that there are enough comparisons for this approximate cancellation.
Second, we show that small adjustments to these coefficients can
produce any sufficiently small matrix perturbation, so that the sum can be made to be exactly zero while keeping all coefficients positive.

\begin{proof}[Proof of \thref{PR-main}\ref{item:impossibility}]
Since the threshold clearly must be monotonic in~$n$, we may assume without loss of generality that
\begin{align*}
\frac{1}{2}
\le
\frac{d^2}{n\log n}
\le
c
<
2.
\end{align*}
We will use the index cutoff $r := \lfloor n(\log n)^{-2}\rfloor$ to construct the dual certificate. For an array $\lambda=(\lambda_{ij})_{i>r,\,j\le r}$, define the linear map
\begin{align*}
\mathcal A(\lambda)
:=
\sum_{i=r+1}^n\sum_{j=1}^r
\lambda_{ij}(u_i-u_j)v_i^\top.
\end{align*}
If some $W\in\R^{d\times d}$ retrieves every association, it would follow that for every nonzero array $\lambda\ge0$,
\begin{align*}
\inprod{W,\mathcal A(\lambda)}_{\mathsf F}
=
\sum_{i=r+1}^n\sum_{j=1}^r
\lambda_{ij}(u_i-u_j)^\top Wv_i
>
0.
\end{align*}
It therefore suffices to construct an array~$\lambda$ with strictly positive entries such that $\mathcal A(\lambda)=0$. We first construct a positive array for which $\mathcal A(\lambda)$ is small.

\begin{lemma}\label{lem:array}
With probability tending to~$1$, there exists an array~$\lambda$ such that
\begin{align*}
\min_{i>r, j\le r}\lambda_{ij}
\ge
\frac{1}{r\log n},
\qquad \|\mathcal A(\lambda)\|_{\mathsf F}
=
o\left(\frac{n}{d\log n}\right).
\end{align*}
\end{lemma}

\begin{proof}
Write
\begin{align*}
U_1 := [u_{r+1}\;\cdots\;u_n], \qquad V_1 := [v_{r+1}\;\cdots\;v_n], \qquad M:= U_1V_1^\top, \qquad H:= V_1V_1^\top.
\end{align*}
Since $n-r>d$ for sufficiently large $n$,~$H$ is invertible almost surely. For each $i>r$, define
\begin{align*}
x_i := MH^{-1}v_i, \qquad \Pi := V_1^\top H^{-1}V_1.
\end{align*}
We have the identity
\begin{align*}
\sum_{i=r+1}^n x_iv_i^\top
=
MH^{-1}\sum_{i=r+1}^n v_iv_i^\top
=
M.
\end{align*}
Also,~$\Pi$ is a uniformly distributed rank-$d$ orthogonal projection in $\R^{n-r}$, and
\begin{align*}
x_i = U_1\Pi e_{i-r}.
\end{align*}
Set
\begin{align*}
\varepsilon_n
:=
\sqrt{\frac{\log n}{d}}
+
\sqrt{\frac{d}{n}}.
\end{align*}
Conditionally on $V_1$, the matrix $U_1$ is Gaussian and independent of $\Pi$.  Standard concentration bounds give, with probability tending to one,
\begin{align*}
\max_{i>r}
\left|
\frac{n-r}{d}\norm{x_i}^2-1
\right|
&\le
C\varepsilon_n,\qquad
\max_{\substack{i,k>r\\i\ne k}}
\frac{|x_i^\top x_k|}{\norm{x_i}\norm{x_k}}
\le
C\varepsilon_n, \qquad \max_{\substack{i,k>r\\i\ne k}}
|v_i^\top v_k|
\le
C\varepsilon_n,
\end{align*}
and moreover
\begin{align*}
\max_{i>r} \norm{u_i} \le 2,\qquad \max_{i>r} \norm{v_i} \le 2,\qquad
\norm{M}_{\mathsf F} \le C\sqrt{n}, \qquad \norm{\sum_{i=r+1}^n v_i} \le C\sqrt{n}.
\end{align*}
Let $\mathcal E$ denote the intersection of these events, so that $\P(\mathcal E)\to1$. We condition on this event.

Let~$\phi,\Phi$ denote the standard Gaussian density and cumulative density function, and define the (inverse) Mills ratio
\begin{align*}
R(z)
:=
\frac{\phi(z)}{1 - \Phi(z)}.
\end{align*}
For each $i>r$, let $z_i>0$ be the unique solution of
\begin{align*}
bR(z_i) = \sqrt d\,\norm{x_i},
\qquad b := 1-\frac{1}{\log n},
\end{align*}
and set $p_i := 1-\Phi(z_i)$. Since $R(z) = z+O(z^{-1})$, it holds
uniformly over $i>r$ that
\begin{align*}
\frac{z_i^2}{2\log n}= \frac{d\norm{x_i}^2}{2b^2\log n}+o(1) =
\frac{1}{2(1-r/n)b^2}
\frac{d^2}{n\log n} +o(1) \le \frac{c}{2} + o(1).
\end{align*}
In particular, there are fixed $\theta_0,\theta<1$ such that for all sufficiently large $n$, for every $i>r$,
\begin{align*}
\frac{z_i^2}{2\log n}
<\theta_0<\theta.
\end{align*}
By Mills' bound, for sufficiently large~$n$,
\begin{align*}
\frac{1}{p_i} \le C(1+z_i)e^{z_i^2/2} \le C\big(1+\sqrt{2\log n}\big) n^{\theta_0} \le n^\theta.
\end{align*}
Now define the array
\begin{align*}
\lambda_{ij}
:=
\frac{1}{r}
\Bigg[
\frac{1}{\log n}
+
\frac{b}{p_i}
1\left\{
\frac{\sqrt d\,u_j^\top x_i}{\norm{x_i}}
\ge
z_i
\right\}
\Bigg], \qquad i>r, \; j\le r.
\end{align*}
We proceed to bound $\mathcal A(\lambda)$. Condition on $U_1,v_1,\ldots,v_n$ for the
remainder of the proof. Let $u\sim\mathcal N(0,I_d/d)$ and define
\begin{align*}
\Xi(u)
:=
\sum_{i=r+1}^n c_i(u)(u_i-u)v_i^\top,\qquad c_i(u)
:=
\frac{1}{\log n}
+
\frac{b}{p_i}
1\left\{
\frac{\sqrt d\,u^\top x_i}{\|x_i\|}
\ge z_i
\right\}.
\end{align*}
By construction,
\begin{align*}
\mathcal A(\lambda)
=
\frac{1}{r}\sum_{j=1}^r\Xi(u_j).
\end{align*}
We first verify that $\Xi(u)$ is centered. Indeed,
\begin{align*}
\E[c_i(u)]
=
\frac{1}{\log n}+\frac{b}{p_i} \Pr(\mathcal N(0,1) \ge z_i) =1
\end{align*}
and
\begin{align*}
\E[uc_i(u)] = \frac{b}{p_i} \E\left[u 1\left\{
\frac{\sqrt d\,u^\top x_i}{\|x_i\|}
\ge z_i\right\} \right] = \frac{b}{p_i} \frac{\phi(z_i)x_i}{\sqrt d\,\|x_i\|} = x_i,
\end{align*}
hence $\E[\Xi(u)]=0$.

Next, we bound the second moment. By direct calculation,
\begin{align*}
\E[c_i(u)^2] = 1+b^2(p_i^{-1}-1).
\end{align*}
For $i\ne k$, it holds that
\begin{align*}
\Bigg(\frac{\sqrt d\,u^\top x_i}{\|x_i\|}, \frac{\sqrt d\,u^\top x_k}{\|x_k\|}\Bigg)^\top \sim 
\mathcal N\Bigg(
0,
\begin{pmatrix}
1 & \rho_{ik}\\
\rho_{ik} & 1
\end{pmatrix}
\Bigg), \qquad \rho_{ik}:=\frac{x_i^\top x_k}{\|x_i\|\|x_k\|}.
\end{align*}
By Plackett's identity \citep{plackett1954reduction},
\begin{align*}
&\left|
\P\left(
\frac{\sqrt d\,u^\top x_i}{\|x_i\|}\ge z_i,\,
\frac{\sqrt d\,u^\top x_k}{\|x_k\|}\ge z_k
\right)
-p_ip_k
\right| \le
C|\rho_{ik}|\phi(z_i)\phi(z_k),
\end{align*}
where we used $|\rho_{ik}|(z_i^2+z_k^2)\le C\varepsilon_n\log n=o(1)$. Hence
\begin{align*}
\left|\E[c_i(u)c_k(u)]-1\right|
&\le
C|\rho_{ik}|
\frac{\phi(z_i)\phi(z_k)}{p_ip_k}\\
&=
C|\rho_{ik}|R(z_i)R(z_k)\\
&\le
C\varepsilon_n\log n.
\end{align*}
Expanding yields
\begin{align*}
\|\Xi(u)\|_{\mathsf F}^2
\le
2\left\|
\sum_{i=r+1}^n c_i(u)u_iv_i^\top
\right\|_{\mathsf F}^2
+
2\|u\|^2
\left\|
\sum_{i=r+1}^n c_i(u)v_i
\right\|^2,
\end{align*}
where
\begin{align*}
\E\left[
\left\|
\sum_{i=r+1}^n c_i(u)u_iv_i^\top
\right\|_{\mathsf F}^2
\right]
&=
\|M\|_{\mathsf F}^2
+
b^2\sum_{i=r+1}^n
(p_i^{-1}-1)\|u_i\|^2\|v_i\|^2\\
&\qquad+
\sum_{\substack{i,k>r\\i\ne k}}
\bigl(\E[c_i(u)c_k(u)]-1\bigr)
(u_i^\top u_k)(v_i^\top v_k),
\\
\E\left[
\left\|
\sum_{i=r+1}^n c_i(u)v_i
\right\|^2
\right]
&=
\left\|
\sum_{i=r+1}^n v_i
\right\|^2
+
b^2\sum_{i=r+1}^n
(p_i^{-1}-1)\|v_i\|^2\\
&\qquad+
\sum_{\substack{i,k>r\\i\ne k}}
\bigl(\E[c_i(u)c_k(u)]-1\bigr)
v_i^\top v_k.
\end{align*}
On the event~$\mathcal E$ and also $\{\|u\|\le2\}$, using the above bounds and $\varepsilon_n^2\log n \le C(\log n)^2/d$, we obtain
\begin{align*}
\E\left[\|\Xi(u)\|_{\mathsf F}^2 1\{\|u\|\le 2\}
\right] \le C\left(n^{1+\theta} + n+ \frac{n^2(\log n)^2}{d}
\right).
\end{align*}
Moreover if $\|u\|>2$,
\begin{align*}
\left\| \sum_{i=r+1}^n c_i(u)v_i \right\| \le Cn^{1+\theta}, \qquad \E\left[
\|u\|^2 1\{\|u\|>2\}
\right] \le e^{-cd},
\end{align*}
so this contribution is negligible. Hence the same bound holds in expectation. It follows that
\begin{align*}
\E[\|\mathcal A(\lambda)\|_{\mathsf F}^2]
&=
\frac{1}{r}
\E[\|\Xi(u)\|_{\mathsf F}^2] \\
&\le
C\left(
n^\theta(\log n)^2+
(\log n)^2+\frac{n(\log n)^4}{d}
\right) \\
&= o\bigg(\left(\frac{n}{d\log n}\right)^2\bigg).
\end{align*}
By Markov's inequality conditioned on~$\mathcal E$, we conclude with probability tending to~$1$ that
\begin{align*}
\|\mathcal A(\lambda)\|_{\mathsf F}
= o\left(\frac{n}{d\log n}\right).
\end{align*}
The lemma is proved.
\end{proof}

We next show the following uniform lower bound.

\begin{lemma}\label{lem:certificate}
There is a constant $c_0>0$ such that, with probability tending to~$1$,
\begin{align*}
\inf_{\norm{X}_{\mathsf F}=1}
\sum_{i=r+1}^n\sum_{j=1}^r
|(u_i-u_j)^\top Xv_i|
\ge
c_0\frac{rn}{d}.
\end{align*}
Moreover on this event,
\begin{align*}
c_0\frac{rn}{d}
\{
X\in\R^{d\times d}:
\norm{X}_{\mathsf F}\le1\} \subseteq \{\mathcal A(\delta):
|\delta_{ij}|\le 1\}.
\end{align*}
\end{lemma}

\begin{proof}
We first prove a uniform lower bound using the first $r$ vectors $u_1,\ldots,u_r$. By symmetrization and the contraction inequality,
\begin{align*}
\sup_{\norm{x}^2+s^2=1}
\Bigg|\frac{1}{r}\sum_{j=1}^r
|\sqrt d\,u_j^\top x-s|
-
\Eb{g\sim\mathcal N(0,I_d)}{|
g^\top x-s|}
\Bigg|
=
O_{\P}\left(\sqrt{d/r}\right)
=
o_{\P}(1).
\end{align*}
Also,
\begin{align*}
\inf_{\norm{x}^2+s^2=1}
\Eb{g\sim\mathcal N(0,I_d)}{|g^\top x-s|}
=
\inf_{a^2+s^2=1}
\Eb{g\sim\mathcal N(0,1)}{|ag-s|} > 0,
\end{align*}
since the expectation is positive and continuous on the unit circle $a^2+s^2=1$. Therefore, with probability tending to~$1$, simultaneously for every $x\in\R^d$ and $s\in\R$,
\begin{align*}
\frac{1}{r}
\sum_{j=1}^r |\sqrt d\,u_j^\top x-s|
\ge
c_0\sqrt{\norm{x}^2+s^2} \ge c_0\|x\|.
\end{align*}
Taking $s=\sqrt d\,u_i^\top x$ thus gives
\begin{align*}
\sum_{j=1}^r |(u_i-u_j)^\top x| \ge
c_0\frac{r}{\sqrt d}\norm{x}, \qquad \forall x\in\R^d.
\end{align*}
A similar argument for the vectors $v_{r+1},\ldots,v_n$ also gives with probability tending to~$1$,
\begin{align*}
\sum_{i=r+1}^n
|x^\top v_i|
\ge
c_0\frac{n}{\sqrt d}\norm{x}
\qquad
\forall x\in\R^d.
\end{align*}
Now for any matrix $X$ with $\norm{X}_{\mathsf F}=1$, write the spectral decomposition as
\begin{align*}
X^\top X
=
\sum_{k=1}^d
\mu_ke_ke_k^\top,
\qquad
\sum_{k=1}^d\mu_k=1.
\end{align*}
For every $i>r$, by Cauchy--Schwarz,
\begin{align*}
\norm{Xv_i}
&=
\left(
\sum_{k=1}^d
\mu_k(e_k^\top v_i)^2
\right)^{1/2} \ge
\sum_{k=1}^d
\mu_k|e_k^\top v_i|.
\end{align*}
Summing over~$i$, from the preceding bounds, we have that
\begin{align*}
\sum_{i=r+1}^n\sum_{j=1}^r
|(u_i-u_j)^\top Xv_i |
&\ge
c_0\frac{r}{\sqrt d}
\sum_{i=r+1}^n\norm{Xv_i} \\
&\ge c_0\frac{r}{\sqrt d} \sum_{k=1}^d \mu_k \sum_{i=r+1}^n |e_k^\top v_i| \ge c_0\frac{rn}{d}.
\end{align*}
This proves the first claim. For the second claim, define the set
\begin{align*}
\mathcal K := \{\mathcal A(\delta):
|\delta_{ij}|\le 1\}.
\end{align*}
$\mathcal K$ is convex and compact, and it holds that
\begin{align*}
\sup_{Y\in\mathcal K}
\inprod{X,Y}_{\mathsf F}
&=
\sup_{|\delta_{ij}|\le 1}
\sum_{i=r+1}^n\sum_{j=1}^r
\delta_{ij}
(u_i-u_j)^\top Xv_i\\
&= \sum_{i=r+1}^n\sum_{j=1}^r
|(u_i-u_j)^\top Xv_i| \ge
c_0\frac{rn}{d}\norm{X}_{\mathsf F}.
\end{align*}
If there exists a matrix $Z\notin\mathcal K$ such that $\|Z\|_{\mathsf F} \le c_0rn/d$, by the separating hyperplane theorem, there must exist~$X$ such that
\begin{align*}
\inprod{X,Z}_{\mathsf F} > \sup_{Y\in\mathcal K}
\inprod{X,Y}_{\mathsf F} \ge \norm{X}_{\mathsf F} \norm{Z}_{\mathsf F},
\end{align*}
a contradiction.
\end{proof}

Now from Lemma~\ref{lem:array}, with probability tending to~$1$, there exists an array~$\lambda$ satisfying
\begin{align*}
\lambda_{ij} \ge \frac{1}{r\log n}, \qquad \|\mathcal A(\lambda)\|_{\mathsf F} = o\left(\frac{n}{d\log n}\right),
\end{align*}
while Lemma~\ref{lem:certificate} implies that, again with probability tending to~$1$,
\begin{align*}
\frac{c_0n}{2d\log n}
\{X:\norm{X}_{\mathsf F}\le1\} \subseteq \left\{
\mathcal A(\delta):
|\delta_{ij}|\le \frac{1}{2r\log n}
\right\}.
\end{align*}
It follows that there exists an array~$\delta$ satisfying
\begin{align*}
|\delta_{ij}| &\le \frac{1}{2r\log n},\qquad \mathcal A(\delta) = -\mathcal A(\lambda).
\end{align*}
Then $\mathcal A(\lambda+\delta) = 0$ by linearity and
\begin{align*}
\lambda_{ij}+\delta_{ij} \ge \frac{1}{r\log n} -\frac{1}{2r\log n} =\frac{1}{2r\log n} > 0.
\end{align*}
This completes the proof.
\end{proof}

\subsection{Finite-size scaling and a second-order conjecture}
\label{sec:top1-finite-size}

We conclude this section with a numerical study of the top-1 transition. Our goal is to compare its
location at moderate dimension with the leading-order threshold in
Theorem~\ref{thm:PR-main}. In our experiments, we learn the memory matrix using ridge-regularized multiclass softmax
cross-entropy,
\begin{equation}\label{eq:ce}
    \widehat W_{\mathsf{CE}}
    \bydef
    \argmin_{W \in \R^{d \times d}}
    \left\{
        \sum_{i=1}^n
        \left[-s_{i,i}+\log\sum_{j=1}^n e^{s_{j,i}}\right]
        +\frac{\lambda n}{2d^2}\norm{W}_{\mathsf F}^2
    \right\},
    \qquad \lambda>0.
\end{equation}
This objective provides a constructive search for a retrieving
memory. Indeed, if $W$ has minimum retrieval margin
$\gamma=\min_{i}\min_{j\ne i}(s_{i,i}-s_{j,i})>0$, then the
unregularized cross-entropy at $tW$ is at most
$n(n-1)e^{-t\gamma}$ and hence tends to zero as $t\to\infty$.
Consequently, as the effective ridge coefficient tends to zero, the
minimizer in \eqref{eq:ce} retrieves whenever a strict retrieving
memory exists: one can let $t\to\infty$ slowly enough that both the
cross-entropy and ridge penalty vanish, whereas any nonretrieving
matrix incurs cross-entropy at least $\log 2$.

On the feasible side, the norm needed to attain a fixed small loss
grows as the best retrieval margin shrinks. We therefore use
the observed peak of $\|\widehat W_{\mathsf{CE}}\|_{\mathsf F}/d$
under weak ridge regularization as an empirical proxy for the
transition. For each dimension, we sweep the load $c=n\log n/d^2$,
using the effective ridge coefficient $\lambda n/d^2=10^{-8}$. The
results are shown in Figure~\ref{fig:small_tail_conjecture_nlogn}.

\FloatBarrier

\begin{figure}[!t]
    \centering
    \includegraphics[width=\textwidth]{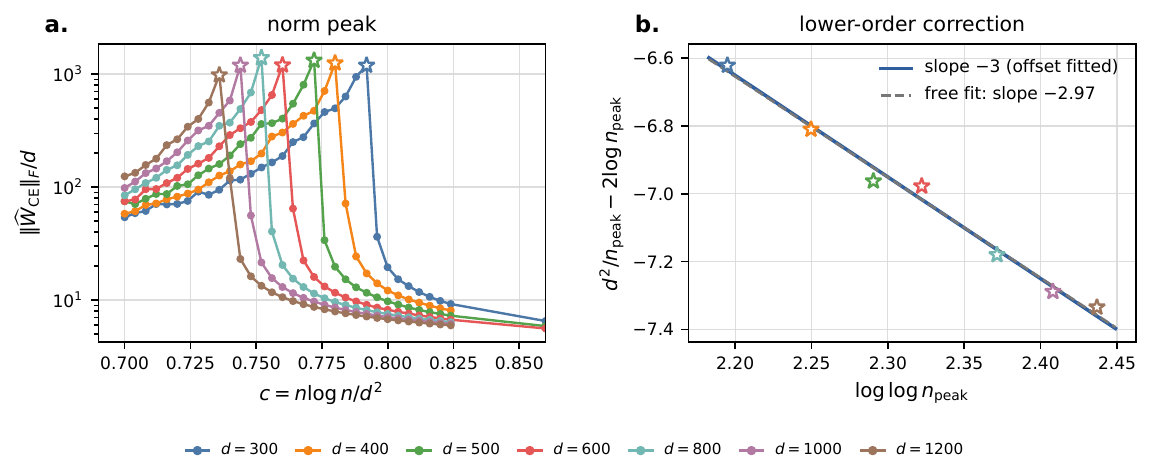}
    \caption{Finite-size transition diagnostic. \textbf{a.}
    Norm scale
    $\|\widehat W_{\mathsf{CE}}\|_{\mathsf F}/d$ along the
    cross-entropy sweep~\eqref{eq:ce}; stars mark the largest value on
    the sampled grid for each dimension. \textbf{b.} The same seven
    grid-selected peaks in the linearized coordinates used below;
    the solid and dashed lines are the fixed-slope and unconstrained
    fits, respectively.}
    \label{fig:small_tail_conjecture_nlogn}
\end{figure}

\FloatBarrier

Because $c$ is the reciprocal of the ratio $d^2/(n\log n)$ used in
Theorem~\ref{thm:PR-main}, the critical value is $1/2$ in the present
normalization rather than $2$. Panel~\textbf{a} shows a pronounced peak at every
dimension. Denoting its location by $\widehat c(d)$, we find that
$\widehat c(d)$ lies above $1/2$ and drifts only slowly toward it
as $d$ increases. Thus while Theorem~\ref{thm:PR-main} captures the leading location of the transition, the simulations exhibit a substantial
finite-size correction. This motivates the following conjecture for the second-order correction to the threshold.

\begin{conjecture}[Second-order correction to the top-1 threshold]
\label{conj:top1-finite-size} Let $d_{\mathrm c}(n)$ be the smallest integer $d$ for
which $\P(\mathsf{PR}_n)\ge 1/2$. Then
\begin{align}\label{eq:lower-order}
d_{\mathrm c}^2(n)
    =
    n\bigl(2\log n-3\log\log n+o(\log\log n)\bigr).
\end{align}
\end{conjecture}

Conjecture~\ref{conj:top1-finite-size} is suggested by a simpler model.
In the original problem, the $n^2$ score directions
$\operatorname{vec}(u_jv_i^\top)=v_i\otimes u_j$ have a correlated
Kronecker structure. Suppose instead that we replace them by $n^2$
independent isotropic Gaussian vectors in $\R^{d^2}$, while leaving the
top-1 condition~\eqref{eq:top1} unchanged. The score map then has
independent Gaussian rows, which allows us to apply the standard conic
kinematic formula~\citep{amelunxen2014living}. The resulting
statistical-dimension calculation gives~\eqref{eq:lower-order}.
We omit this calculation, since it does not apply directly to the
original model. Proving the conjecture for linear associative memory is left to future work.

The conjecture also appears to be supported by the simulations. Conjecture~\ref{conj:top1-finite-size} predicts the threshold to satisfy
\[
    \frac{d_{\mathrm c}^2(n)}{n}-2\log n
    =-3\log\log n+o(\log\log n).
\]
Panel~\textbf{b} plots the numerical peak locations according to the above formula. The fitted slope $-2.97$ gives an excellent fit with the predicted slope $-3$. In particular,~\eqref{eq:lower-order} predicts the critical load~$c$ to converge to $1/2$ at a rate of $O(\frac{\log\log n}{\log n})$, explaining the slow convergence.

\section{Listwise retrieval and the Tail-Average Margin}
\label{sec:tam}

Top-1 retrieval has the finest possible resolution: each query must rank
its matched target first. This is natural for final decoding. At an
intermediate stage, however, it may be enough to keep the target among
the strongest candidates and let later computation refine the choice.
This motivates the coarser, listwise notion studied here.

A direct listwise requirement would ask the signal score to beat a high
competitor quantile, thereby ensuring that the correct target belongs
to a controlled candidate set. This rank condition is natural, but it
is non-convex and combinatorial. We therefore use a convex certificate:
we ask the signal score to beat the \emph{average} of the strongest
competitors. We call the resulting functional the \emph{Tail-Average
Margin} (TAM). TAM is a listwise retrieval criterion in its own right:
it controls the mass of strong distractors rather than only the single
largest one, while retaining enough convex structure for optimization
and asymptotic analysis.

In this section we define the TAM functional, record its basic
properties, and formulate the corresponding learning problem and
performance criteria for linear associative memory. In the next
section, we derive an asymptotic theory of TAM learning in the
quadratic regime $d^2\asymp n$, with a fixed tail fraction
$r\in(0,1)$.

\subsection{The Tail-Average Margin functional}
\label{sec:tam-functional}

Let $x = (x_1, \ldots, x_m) \in \R^m$ and fix a tail fraction
$r \in (0, 1]$. Set $k \bydef \lceil r m \rceil$.

\begin{definition}[Tail-Average Margin]\label{def:tam}
The \emph{Tail-Average Margin} (TAM) at level $r$ is
\begin{equation}\label{eq:tam-def}
    \TAM_r(x) \bydef \frac{1}{k} \sum_{j=1}^{k} x_{(j)},
\end{equation}
where $x_{(1)} \ge x_{(2)} \ge \cdots \ge x_{(m)}$ are the order
statistics of $x$.
\end{definition}

We note two boundary cases:
$\TAM_{1/m}(x) = \max_j x_j$ (take $k = 1$), and
$\TAM_{1}(x) = \tfrac{1}{m} \sum_j x_j$ (take $k = m$). For fixed
$r \in (0, 1)$, TAM then interpolates strictly between the maximum and
the mean by averaging over the top $r$-fraction of entries.

The TAM criterion is a conservative certificate for top-$k$ list
retrieval. Indeed, since
\begin{equation}
    \TAM_r(x) = \frac{1}{k}\sum_{j=1}^k x_{(j)} \ge x_{(k)},
\end{equation}
the inequality $s_{i,i} > \TAM_r((s_{j,i})_{j\ne i})$ implies
$s_{i,i} > s_{(k),i}$ for the ordered competitor scores in column
$i$. Hence at most $k-1$ competitors can outrank the signal. In this
sense, exact top-1 retrieval implies TAM retrieval, and TAM retrieval
implies inclusion of the correct item in the top-$k$ candidate list.

The TAM functional is essentially a reparametrization of the empirical
\emph{Conditional Value-at-Risk} (CVaR) studied extensively in the
quantitative risk literature; see e.g., \citet{rockafellar2000}.
A classical observation is that TAM admits a
convex variational representation,
\begin{equation}\label{eq:tam-variational}
    \TAM_r(x) \;=\; \min_{\mu \in \R}
    \Bigg\{ \mu + \frac{1}{k} \sum_{j=1}^{m} (x_j - \mu)_+ \Bigg\},
\end{equation}
where $(x)_+ = \max\set{x, 0}$ is the hinge function.
The minimizer $\mu^\star = x_{(k)}$ of \eqref{eq:tam-variational} is the empirical
$(1-r)$-quantile of $x$, also known as the \emph{Value-at-Risk}. The
optimal value of \eqref{eq:tam-variational} is then the mean of the
exceedances of $x$ above this quantile, which recovers the tail-average
interpretation of \eqref{eq:tam-def}. As an immediate consequence of
\eqref{eq:tam-variational}, we see that $x \mapsto \TAM_r(x)$ is convex on~$\R^m$, as it is the pointwise minimum over~$\mu$ of functions jointly convex in~$(x,\mu)$.

For optimization it is convenient to replace the piecewise-linear
hinge $(\cdot)_+$ in \eqref{eq:tam-variational} with a differentiable
surrogate. For a smoothing parameter $\beta > 0$, let
\begin{equation}
    \phi_\beta(t) \;\bydef\; \frac{1}{\beta} \log(1 + e^{\beta t})
\end{equation}
denote the softplus function, and define the \emph{smoothed} TAM
\begin{equation}\label{eq:sTAM}
    \TAM_{r, \beta}(x)
    \;\bydef\;
    \min_{\mu \in \R}
    \Bigg\{
        \mu + \frac{1}{k} \sum_{j=1}^{m} \phi_\beta(x_j - \mu)
    \Bigg\}.
\end{equation}
Since $\phi_\beta$ is convex, $\TAM_{r,\beta}$ is convex in $x$;
moreover $\phi_\beta(t) \to (t)_+$ pointwise as $\beta \to \infty$,
and hence $\TAM_{r,\beta}(x) \to \TAM_r(x)$ in the same limit. The
exact TAM is recovered in the limiting notation $\TAM_{r, \infty} =
\TAM_r$.

Throughout the remainder of the paper we specialize TAM to the
linear associative memory setting: the collection
$x_1, \ldots, x_m$ is taken to be the competing scores
$(s_{j,i})_{j \ne i}$ associated with a query $v_i$, so that
$m = n - 1$ and $\TAM_r$ quantifies the strength of the top $r$-fraction
of competitors. The next subsection turns this into a convex learning
objective.

\subsection{Learning under the TAM criterion}
\label{sec:tam-learning}

With the TAM functional in hand, we now formulate retrieval and
learning for the linear associative memory as a convex optimization
problem.

Given stored pairs $\{(u_i, v_i)\}_{i=1}^n\subset\R^d\times\R^d$ and a memory or weight matrix
$W \in \R^{d \times d}$, recall that the score matrix
is $s_{j,i} = u_j^\top W v_i$. Fix a tail fraction $r \in (0, 1)$ and set
$k = \lceil r(n-1) \rceil$. The \emph{TAM retrieval criterion} asks
that each signal score exceed the TAM-aggregate of its competitors:
\begin{equation}\label{eq:tam-criterion}
    s_{i,i} \;>\; \TAM_r ((s_{j,i})_{j \ne i} ),
    \qquad \forall\, i \in [n].
\end{equation}
For $r = 1/(n-1)$, \eqref{eq:tam-criterion} reduces to the exact top-1
criterion~\eqref{eq:top1} studied in \sref{top1}; for fixed
$r \in (0, 1)$, it imposes an upper-tail margin requirement that
controls the mass of strong competitors rather than a single
extreme-value outlier.

Rather than demanding strict satisfaction of \eqref{eq:tam-criterion},
we minimize a convex surrogate that penalizes violations softly. Let
\begin{equation}
    \ell(t) \;\bydef\; \log (1 + e^{-t})
\end{equation}
denote the logistic loss, and let $\beta > 0$ be the smoothing
parameter of $\TAM_{r,\beta}$. We
consider the regularized objective
\begin{equation}\label{eq:tam-learning}
    \min_{W \in \R^{d \times d}}
    \Bigg\{
        \sum_{i=1}^n
        \ell \bigl(
            s_{i,i} - \TAM_{r, \beta} ((s_{j,i})_{j \ne i} )
        \bigr)
        + \frac{\lambda n}{2 d^2} \norm{W}_\mathsf{F}^2
    \Bigg\},
\end{equation}
where $\lambda \ge 0$ is a ridge penalty. Since each score
$s_{j,i}$ is linear in $W$ and $\TAM_{r, \beta}$ is convex, the
margin $s_{i,i} - \TAM_{r, \beta}((s_{j,i})_{j \ne i})$ is concave in $W$; composing
with the decreasing convex~$\ell$ yields a convex function of~$W$, so \eqref{eq:tam-learning} is a convex program (strongly
convex whenever $\lambda > 0$).

The variational representation~\eqref{eq:sTAM} of $\TAM_{r, \beta}$
suggests lifting \eqref{eq:tam-learning} to a joint optimization over
$W$ and the auxiliary quantile variables
$\mu = (\mu_1, \ldots, \mu_n) \in \R^n$. Substituting \eqref{eq:sTAM}
into \eqref{eq:tam-learning} yields
\begin{equation}\label{eq:tam-learning-joint}
    \min_{W, \mu}
    \Bigg\{
        \sum_{i=1}^n
        \ell \biggl(
            s_{i,i} - \mu_i
            - \frac{1}{k} \sum_{j \ne i}
                \phi_\beta(s_{j,i} - \mu_i)
        \biggr)
        + \frac{\lambda n}{2 d^2} \norm{W}_\mathsf{F}^2
    \Bigg\},
\end{equation}
Minimizing \eqref{eq:tam-learning-joint} over $\mu$ at fixed $W$
recovers \eqref{eq:tam-learning}. This is the workhorse formulation
for the leave-one-out analysis of \sref{tam_theory}.

It is useful to compare \eqref{eq:tam-learning} with the
ridge-regularized multiclass softmax cross-entropy~\eqref{eq:ce}, the
canonical convex surrogate for the top-1 criterion~\eqref{eq:top1}
introduced in \sref{top1-finite-size}.
Denote the log-sum-exp aggregate of the competitor scores by
$\mathrm{LSE} ((s_{j,i})_{j \ne i})
\bydef \log \sum_{j \ne i} e^{s_{j,i}}$.
A short calculation rewrites each summand in \eqref{eq:ce} as
\begin{equation}\label{eq:ce-structural}
    -\,s_{i,i} + \log \sum_{j=1}^n e^{s_{j,i}}
    \;=\;
    \ell\Bigl(
        s_{i,i} - \mathrm{LSE}\bigl((s_{j,i})_{j \ne i}\bigr)
    \Bigr),
\end{equation}
with the same logistic loss $\ell$ used in \eqref{eq:tam-learning}. In
this form, \eqref{eq:ce} has precisely the same structure as the TAM
objective \eqref{eq:tam-learning}: a logistic loss applied to the
margin between the signal score $s_{i,i}$ and a convex aggregate of
the competitor scores $(s_{j,i})_{j \ne i}$. They differ only in the
choice of aggregate. The log-sum-exp is a smoothed maximum, with
$\mathrm{LSE}((s_{j,i})_{j \ne i}) = \max_{j \ne i} s_{j,i}
+ O(\log n)$, and therefore inherits the $\sqrt{2 \log n}$
extreme-value statistics of the top-1 criterion. Replacing it with
the tail average $\TAM_{r, \beta}$, which remains on the scale of a
typical competitor at fixed $r$, enables retrieval at the
quadratic model-size scale $d^2 \asymp n$ instead of $d^2 \asymp n\log n$.

\subsection{Performance metrics}
\label{sec:tam-metrics}

With the TAM learning problem in place, we consider three
asymptotic quantities that together characterize retrieval
performance: the empirical average fitting loss, the per-sample TAM
margin, and the distribution of signal percentiles. Throughout this
subsection, $(\Wh, \muh)$ denotes the unique minimizer of
\eqref{eq:tam-learning-joint}, with corresponding scores
$\widehat s_{j,i} = u_j^\top \Wh v_i$.

\paragraph{Empirical average fitting loss.}
Let
\begin{equation}\label{eq:Ln-def}
    L_n(W, \mu) \;\bydef\; \frac{1}{n} \sum_{i=1}^n
    \ell\left(
        s_{i,i} - \mu_i
        - \frac{1}{k} \sum_{j \ne i}
            \phi_\beta(s_{j,i} - \mu_i)
    \right)
\end{equation}
denote the normalized data-fitting part of
\eqref{eq:tam-learning-joint}, and write
$\widehat L_n \bydef L_n(\Wh, \muh)$ for its value at the
optimizer. We will say that the TAM problem is \emph{asymptotically
fit at load}~$\alpha$ if $\widehat L_n \to 0$ in probability as
$n, d \to \infty$ with $n/d^2 \to \alpha$ and $\lambda\to0$.
Since $\ell$ is nonnegative and vanishes
only in the large-margin limit, asymptotic fit is equivalent to the
TAM retrieval criterion~\eqref{eq:tam-criterion} holding with a
positive margin for all but a vanishing fraction of keys.

\paragraph{TAM margin distribution.}
A finer summary than $\widehat L_n $ is the empirical distribution of
the per-sample TAM margins
\[
    \widehat m_i \;\bydef\; \widehat s_{i,i} - \muh_i
    - \frac{1}{k} \sum_{j \ne i}
        \phi_\beta (\widehat s_{j,i} - \muh_i ),
    \qquad i \in [n],
\]
which are precisely the arguments of the logistic loss appearing in
$\widehat L_n$. A positively-supported limiting distribution for $\widehat m_i$
indicates asymptotic fit; its shape shows how comfortably the
signal beats the tail average.

\paragraph{Signal percentile distribution.}
As a complementary, rank-based summary, we also ask, for each sample,
what fraction of competitors have score below the signal.
Formally, define the \emph{signal percentile}
\[
    \Omega_i \;\bydef\; \frac{1}{n-1} \sum_{j \ne i}
    \mathbf{1}\{\widehat s_{j,i} \le \widehat s_{i,i}\},
    \qquad i \in [n],
\]
so that $\Omega_i$ close to $1$ signals that the $i$th signal
dominates nearly all of its competitors and $\Omega_i = 0$ that
every competitor beats it. The random probability measure
\begin{equation}\label{eq:nu-def}
    \nu_n \;\bydef\; \frac{1}{n} \sum_{i=1}^n \delta_{\Omega_i}
\end{equation}
on $[0,1]$ is our object of study: when does $\nu_n$ converge in
probability (say, in Wasserstein-$1$) to a deterministic limit?

\paragraph{CMM benchmark: from top-1 to percentiles.}
Section~\ref{sec:superposition} used a Gaussian approximation to the
scores of the non-optimized CMM $W_{\mathsf{CMM}}$ to explain its
top-1 behavior. Here we return to the same picture, now in the
proportional regime $n/d^2\to\alpha\in(0,\infty)$, to predict the full
distribution of signal percentiles. In particular,
\eqref{eq:cmm-score-heuristic} suggests, for a fixed query,
\[
    s_{i,i}\approx 1+\sqrt{\alpha}\,Z_i,
    \qquad
    s_{j,i}\approx \sqrt{\alpha}\,Z_{j,i},\quad j\ne i,
\]
where $Z_i$ and $(Z_{j,i})_{j\ne i}$ are treated as independent
standard Gaussians. As discussed in Section~\ref{sec:superposition},
Lemma~\ref{lem:single-coord-law} gives the exact one-column
representation underlying this approximation: the competitor vector
is a random radius times $g/\|g\|$, and
$\sqrt n\,g/\|g\|$ is approximately Gaussian coordinatewise, while
the radius concentrates at $\sqrt\alpha$.

Conditioning on $Z_i$ and applying a law-of-large-numbers heuristic to
the competitors gives
\[
    \Omega_i
    \approx
    \frac{1}{n-1}\sum_{j\ne i}
    \mathbf 1\!\left\{Z_{j,i}\le Z_i+\alpha^{-1/2}\right\}
    \approx
    \Phi\!\left(Z_i+\alpha^{-1/2}\right).
\]
Consequently, the predicted CDF of the CMM percentile is
\begin{equation}\label{eq:percentile-copula}
    F_\alpha(\omega) \;\bydef\;
    \Phi\!\left(\Phi^{-1}(\omega) - \alpha^{-1/2}\right),
    \qquad \omega \in (0, 1),
\end{equation}
with density $f_\alpha(\omega)=\exp(\Phi^{-1}(\omega)/\sqrt{\alpha}
-1/(2\alpha))$. The parameter $\alpha^{-1/2}$ plays the role of a
signal-to-noise ratio: the law concentrates at $1$ as $\alpha\to0$
and tends to $\mathrm{Unif}(0,1)$ as $\alpha\to\infty$.

The next section develops a leave-one-out theory for these three
observables under the TAM optimizer in the quadratic regime
$n/d^2\to\alpha$ with fixed $r\in(0,1)$.

\section{Asymptotic theory of TAM}
\label{sec:tam_theory}

We now study the TAM optimizer in the quadratic regime $n,d\to\infty$
with $n/d^2\to\alpha$ and $k/n\to r\in(0,1)$. Throughout this section
the smoothing parameter $\beta<\infty$ and ridge parameter $\lambda>0$
are fixed. The goal is to predict the empirical laws of the true scores
$s_{i,i}(\Wh)$ and competitor scores $s_{j,i}(\Wh)$, and hence the
limiting TAM loss, margin distribution, and signal-percentile profile
introduced in \sref{tam-metrics}.

The main technique is a leave-one-out (LOO) analysis. We compare the TAM
optimizer with the optimizer obtained after removing a held-out pair
$(u_h,v_h)$, and use this comparison to separate the fresh randomness
of the held-out sample from the remaining problem. The first three
inputs below are LOO stability estimates, proved in
Appendix~\ref{app:loo_stability_proofs}. The remaining three are
spectral and self-averaging estimates: the spectral estimate is proved in
Appendix~\ref{app:frobenius_bounds}, while the two self-averaging estimates are proved in Appendix~\ref{app:spectral_self_averaging}. Together, the six estimates yield the two-parameter scalar variational description stated in
Theorem~\ref{thm:tam_scalar}.

The section is organized as follows. We first compute the gradient and
Hessian of the smoothed TAM loss. We then formulate the LOO rank-one
reduction, state the scalar variational prediction, and finally derive
the finite-dimensional equations that lead to this prediction. The
ridgeless limit $\lambda\downarrow0$ and the associated phase transition
are deferred to \sref{tam_phase_transitions}.

\subsection{Gradient and Hessian of the TAM loss}
\label{sec:tam_gradient}

Throughout this section the threshold variables are unregularized, and we work with the joint TAM objective of \eqref{eq:tam-learning-joint}, written with the reparameterized ridge:
\begin{equation}\label{eq:loss_stam_joint}
    \Psi(W, \mu) = \sum_{i=1}^n
\ell\Big(
s_{i,i}(W) - \mu_i - \frac{1}{k}\sum_{j\ne i}\phi_\beta\!\bigl(s_{j,i}(W)-\mu_i\bigr)
\Big)
+ \frac{\lambda n}{2 d^2}\,\norm{W}_\mathsf{F}^2,
\end{equation}
where $s(W)=U^\top W V\in\R^{n\times n}$, with entries
$s_{j,i}(W)=u_j^\top Wv_i$, and $U=[u_1\;\cdots\;u_n]$,
$V=[v_1\;\cdots\;v_n]$ are the data matrices from the Gaussian
model~\eqref{eq:gaussian_model}. We use lowercase $s_{j,i}$ for
finite-dimensional scores, reserving uppercase $S$ for scalar limiting
random variables. When the underlying matrix is clear from context, we
write $s_{j,i}$ instead of $s_{j,i}(W)$; in particular, at the full
optimizer this means $s_{j,i}=s_{j,i}(\Wh)$. Since $\Psi$ depends on
$W$ only through $s(W)$, the chain rule reduces the computation of
$\nabla_W\Psi$ to an $n\times n$ matrix of score derivatives,
\begin{equation}\label{eq:grad_W_chain}
    \nabla_W \Psi \;=\; U\,G\,V^\top \;+\; \frac{\lambda n}{d^2}\,W,
    \qquad
    G_{j,i} := \frac{\partial \Psi}{\partial s_{j,i}}.
\end{equation}
The remainder of this subsection evaluates the entries of $G$ and rewrites $\nabla_W\Psi$ in a form convenient for the leave-one-out analysis in \sref{tam_loo}.

The $i$-th summand of $\Psi$ equals $\ell(m_i)$, where the \emph{TAM aggregate} and \emph{margin} associated with sample $i$ are
\begin{equation}\label{eq:c_i_m_i}
    c_i := \mu_i + \frac{1}{k}\sum_{j\ne i}\phi_\beta\!\big(s_{j,i}-\mu_i\big),
    \qquad
    m_i := s_{i,i} - c_i.
\end{equation}
The margin $m_i$ is the only part of the $i$-th summand that depends on column $i$ of the score matrix, and the $i$-th summand is the only summand that does. Differentiating $\ell(m_i)$ in $s_{j,i}$ therefore yields
\begin{equation}\label{eq:dPsi_ds_raw}
    G_{j,i} \;=\; \ell'(m_i)\,\frac{\partial m_i}{\partial s_{j,i}}
    \;=\;\begin{cases}
        \ell'(m_i), & j = i, \\[2pt]
        -\tfrac{1}{k}\,\ell'(m_i)\,\phi_\beta'(s_{j,i}-\mu_i), & j \ne i.
    \end{cases}
\end{equation}
Both $\ell'$ and $\phi_\beta'$ are conveniently expressed through the sigmoid $\sigma(t) := (1+e^{-t})^{-1}$:
\begin{equation}\label{eq:sigma_derivs}
    \ell'(t) \;=\; \sigma(t)-1 \;=\; -\,\sigma(-t),
    \qquad
    \phi_\beta'(t) \;=\; \sigma(\beta t).
\end{equation}
Substituting into \eqref{eq:dPsi_ds_raw} gives the two scalar weights
\begin{equation}\label{eq:a_def}
    a_i \;:=\; \sigma(-m_i) \;=\; -\,\ell'(m_i) \;\in\; (0,1),
\end{equation}
and, for $j\ne i$,
\begin{equation}\label{eq:q_def}
    q_{j,i} \;:=\; \sigma\!\big(\beta\,(s_{j,i}-\mu_i)\big) \;=\; \phi_\beta'(s_{j,i}-\mu_i) \;\in\; (0,1), \qquad j\ne i.
\end{equation}
In these variables \eqref{eq:dPsi_ds_raw} becomes $G_{j,i} = -a_i$ for $j=i$ and $G_{j,i} = \tfrac{1}{k}q_{j,i}a_i$ otherwise. Equivalently $G = Q\Lambda_a$, where $\Lambda_a := \diag(a_1,\dots,a_n)$ and
\begin{equation}\label{eq:Q_def}
    Q_{j,i} \;:=\;
    \begin{cases}
      -1, & j = i,\\[2pt]
      \dfrac{1}{k}\, q_{j,i}, & j \ne i.
    \end{cases}
\end{equation}

Substituting $G = Q\Lambda_a$ into \eqref{eq:grad_W_chain} and performing the analogous calculation for $\mu$ --- where $\partial m_i/\partial \mu_i = -1 + \tfrac{1}{k}\sum_{j\ne i} q_{j,i} = -(Q^\top\mathbf{1}_n)_i$ --- yields
\begin{equation}\label{eq:grad_W}
    \nabla_W \Psi(W,\mu) = U\,Q\Lambda_a\,V^\top \;+\; \frac{\lambda n}{d^2}\,W,
\end{equation}
\begin{equation}\label{eq:grad_mu}
    \nabla_\mu \Psi(W,\mu) = -\Lambda_a\, Q^\top \mathbf{1}_n,
\end{equation}
where $\mathbf{1}_n \in \R^n$ is the all-ones vector. Expanding \eqref{eq:grad_W} by rows of $Q$ separates the diagonal and off-diagonal contributions,
\begin{equation}\label{eq:grad_W_expanded}
    \nabla_W \Psi \;=\; -\sum_{i=1}^n a_i\, u_i v_i^\top
    \;+\; \frac{1}{k}\sum_{i=1}^n a_i \sum_{j\ne i} q_{j,i}\, u_j v_i^\top
    \;+\; \frac{\lambda n}{d^2}\,W.
\end{equation}

For the Taylor expansion in \sref{tam_loo} we will also need the second derivative of the outer loss at the margin, which is governed by a companion diagonal:
\begin{equation}\label{eq:b_def}
    b_i \;:=\; \ell''(m_i) \;=\; a_i(1-a_i),
    \qquad \Lambda_b := \diag(b_1,\dots,b_n).
\end{equation}

\begin{definition}[Full Hessian]\label{def:full_hessian}
Let $H:\R^{d\times d}\to\R^{d\times d}$ denote the Hessian operator of $\Psi$ in the $W$ variable, evaluated at a fixed point $(W,\mu)$. Applied to a perturbation $\Delta$, it is
\begin{equation}\label{eq:H_full_def}
    H[\Delta]
    \;:=\; \frac{\lambda n}{d^2}\,\Delta
    \;+\; U\,M\!\bigl(U^\top\Delta V\bigr)\,V^\top.
\end{equation}
Here $M:\R^{n\times n}\to\R^{n\times n}$ acts columnwise,
\begin{equation}\label{eq:M_full_operator}
    M(X)_{:,i} \;=\; M_i X_{:,i},
\end{equation}
where
\begin{equation}\label{eq:M_full_i_def}
\begin{aligned}
    M_i
    \;&:=\; b_i Q_{:,i}Q_{:,i}^{\top}
    \;+\; a_i\diag(\widetilde R_{:,i}),\\
    \widetilde R_{j,i}
    \;&:=\;
    \begin{cases}
    \dfrac{\beta}{k}\,q_{j,i}(1-q_{j,i}), & j\ne i,\\
    0, & j=i.
    \end{cases}
\end{aligned}
\end{equation}
\end{definition}

The operator $H$ is self-adjoint and strictly positive-definite, with smallest eigenvalue at least $\lambda n/d^2$.
Under vectorization, writing $\mathcal D:=\diag(M_1,\ldots,M_n)\in\R^{n^2\times n^2}$, the same operator has matrix representation
\begin{equation}\label{eq:H_full_kron}
    H
    \;=\; \frac{\lambda n}{d^2}\,I_{d^2}
    \;+\; (V\otimes U)\,\mathcal D\,(V\otimes U)^\top.
\end{equation}

\subsection{Leave-one-out reduction}
\label{sec:tam_loo}

We relate the joint minimizer $(\Wh, \muh)$ of \eqref{eq:tam-learning-joint} to a leave-one-out reference point by a first-order expansion. Fix a held-out index $h \in [n]$, and let $(\WLOO, \muLOO)$ denote the minimizer of the TAM objective with sample $h$ removed from both the outer sum and from all inner summations over interferers:
\begin{equation}\label{eq:Psi_LOO}
    \Psi^{\setminus h}(W, \mu) \;:=\; \sum_{i \ne h}
    \ell\Big(
        s_{i,i}(W) - \mu_i
        - \frac{1}{k} \!\!\sum_{\substack{j \ne i \\ j \ne h}}\!\! \phi_\beta\bigl(s_{j,i}(W) - \mu_i\bigr)
    \Big)
    + \frac{\lambda n}{2 d^2}\,\norm{W}_\mathsf{F}^2.
\end{equation}
Write $\Delta W := \Wh - \WLOO$, $\Delta \mu := \muh - \muLOO$, and $\Delta S := U^\top (\Delta W) V \in \R^{n \times n}$ for the corresponding perturbations. A hat decoration (e.g., $\hat a_i, \hat q_{j,i}, \hat b_i, \hat m_i$) denotes a quantity from \sref{tam_gradient} evaluated at $(\Wh, \muh)$; a superscript $\setminus h$ denotes the same quantity evaluated at $(\WLOO, \muLOO)$. We adopt the conventions $a_h^{\setminus h} := 0$ and $q_{j, h}^{\setminus h} := q_{h, i}^{\setminus h} := 0$, so that $Q^{\setminus h}, \Lambda_{a^{\setminus h}}$ extend naturally to $n \times n$ arrays.

\subsubsection{LOO Hessian and rank-one reduction}

The LOO Hessian has the same form as the full Hessian in Definition~\ref{def:full_hessian}, with the $h$-th sample removed and all coefficients evaluated at $(\WLOO,\muLOO)$. Explicitly, $H^{\setminus h}: \R^{d\times d} \to \R^{d\times d}$ is defined by
\begin{equation}\label{eq:H_LOO_def}
    H^{\setminus h}[\Delta] \;:=\; \frac{\lambda n}{d^2}\,\Delta
    \;+\; U\,M^{\setminus h}\!\bigl(U^\top \Delta\, V\bigr)\,V^\top,
\end{equation}
where $M^{\setminus h}: \R^{n\times n}\to \R^{n\times n}$ acts columnwise as
\begin{equation}\label{eq:M_operator}
    M^{\setminus h}(X)_{:, i} \;=\; M_i^{\setminus h}\,X_{:, i}\quad (i \ne h),
    \qquad M^{\setminus h}(X)_{:, h} \;=\; 0.
\end{equation}
In the above display, $M_i^{\setminus h} \in \R^{n\times n}$ is a symmetric matrix with a rank-one-plus-diagonal form
\begin{equation}\label{eq:M_i_def}
\begin{aligned}
    M_i^{\setminus h}
    \;&:=\; b_i^{\setminus h}\,Q_{:, i}^{\setminus h} (Q_{:, i}^{\setminus h})^\top
    \;+\; a_i^{\setminus h}\,\diag\bigl(\widetilde R_{:, i}^{\setminus h}\bigr)\\
    \widetilde R_{j, i}^{\setminus h} \;&:=\;\begin{cases}
    \frac{\beta}{k}\,q_{j, i}^{\setminus h}\bigl(1 - q_{j, i}^{\setminus h}\bigr)\,&j \ne i,\ j \ne h\\
    0, &\text{otherwise}\\
    \end{cases}.
\end{aligned}
\end{equation}
Here $b_i^{\setminus h}$ is the quantity in \eqref{eq:b_def} evaluated at $(\WLOO, \muLOO)$. The operator $H^{\setminus h}$ is self-adjoint and strictly positive-definite on $\R^{d \times d}$: its smallest eigenvalue is at least $\lambda n / d^2 = \Theta(1)$ in the regime $d^2 \asymp n$, so that $\norm{(H^{\setminus h})^{-1}}_\mathrm{op} = \Op(1)$.

\paragraph{Rank-one reduction.}
The leave-one-out perturbation is governed, to first order, by the held-out feature pair $(u_h,v_h)$.
\begin{proposition}[LOO rank-one reduction]\label{prop:loo_rank_one_reduction}
Under the LOO stability estimates \ref{post:loo-mu}--\ref{post:loo-bounded} stated below, with high probability,
\begin{equation}\label{eq:What_WLOO_expansion}
    \Wh - \WLOO
    \;=\; \sigma(-\hat m_h)\,\bigl(H^{\setminus h}\bigr)^{-1}\!\bigl[u_h v_h^\top\bigr]
    \;+\; R_h,
    \qquad
    \bigl\lVert R_h \bigr\rVert_\mathsf{F} \;=\; \Op(d^{-1/2}).
\end{equation}
\end{proposition}
Here $\hat m_h$ is the TAM margin at the held-out sample evaluated at the full minimizer $(\Wh, \muh)$:
\begin{equation}\label{eq:m_hat_h}
    \hat m_h \;=\; s_{h, h}(\Wh) \;-\; \hat\mu_h \;-\; \frac{1}{k}\sum_{j \ne h} \phi_\beta\!\bigl(s_{j, h}(\Wh) - \hat\mu_h\bigr).
\end{equation}

Thus the leading change in the weight matrix is the inverse LOO Hessian applied to $\sigma(-\hat m_h)\,u_h v_h^\top$, with a Frobenius-norm correction of lower order. In \sref{tam_selfconsistent}, this formula is combined with trace concentration for $(H^{\setminus h})^{-1}$ to obtain the susceptibility equation.
\subsubsection{Proof of Proposition~\ref{prop:loo_rank_one_reduction}}
We derive \eqref{eq:What_WLOO_expansion} from the following LOO
stability proposition.

\begin{proposition}[LOO stability]\label{prop:loo_stability}
Uniformly after removing any fixed finite set of samples,
\begin{enumerate}[label=\textup{(A\arabic*)}, ref=\textup{(A\arabic*)}, leftmargin=*]
  \item\label{post:loo-mu} $\displaystyle \max_{i \ne h}\, \bigl|(\Delta \mu)_i\bigr| \;=\; \Op(d^{-3/2})$;
  \item\label{post:loo-score} $\displaystyle \max_{j, i \ne h}\, \bigl|(\Delta S)_{j, i}\bigr| \;=\; \Op(d^{-1})$;
  \item\label{post:loo-bounded} $\displaystyle \max_{j, i \in [n]}\, \bigl|s_{j, i}(\Wh)\bigr| \;=\; \Op(1)$.
\end{enumerate}
\end{proposition}
The proposition is proved in
Appendix~\ref{app:loo_stability_proofs}. The estimates are in the sense of
stochastic domination, absorbing $d^\varepsilon$ factors. In
particular, \ref{post:loo-bounded} applies to the optimizer of each
finite-removal problem. These estimates control the Taylor remainder
and the coefficient perturbations in the LOO expansion \eqref{eq:What_WLOO_expansion}. 

\subparagraph*{Step 1: Subtract the stationarity conditions.}
The $W$-components of the stationarity conditions \eqref{eq:grad_W} at the full and LOO optimizers, subtracted, yield
\begin{equation}\label{eq:FOC_W_diff}
    \frac{\lambda n}{d^2}\,\Delta W
    \;+\; U\,\bigl(\hat Q \Lambda_{\hat a} - Q^{\setminus h} \Lambda_{a^{\setminus h}}\bigr)\, V^\top
    \;=\; 0.
\end{equation}

\subparagraph*{Step 2: Separate the terms involving the held-out sample.}
With the extension conventions on $a^{\setminus h}$ and $q^{\setminus h}$, the matrix $\hat Q \Lambda_{\hat a} - Q^{\setminus h} \Lambda_{a^{\setminus h}}$ decomposes entrywise as
\begin{equation}\label{eq:Rsetminus_explicit}
    \bigl(\hat Q \Lambda_{\hat a} - Q^{\setminus h} \Lambda_{a^{\setminus h}}\bigr)_{j, i}
    \;=\;
    \begin{cases}
      -\,\hat a_h & (j, i) = (h, h),                                    \\[2pt]
      \hat a_h\,\hat q_{j, h}/k       & j \ne h,\ i = h,                \\[2pt]
      \hat a_i\,\hat q_{h, i}/k       & j = h,\ i \ne h,                \\[2pt]
      \mathcal{R}_{j, i}^{\setminus h} & j, i \ne h,
    \end{cases}
\end{equation}
where the first three cases are the entries involving the held-out sample: the held-out outer term ($i = h$) and the held-out interferer row ($j = h$, across interior columns) are present in the full objective but absent in the LOO objective. The remaining entries are collected in $\mathcal{R}^{\setminus h} \in \R^{n \times n}$, supported on $\{j, i \ne h\}$:
\begin{equation}\label{eq:R_def}
    \mathcal{R}_{j, i}^{\setminus h}
    \;:=\;
    \begin{cases}
      \dfrac{\hat a_i\,\hat q_{j, i} \;-\; a_i^{\setminus h}\,q_{j, i}^{\setminus h}}{k}
        & j \ne i,\ j, i \ne h,  \\[8pt]
      a_i^{\setminus h} \;-\; \hat a_i
        & j = i \ne h.
    \end{cases}
\end{equation}
Separating these entries inside \eqref{eq:FOC_W_diff},
\begin{equation}\label{eq:FOC_W_sources}
    \frac{\lambda n}{d^2}\,\Delta W \;+\; U\,\mathcal{R}^{\setminus h}\,V^\top
    \;=\; \hat a_h\, u_h v_h^\top \;-\; \hat a_h\,\eta_h\, v_h^\top \;-\; u_h\,\xi_h^\top,
\end{equation}
where
\begin{equation}\label{eq:eta_xi_def}
    \eta_h \;:=\; \frac{1}{k}\sum_{j \ne h} \hat q_{j, h}\,u_j,
    \qquad
    \xi_h \;:=\; \frac{1}{k}\sum_{i \ne h} \hat a_i\,\hat q_{h, i}\,v_i.
\end{equation}

\subparagraph*{Step 3: Linearize the common part.}
In Appendix~\ref{app:tam_taylor}, we carry out a first-order Taylor expansion of $\mathcal{R}^{\setminus h}$ in the perturbations $\Delta S$ and $\Delta \mu$, with all higher-order terms collected into a residual matrix $\hat E^{\setminus h}$ supported on $\{j, i \ne h\}$. This yields
\begin{equation}\label{eq:interior_linearization}
    \mathcal{R}_{:, i}^{\setminus h} \;=\; M_i^{\setminus h}\,\Delta S_{:, i} \;+\; \hat E^{\setminus h}_{:, i}, \qquad i \ne h,
\end{equation}
with $M_i^{\setminus h}$ as in \eqref{eq:M_i_def}. Combining \eqref{eq:interior_linearization} with \eqref{eq:FOC_W_sources} and the definition of $H^{\setminus h}$ yields
\begin{equation}\label{eq:H_LOO_identity}
    H^{\setminus h}[\Delta W]
    \;=\; \hat a_h\,u_h v_h^\top
    \;-\; \hat a_h\,\eta_h\,v_h^\top
    \;-\; u_h\,\xi_h^\top
    \;-\; U \hat E^{\setminus h} V^\top.
\end{equation}
Moreover, we show in Appendix~\ref{app:tam_taylor} that, under \ref{post:loo-mu}--\ref{post:loo-bounded},
\begin{equation}\label{eq:Ehat_entry_sizes}
    \bigl|\hat E^{\setminus h}_{i, i}\bigr| \;=\; \Op(n^{-1})\ \text{for}\ i \ne h,
    \qquad
    \bigl|\hat E^{\setminus h}_{j, i}\bigr| \;=\; \Op(n^{-2})\ \text{for}\ j \ne i,\ j, i \ne h.
\end{equation}

\subparagraph*{Step 4: Invert and bound the correction.}
Applying $(H^{\setminus h})^{-1}$ to \eqref{eq:H_LOO_identity} produces \eqref{eq:What_WLOO_expansion} with
\begin{equation}\label{eq:R_h_explicit}
    R_h
    \;=\; \bigl(H^{\setminus h}\bigr)^{-1}\!\Bigl[
        -\hat a_h\,\eta_h\,v_h^\top
        \;-\; u_h\,\xi_h^\top
        \;-\; U \hat E^{\setminus h} V^\top
    \Bigr].
\end{equation}
Since $\norm{(H^{\setminus h})^{-1}}_\mathrm{op} = \Op(1)$, it suffices to bound each of the three remaining terms in \eqref{eq:R_h_explicit} in Frobenius norm at scale $\Op(d^{-1/2})$.

Writing $k\,\eta_h = U\,\hat q_{:, h}$ (with the $h$-th entry of $\hat q_{:, h}$ understood to be zero), the Marchenko--Pastur bound $\norm{U}_\mathrm{op} = \Op(\sqrt{d})$ in the regime $d^2 \asymp n$, together with the deterministic entrywise bound $\norm{\hat q_{:, h}}_2 \le \sqrt{n-1}$, gives $\norm{\eta_h}_2 \le \norm{U}_\mathrm{op}\,\norm{\hat q_{:, h}}_2 / k = \Op(\sqrt{nd}/n) = \Op(d^{-1/2})$; the same argument applies to $\xi_h$. Consequently
\begin{equation}\label{eq:struct_corr_sizes}
    \norm{\hat a_h\,\eta_h\,v_h^\top}_\mathsf{F}
    \;=\; |\hat a_h|\,\norm{\eta_h}_2\,\norm{v_h}_2
    \;=\; \Op(d^{-1/2}),
    \qquad
    \norm{u_h\,\xi_h^\top}_\mathsf{F} \;=\; \Op(d^{-1/2}).
\end{equation}

Next, we consider the term $U \hat E^{\setminus h} V^\top$ in \eqref{eq:R_h_explicit}. Let $\hat E^{\mathsf{d}}$ and $\hat E^{\mathsf{o}}$ denote the diagonal and off-diagonal parts of $\hat E^{\setminus h}$, respectively. The diagonal part is a diagonal matrix, so $\norm{\hat E^{\mathsf{d}}}_\mathrm{op} = \max_{i \ne h} |\hat E^{\setminus h}_{i, i}| = \Op(n^{-1})$ by \eqref{eq:Ehat_entry_sizes}. The off-diagonal part is controlled through its Frobenius norm,
\begin{equation}
    \norm{\hat E^{\mathsf{o}}}_\mathrm{op}
    \;\le\; \norm{\hat E^{\mathsf{o}}}_\mathsf{F}
    \;=\; \Bigl(\sum_{\substack{j \ne i\\ j, i \ne h}} |\hat E^{\setminus h}_{j, i}|^2\Bigr)^{1/2}
    \;\le\; \sqrt{n^2 \cdot \Op(n^{-4})}
    \;=\; \Op(n^{-1}),
\end{equation}
again by \eqref{eq:Ehat_entry_sizes}. Combining, $\norm{\hat E^{\setminus h}}_\mathrm{op} = \Op(n^{-1}) = \Op(d^{-2})$. Submultiplicativity with $\norm{U}_\mathrm{op}, \norm{V}_\mathrm{op} = \Op(\sqrt{d})$ yields $\norm{U \hat E^{\setminus h} V^\top}_\mathrm{op} \le \Op(d \cdot d^{-2}) = \Op(d^{-1})$. Since $U \hat E^{\setminus h} V^\top \in \R^{d \times d}$ has rank at most $d$, $\norm{\cdot}_\mathsf{F} \le \sqrt{d}\,\norm{\cdot}_\mathrm{op}$, and
\begin{equation}\label{eq:UEV_F_bound}
    \norm{U \hat E^{\setminus h} V^\top}_\mathsf{F}
    \;\le\; \sqrt{d}\,\norm{U \hat E^{\setminus h} V^\top}_\mathrm{op}
    \;=\; \Op(d^{-1/2}).
\end{equation}
Combining \eqref{eq:struct_corr_sizes} and \eqref{eq:UEV_F_bound} and applying $(H^{\setminus h})^{-1}$ establishes \eqref{eq:What_WLOO_expansion}. \qed

\subsection{Scalar characterization of the TAM optimizer}
\label{sec:tam_selfconsistent}

We now state the asymptotic characterization of the TAM minimizer. In
addition to Proposition \ref{prop:loo_stability}, we use the following
spectral and self-averaging proposition. The derivation is given in
\sref{tam_scalar_derivation}.

\begin{proposition}[Spectral and self-averaging estimates]
\label{prop:spectral_self_averaging}
We have the following:
\begin{enumerate}[label=\textup{(A\arabic*)}, ref=\textup{(A\arabic*)}, leftmargin=*, start=4] 

  \item\label{post:spectral-op} (Spectrum bound) \begin{equation}
      \norm{\Wh}_\mathrm{op} \;=\; \Op(\sqrt{d})
       \end{equation}
  \item\label{post:empirical-concentration} (Empirical-mean concentration) Let $\psi\in C^1(\R)$ be pseudo-Lipschitz of order two. Then,
\begin{equation}\label{eq:A5_concentration}
    \frac{1}{n}\sum_{i = 1}^n \psi\bigl(s_{i, i}(\Wh)\bigr)
    -
    \E\!\left[\frac{1}{n}\sum_{i = 1}^n \psi\bigl(s_{i, i}(\Wh)\bigr)\right]
    \;=\; \Op(d^{-1/2}),
\end{equation}
and
\begin{equation}
    \frac{1}{n(n - 1)}\sum_{i \ne j} \psi\bigl(s_{j, i}(\Wh)\bigr)
    -
    \E\!\left[\frac{1}{n(n - 1)}\sum_{i \ne j} \psi\bigl(s_{j, i}(\Wh)\bigr)\right]
    \;=\; \Op(d^{-1/2}).
\end{equation}
  \item\label{post:trace-concentration} (Trace concentration)
\begin{equation}\label{eq:A6_trace_concentration}
    \frac{1}{d^2}\,\tr\!\left(H^{-1}\right)
    -
    \E\!\left[
        \frac{1}{d^2}\,\tr\!\left(H^{-1}\right)
    \right]
    \;=\; \Op(d^{-1/2}).
\end{equation}
\end{enumerate}
\end{proposition}
The spectral estimate \ref{post:spectral-op} is proved in
Appendix~\ref{app:frobenius_bounds}; the empirical-mean and trace
concentration estimates \ref{post:empirical-concentration} and
\ref{post:trace-concentration} are proved in
Appendix~\ref{app:spectral_self_averaging}. Here \ref{post:spectral-op} is a
spectral estimate on the optimizer $\Wh$, while \ref{post:empirical-concentration} and
\ref{post:trace-concentration} are self-averaging estimates at the
common scale $\Op(d^{-1/2})$ for empirical means and normalized inverse
traces, respectively.

For $\nu>0$, define the scalar TAM functional
\begin{equation}\label{eq:c_r_nu_def}
    c_r(\nu)
    :=
    \min_{\mu\in\R}
    \left\{
    \mu+\frac{1}{r}\Eb{Z\sim \mathcal{N}(0,1)}{\phi_\beta(\nu Z-\mu)}
    \right\}.
\end{equation}
Let $\mu_r(\nu)$ denote the minimizer in \eqref{eq:c_r_nu_def}. It is characterized by
\begin{equation}\label{eq:mu_r_characterization}
    r
    =
    \Eb{Z\sim \mathcal{N}(0,1)}
    {\sigma\!\bigl(\beta(\nu Z-\mu_r(\nu))\bigr)}.
\end{equation}

\begin{definition}[Scalar saddle and reduced potentials]\label{def:scalar_potentials}
For $\nu,\chi>0$, define the scalar Moreau envelope
\begin{equation}\label{eq:scalar_moreau_def}
    M_r(\nu,\chi)
    :=
    \Eb{G\sim \mathcal{N}(0,1)}
    {
    \min_{s\in\R}
    \left[
    \frac{(s-\nu G)^2}{2\chi}
    +
    \ell\bigl(s-c_r(\nu)\bigr)
    \right]
    }.
\end{equation}
The scalar saddle potential is
\begin{equation}\label{eq:scalar_potential_def}
    \mathcal P_{\alpha,r}(\nu,\chi)
    :=
    \frac{\lambda}{2}\nu^2
    -
    \frac{\nu^2}{2\alpha\chi}
    +
    M_r(\nu,\chi),
\end{equation}
and the reduced scalar potential is
\begin{equation}\label{eq:reduced_potential_def}
    \mathcal V_{\alpha,r}(\nu)
    :=
    \sup_{\chi>0}\mathcal P_{\alpha,r}(\nu,\chi).
\end{equation}
\end{definition}
Proposition~\ref{prop:scalar_variational}, stated in \sref{tam_scalar_derivation}, shows that, for each fixed $\nu>0$, the maximization over $\chi$ is well posed and has a unique maximizer, and that the reduced potential $\mathcal V_{\alpha,r}$ is strongly convex in $\nu$.

For $\chi>0$, define
\begin{equation}\label{eq:prox_chi_def}
    \Prox_\chi(x)
    :=
    \argmin_{y\in\R}
    \left\{
    \frac{(y-x)^2}{2\chi}
    +
    \ell(y)
    \right\}.
\end{equation}
Equivalently, $y=\Prox_\chi(x)$ if and only if
\begin{equation}\label{eq:prox_chi_foc}
    y-\chi\sigma(-y)=x.
\end{equation}
If $G\sim \mathcal N(0,1)$, define
\begin{equation}\label{eq:scalar_channel_nu_chi}
    S_{\nu,\chi}
    :=
    c_r(\nu)+\Prox_\chi\!\bigl(\nu G-c_r(\nu)\bigr).
\end{equation}

Let $\nu_\ast$ be the minimizer of $\mathcal V_{\alpha,r}$, and let $\chi_\ast$ be the maximizer of $\mathcal P_{\alpha,r}(\nu_\ast,\cdot)$. Define $\mu_\ast$ as the minimizer in \eqref{eq:c_r_nu_def} at $\nu=\nu_\ast$, and set
\begin{equation}\label{eq:c_star_variational}
    c_\ast:=c_r(\nu_\ast).
\end{equation}
The variational facts used here, including uniqueness and the equivalence with the finite-dimensional equations derived below, are recorded in Proposition~\ref{prop:scalar_variational}.

\begin{theorem}[Scalar description of the joint TAM minimizer]\label{thm:tam_scalar}
Fix $(\alpha, r, \beta, \lambda)$ with $\beta,\lambda>0$. Assume that the minimizer $\nu_\ast$ of $\mathcal V_{\alpha,r}$ is positive. In the joint limit $n,d\to\infty$ with $n/d^2\to\alpha$ and $k/n\to r$, the following hold in probability:
\begin{enumerate}[label=\textup{(\roman*)}]
    \item\label{item:op-concentration} \textup{(Order-parameter concentration.)}
	\begin{equation}\label{eq:tam_order_parameter_prediction}
	    \frac{1}{d}\,\|\Wh\|_\mathsf{F}
	    \longrightarrow
	    \nu_\ast,
        \qquad
        \frac{1}{d^2}\,\tr\!\bigl(H^{-1}\bigr)
        \longrightarrow
        \chi_\ast.
\end{equation}

    \item\label{item:offdiag-law} \textup{(Gaussian law for competitor scores.)}
    For each $i\in[n]$, the empirical distribution
    $\displaystyle \frac{1}{n-1}\sum_{j\ne i}\delta_{s_{j,i}(\Wh)}$
    converges weakly to $\mathcal{N}(0,\nu_\ast^2)$.

    \item\label{item:diag-law} \textup{(Proximal law for diagonal scores.)}
    The empirical distribution
    $\displaystyle \frac{1}{n}\sum_{i=1}^{n}\delta_{s_{i,i}(\Wh)}$
    converges weakly to the law of $S_{\nu_\ast,\chi_\ast}$ in \eqref{eq:scalar_channel_nu_chi}.
\end{enumerate}
\end{theorem}

\begin{remark}[Reading off the retrieval metrics]
Theorem~\ref{thm:tam_scalar} directly answers the three performance questions posed in \sref{tam-metrics}. Since the TAM aggregate converges to $c_\ast$, the limiting margin is
\begin{equation}\label{eq:tam_margin_prediction}
    M_\ast
    :=
    S_{\nu_\ast,\chi_\ast}-c_\ast,
\end{equation}
and hence the average fitting loss converges to
\begin{equation}\label{eq:tam_loss_prediction}
    \E\!\left[\ell(M_\ast)\right].
\end{equation}
The empirical margin distribution converges to the law of $M_\ast$. Finally, writing $\Phi$ for the standard normal CDF, the signal percentile distribution converges to the law of
\begin{equation}\label{eq:tam_percentile_prediction}
    \Phi\!\left(\frac{S_{\nu_\ast,\chi_\ast}}{\nu_\ast}\right),
\end{equation}
because the competitor scores have limiting law $\mathcal{N}(0,\nu_\ast^2)$.
\end{remark}

\subsection{Proof of Theorem~\ref{thm:tam_scalar}}
\label{sec:tam_scalar_derivation}

Define
\begin{equation}\label{eq:nu_d_def}
    \nu_d^2
    \;:=\;
    \E\!\left[\,\frac{\|\Wh\|_\mathsf{F}^2}{d^2}\,\right].
\end{equation}
With $H$ denoting the full Hessian operator from Definition~\ref{def:full_hessian}, evaluated at $(\Wh,\muh)$, define
\begin{equation}\label{eq:chi_d_def}
    \chi_d
    \;:=\;
    \E\!\left[
        \frac{1}{d^2}\,\tr\!\left(H^{-1}\right)
    \right].
\end{equation}

In this section, we prove Theorem~\ref{thm:tam_scalar} by deriving the
finite-dimensional equations for $\nu_d$ and $\chi_d$. The proof
proceeds in four steps. First, using \ref{post:spectral-op} and the bilinear coupling estimate in Appendix~\ref{app:bilinear_coupling}, we identify the off-diagonal competitor scores with Gaussians of variance $\nu_d^2$. Second, this Gaussian approximation and the trace input \ref{post:trace-concentration} give the proximal characterization of the diagonal score $s_{i,i}(\Wh)$. Third, $W$-stationarity, together with empirical concentration \ref{post:empirical-concentration}, yields a finite-dimensional equation for $\nu_d$. Fourth, a random matrix calculation for a simplified Hessian gives the companion equation for $\chi_d$. Combining the two equations gives an approximate residual system, whose zero-residual limit is the saddle-point condition for the scalar potential in Definition~\ref{def:scalar_potentials}.

\subsubsection{Gaussian law for competitor scores}
\label{sec:tam_competitor_gaussian}

We begin with the off-diagonal scores. The statement below gives a finite-dimensional joint Gaussian coupling for any fixed number of competitor scores in a single column, using an iterated leave-one-out reduction and the bilinear coupling estimate of \corref{bilinear_gauss_coupling_rand}.

\begin{proposition}[Joint Gaussian limit for competitor scores]\label{prop:competitor_gaussian}
Fix $i \in [n]$, $m \ge 1$, and distinct $j_1, \ldots, j_m \in [n] \setminus \{i\}$. Assume \ref{post:loo-mu}--\ref{post:trace-concentration}. Then, in the joint regime $n, d \to \infty$ with $n/d^2 \to \alpha$ and $k/n \to r$:
on a common probability space, there exist i.i.d. Gaussians $Y_1,\ldots,Y_m\sim\mathcal N(0,\nu_d^2)$ such that
\begin{equation}\label{eq:competitor_coupling}
    \bigl\|\bigl(s_{j_1, i}(\Wh), \ldots, s_{j_m, i}(\Wh)\bigr) - (Y_1, \ldots, Y_m)\bigr\|_2 \;=\; \Op(d^{-1/2}).
\end{equation}
\end{proposition}

\begin{remark}[Connection with the limiting statement]
Proposition~\ref{prop:competitor_gaussian} is the finite-dimensional form of the off-diagonal Gaussian law in Theorem~\ref{thm:tam_scalar}. Empirical-measure consequences, such as weak convergence of $\frac{1}{n - 1}\sum_{j \ne i}\delta_{s_{j, i}(\Wh)}$ to $\mathcal{N}(0, \nu_d^2)$, follow by combining exchangeability with concentration of the optimizer. The expectation-level replacements used below are obtained from this coupling for the finite list of bounded or $C^1$ pseudo-Lipschitz observables appearing in the stationarity identities, together with the moment bounds implicit in \ref{post:empirical-concentration}.
\end{remark}

\paragraph{Proof.}
Throughout, write $X_k := s_{j_k, i}(\Wh)$ for $k = 1, \ldots, m$. Set $\Wfresh := \WLOO[\{i, j_1, \ldots, j_m\}]$ and $V := \|\Wfresh v_i\|^2/d$.

\subparagraph*{Step 1: Iterated leave-one-out.}
Apply \eqref{eq:What_WLOO_expansion} successively after removing the samples $i, j_1, \ldots, j_m$. The finite-removal version of the LOO stability estimates, together with the spectral and self-averaging estimates \ref{post:spectral-op}--\ref{post:empirical-concentration}, gives the same rank-one expansion at each level. The resulting iterated identity expresses
\begin{equation}\label{eq:competitor_iterated_LOO}
    \Wh \;=\; \Wfresh \;+\; \sum_{\ell = 0}^{m}\Bigl[\,\hat a_{i_\ell}^{\setminus I_\ell}\,(H^{\setminus I_\ell})^{-1}\!\bigl[u_{i_\ell} v_{i_\ell}^\top\bigr] + R_\ell\,\Bigr],
\end{equation}
where $i_0 = i$, $i_\ell = j_\ell$ for $\ell \ge 1$, $I_\ell := \{i, j_1, \ldots, j_\ell\}$, and each $\|R_\ell\|_\mathsf{F} = \Op(d^{-1/2})$ by \sref{tam_loo} applied at the corresponding LOO level.

\subparagraph*{Step 2: Score decomposition.}
For each $k \in [m]$, sandwiching \eqref{eq:competitor_iterated_LOO} with $u_{j_k}^\top \cdot v_i$ yields
\begin{equation}\label{eq:competitor_score_decomp}
    X_k \;=\; u_{j_k}^\top \Wfresh v_i \;+\; \mathcal{E}_k,
\end{equation}
where $\mathcal{E}_k$ is the sum of $m + 1$ rank-one bilinear contributions plus residual terms.

\subparagraph*{Step 3: Bound the corrections.}
Each rank-one contribution at level $\ell$
\begin{equation}
    \hat a_{i_\ell}^{\setminus I_\ell}\,
    u_{j_k}^\top
    (H^{\setminus I_\ell})^{-1}
    [u_{i_\ell} v_{i_\ell}^\top]\,v_i
\end{equation}
admits a self-adjoint rewrite that places one of the held-out vectors --- $u_i$ for $\ell = 0$ and $v_i$ for $\ell \ge 1$ --- on the outer test side, independent of the LOO operator $(H^{\setminus I_\ell})^{-1}$ (which sees no sample in $I_\ell$). Combined with the operator-norm bound $\|(H^{\setminus I_\ell})^{-1}\|_\mathrm{op} \le d^2/(\lambda n) = O(1)$ from the ridge floor and $\|u_a v_b^\top\|_\mathsf{F} = O(1)$, the unmatched-Gaussian inner product gives an $\Op(d^{-1/2})$ bound on each rank-one contribution. The residuals $u_{j_k}^\top R_\ell v_i$ are bounded by $\|R_\ell\|_\mathsf{F}\,\|u_{j_k}\|\|v_i\| = \Op(d^{-1/2})$. Summing $m + 1$ such terms,
\begin{equation}\label{eq:competitor_correction_bound}
    |\mathcal{E}_k| \;=\; \Op(d^{-1/2}) \quad \text{uniformly in } k.
\end{equation}

\subparagraph*{Step 4: Apply the bilinear coupling.}
Since $\Wfresh$ is a function of $\{(u_l, v_l) : l \notin \{i, j_1, \ldots, j_m\}\}$, it is independent of $u_{j_1}, \ldots, u_{j_m}, v_i$. The bilinear coupling \corref{bilinear_gauss_coupling_rand} requires a Frobenius self-averaging estimate and an operator-norm input for $\Wfresh$; both are transferred from $\Wh$ by the same finite-removal stability used above. In particular, the stochastic-dominance deviation $\|\Wfresh\|_\mathsf{F}^2/d^2 - \nu_d^2 = \Op(d^{-1/2})$ follows from the Frobenius consequence of \ref{post:empirical-concentration}, while the operator-norm input $\|(\Wfresh)^\top \Wfresh\|_\mathsf{F}/d^2 = \Op(d^{-1/2})$ follows from \ref{post:spectral-op}. Apply \corref{bilinear_gauss_coupling_rand} with $W = \Wfresh$, $v_0 = v_i$, and $\rho_f=\rho_\mathrm{op}=O(d^{-1/2})$: there exist i.i.d. $Y_1, \ldots, Y_m \sim \mathcal{N}(0,\nu_d^2)$ with
\begin{equation}\label{eq:competitor_lead_coupling}
    \bigl\|(u_{j_1}^\top \Wfresh v_i, \ldots, u_{j_m}^\top \Wfresh v_i) - (Y_1, \ldots, Y_m)\bigr\|_2 \;=\; \Op(d^{-1/2}),
\end{equation}
with $Y_1,\ldots,Y_m$ i.i.d. $\mathcal{N}(0,\nu_d^2)$.

\subparagraph*{Step 5: Combine.}
Triangle inequality applied to \eqref{eq:competitor_score_decomp}, \eqref{eq:competitor_correction_bound}, and \eqref{eq:competitor_lead_coupling} yields \eqref{eq:competitor_coupling}. \qed

\subsubsection{Proximal law for diagonal scores}

We next record the finite-dimensional proximal law for the diagonal scores. Set
\begin{equation}
    \mu_d:=\mu_r(\nu_d),
    \qquad
    c_d:=c_r(\nu_d),
\end{equation}
where $\mu_r$ and $c_r$ are defined in \eqref{eq:c_r_nu_def}--\eqref{eq:mu_r_characterization}.

\begin{proposition}[Coupled proximal law for $s_{ii}$]\label{prop:diag_score_proximal}
Assume \ref{post:loo-mu}--\ref{post:trace-concentration}. Fix $i\in[n]$. On a common probability space, there is a standard Gaussian random variable $G_i\sim \mathcal{N}(0,1)$ such that
\begin{equation}\label{eq:diag_score_proximal_fd}
    s_{i,i}(\Wh)
    =
    c_d
    +
    \Prox_{\chi_d}\!\left(\nu_d G_i-c_d\right)
    +
    \Op(d^{-1/2}),
\end{equation}
where $\Prox_\chi$ is defined in \eqref{eq:prox_chi_def}.
\end{proposition}

\begin{proof}
Apply the LOO rank-one expansion \eqref{eq:What_WLOO_expansion} with $h=i$ and contract with $u_i^\top(\cdot)v_i$. This gives
\begin{equation}\label{eq:diag_score_from_loo}
    s_{i,i}(\Wh)
    =
    g_i
    +
    \hat a_i
    \bigl\langle u_i v_i^\top,\,(H^{\setminus i})^{-1}[u_i v_i^\top]\bigr\rangle
    +
    \Op(d^{-1/2}),
\end{equation}
where
\begin{equation}
    g_i:=u_i^\top\WLOO[i]v_i,
    \qquad
    \hat a_i=\sigma(-\hat m_i).
\end{equation}
The LOO Hessian $H^{\setminus i}$ is independent of $(u_i,v_i)$. Lemma~\ref{lem:fresh_rank_one_trace}, \ref{post:trace-concentration}, and the finite-removal trace stability therefore give
\begin{equation}
    \bigl\langle u_i v_i^\top,\,(H^{\setminus i})^{-1}[u_i v_i^\top]\bigr\rangle
    =
    \chi_d+\Op(d^{-1/2}).
\end{equation}
Since $\hat a_i$ is bounded, \eqref{eq:diag_score_from_loo} becomes
\begin{equation}\label{eq:diag_scalar_approx_ci}
    s_{i,i}(\Wh)
    =
    g_i
    +
    \chi_d\,\sigma\!\bigl(-s_{i,i}(\Wh)+\hat c_i\bigr)
    +
    \Op(d^{-1/2}),
\end{equation}
where
\begin{equation}
    \hat c_i
    :=
    \hat\mu_i
    +
    \frac{1}{k}\sum_{j\ne i}
    \phi_\beta\!\bigl(s_{j,i}(\Wh)-\hat\mu_i\bigr).
\end{equation}

We now replace $g_i$ and $\hat c_i$ by their finite-dimensional deterministic counterparts. The bilinear coupling estimate, applied to the leave-one-out matrix $\WLOO[i]$ and the fresh pair $(u_i,v_i)$, gives a coupling
\begin{equation}\label{eq:diag_gi_coupling}
    g_i
    =
    \nu_d G_i+\Op(d^{-1/2}),
    \qquad
    G_i\sim \mathcal{N}(0,1).
\end{equation}
Next, $\mu$-stationarity gives
\begin{equation}
    \frac{1}{n}\sum_{j\ne i}
    \sigma\!\bigl(\beta(s_{j,i}(\Wh)-\hat\mu_i)\bigr)
    =
    r+\Op(d^{-1/2}).
\end{equation}
The off-diagonal Gaussian law from Proposition~\ref{prop:competitor_gaussian}, together with the empirical-mean concentration estimate, implies that the left-hand side is
\begin{equation}
    \Eb{Z\sim \mathcal{N}(0,1)}
    {\sigma\!\bigl(\beta(\nu_d Z-\hat\mu_i)\bigr)}
    +
    \Op(d^{-1/2}).
\end{equation}
For fixed $r\in(0,1)$ and fixed $\beta$, the bounded-score event keeps $\hat\mu_i$ and $\mu_d$ in a compact set; on this set the derivative in $\mu$ of the preceding expectation is bounded away from zero. Hence \eqref{eq:mu_r_characterization} yields
\begin{equation}\label{eq:mu_i_to_mu_d}
    \hat\mu_i=\mu_d+\Op(d^{-1/2}).
\end{equation}
Applying the same Gaussian approximation to the pseudo-Lipschitz test function $\phi_\beta(\cdot-\mu)$, and using the Lipschitz dependence on $\mu$, gives
\begin{equation}\label{eq:c_i_to_c_d}
    \hat c_i=c_d+\Op(d^{-1/2}).
\end{equation}
Combining \eqref{eq:diag_scalar_approx_ci}, \eqref{eq:diag_gi_coupling}, and \eqref{eq:c_i_to_c_d}, we obtain
\begin{equation}
    s_{i,i}(\Wh)
    =
    \nu_d G_i
    +
    \chi_d\,\sigma\!\bigl(-s_{i,i}(\Wh)+c_d\bigr)
    +
    \Op(d^{-1/2}).
\end{equation}
Let $y_i:=s_{i,i}(\Wh)-c_d$. The last display is equivalent to
\begin{equation}
    y_i-\chi_d\sigma(-y_i)
    =
    \nu_d G_i-c_d+\Op(d^{-1/2}).
\end{equation}
The map $y\mapsto y-\chi_d\sigma(-y)$ is strictly increasing with derivative at least one, so its inverse $\Prox_{\chi_d}$ is $1$-Lipschitz. This proves \eqref{eq:diag_score_proximal_fd}.
\end{proof}

\subsubsection{Finite-dimensional scalar equations and variational consistency}
We now derive the two finite-dimensional scalar equations for $\nu_d$ and $\chi_d$, and then identify their zero-residual limit with the saddle-point equations of the scalar variational problem.
The variance equation is obtained directly from $W$-stationarity. Let
$G$ and $Z$ be independent standard Gaussian variables, and define
\begin{equation}\label{eq:finite_d_scalar_vars}
\begin{aligned}
    S_d
    :=
    c_d+\Prox_{\chi_d}\!\left(\nu_d G-c_d\right),
    &\qquad
    A_d
    :=
    \sigma(-S_d+c_d).
\end{aligned}
\end{equation}

\begin{proposition}[Finite-dimensional variance equation]\label{prop:nu_finite_d}
Assume \ref{post:loo-mu}--\ref{post:trace-concentration}, and use the coupled diagonal score law of
Proposition~\ref{prop:diag_score_proximal}. Then
\begin{equation}\label{eq:nu_finite_d_fp}
    \nu_d^2
    \left[
    \lambda
    +
    \frac{1}{r}\,\E[A_d]\,
    \E\!\left[\phi_\beta''\!\bigl(\nu_d Z-\mu_d\bigr)\right]
    \right]
    =
    \E[A_dS_d]
    +
    O(d^{-1/2}).
\end{equation}
\end{proposition}

\begin{proof}
At the full minimizer, the $W$-stationarity condition \eqref{eq:grad_W} gives
\begin{equation}
    UQ\Lambda_aV^\top+\frac{\lambda n}{d^2}\Wh=0.
\end{equation}
Taking the Frobenius inner product with $\Wh$ and using
$s_{j,i}(\Wh)=u_j^\top \Wh v_i$ gives the exact identity
\begin{equation}\label{eq:nu_exact_stationarity}
    \lambda\,\frac{\|\Wh\|_\mathsf{F}^2}{d^2}
    =
    \frac{1}{n}\sum_{i=1}^n
    a_i
    \left(
    s_{i,i}(\Wh)
    -
    \frac{1}{k}\sum_{j\ne i}q_{j,i}s_{j,i}(\Wh)
    \right).
\end{equation}
Taking expectations turns the left-hand side into $\lambda\nu_d^2$.

The first empirical average in \eqref{eq:nu_exact_stationarity} is a smooth
test function of the diagonal score, after replacing $c_i$ by $c_d$ at the
$d^{-1/2}$ scale. Proposition~\ref{prop:diag_score_proximal} and \ref{post:empirical-concentration} therefore
yield
\begin{equation}\label{eq:diag_contrib_nu}
    \E\left[
    \frac{1}{n}\sum_{i=1}^n a_i s_{i,i}(\Wh)
    \right]
    =
    \E[A_dS_d]+O(d^{-1/2}).
\end{equation}
For the off-diagonal contribution, the competitor Gaussian law
Proposition~\ref{prop:competitor_gaussian}, the same empirical-mean concentration,
and the leave-one-out independence between the diagonal factor and the competitor
score give
\begin{equation}
    Q_d
    :=
    \sigma\!\bigl(\beta(\nu_d Z-\mu_d)\bigr).
\end{equation}
\begin{equation}\label{eq:offdiag_contrib_nu}
    \E\left[
    \frac{1}{n}\sum_{i=1}^n
    a_i
    \frac{1}{k}\sum_{j\ne i}q_{j,i}s_{j,i}(\Wh)
    \right]
    =
    \frac{1}{r}\,\E[A_d]\E[\nu_d ZQ_d]+O(d^{-1/2}).
\end{equation}
Here the distinction between $k/n$ and $k/(n-1)$ contributes only $O(n^{-1})$,
which is absorbed into $O(d^{-1/2})$. Combining
\eqref{eq:nu_exact_stationarity}--\eqref{eq:offdiag_contrib_nu} gives
\begin{equation}\label{eq:nu_finite_d_pre_stein}
    \lambda\nu_d^2
    =
    \E[A_dS_d]
    -
    \frac{1}{r}\,\E[A_d]\E[\nu_d ZQ_d]
    +
    O(d^{-1/2}).
\end{equation}
Finally, Stein's identity for the standard Gaussian $Z$ gives
\begin{equation}
    \E[\nu_d ZQ_d]
    =
    \nu_d^2
    \E\!\left[
    \phi_\beta''\!\bigl(\nu_d Z-\mu_d\bigr)
    \right],
\end{equation}
which gives \eqref{eq:nu_finite_d_fp}.
\end{proof}

We next derive the companion equation for the scalar $\chi_d$
defined in \eqref{eq:chi_d_def}. It is
the expectation of the normalized inverse trace of the full Hessian, and is often referred to as the susceptibility in the statistical physics literature. Set
\begin{equation}\label{eq:B_d_def}
    B_d:=A_d(1-A_d).
\end{equation}

\begin{proposition}[Finite-dimensional susceptibility equation]\label{prop:chi_finite_d}
Assume \ref{post:loo-mu}--\ref{post:trace-concentration}. Then the finite-dimensional susceptibility $\chi_d$ satisfies
\begin{equation}\label{eq:chi_finite_d_fp}
    1
    =
    \alpha\lambda\chi_d
    +
    \alpha\,
    \E\!\left[
    \frac{\chi_d B_d}{1+\chi_d B_d}
    \right]
    +
    \frac{\alpha\chi_d}{r}\,
    \E[A_d]\,
    \E\!\left[\phi_\beta''\!\bigl(\nu_d Z-\mu_d\bigr)\right]
    +
    O(d^{-1/2}).
\end{equation}
In the above display, $A_d$ is the random variable defined in \eqref{eq:finite_d_scalar_vars}.
\end{proposition}

\begin{proof}
For $j,i\in[n]$, write
\begin{equation}
    f_{j i}:=\vecop(u_j v_i^\top)\in \R^{d^2}.
\end{equation}
Also let
\begin{equation}
    r_{j,i}
    :=
    \phi_\beta''\!\bigl(s_{j,i}-\mu_i\bigr)
    =
    \beta\,q_{j,i}(1-q_{j,i}),
    \qquad
    q_{j,i}
    =
    \sigma\!\bigl(\beta(s_{j,i}-\mu_i)\bigr),
\end{equation}
so that $\widetilde R_{j,i}=r_{j,i}/k$. Define the simplified Hessian
\begin{equation}\label{eq:simplified_hessian_main}
    \widetilde H
    =
    \frac{\lambda n}{d^2}I
    +
    \sum_{i=1}^n b_i f_{i i}f_{i i}^\top
    +
    \sum_{i=1}^n\sum_{j\ne i}
    \frac{a_i r_{j,i}}{k} f_{j i}f_{j i}^\top.
\end{equation}
By Lemmas~\ref{lem:chi_rank_one_block_reduction} and~\ref{lem:chi_coefficient_stability} in Appendix~\ref{app:chi_trace}, the normalized inverse trace of the original Hessian may be replaced, up to $\Op(d^{-1/2})$, by the normalized inverse trace of $\widetilde H$. Thus, by the trace concentration estimate \ref{post:trace-concentration},
\begin{equation}\label{eq:simplified_trace_to_chid}
    \frac{1}{d^2}\tr(\widetilde H^{-1})
    =
    \chi_d+\Op(d^{-1/2}).
\end{equation}
We now use the identity
\begin{equation}
    1
    =
    \frac{1}{d^2}\tr\!\left(\widetilde H\widetilde H^{-1}\right).
\end{equation}
Expanding the three terms in \eqref{eq:simplified_hessian_main} gives
\begin{equation}
\begin{aligned}
    1
    &=
    \frac{\lambda n}{d^2}\,
    \frac{1}{d^2}\tr(\widetilde H^{-1})
    +
    \frac{1}{d^2}\sum_{i=1}^n
    b_i f_{ii}^\top \widetilde H^{-1}f_{ii}
    \\
    &\quad+
    \frac{1}{d^2}\sum_{i=1}^n\sum_{j\ne i}
    \frac{a_i r_{j,i}}{k}
    f_{j i}^\top \widetilde H^{-1}f_{j i}.
\end{aligned}
\end{equation}
The first term is evaluated by \eqref{eq:simplified_trace_to_chid}. Lemmas~\ref{lem:diag_quadratic_chi} and~\ref{lem:offdiag_quadratic_chi} in Appendix~\ref{app:chi_trace} evaluate the two families of quadratic forms:
\begin{equation}
    f_{ii}^\top \widetilde H^{-1}f_{ii}
    =
    \frac{\chi_d}{1+b_i\chi_d}
    +
    \Op(d^{-1/2}),
    \qquad
    f_{j i}^\top \widetilde H^{-1}f_{j i}
    =
    \chi_d+\Op(d^{-1/2}).
\end{equation}
Substitution gives the finite-dimensional trace identity
\begin{equation}\label{eq:chi_trace_pre_expectation}
    1
    =
    \frac{\lambda n}{d^2}\,\chi_d
    +
    \frac{1}{d^2}
    \sum_{i=1}^n
    b_i\,\frac{\chi_d}{1+b_i\chi_d}
    +
    \frac{\chi_d}{d^2}
    \sum_{i=1}^n\sum_{j\ne i}
    \frac{a_i r_{j,i}}{k}
    +
    \Op(d^{-1/2}).
\end{equation}
Taking expectations under the moment bounds established in the appendices, and then applying the diagonal proximal law, the off-diagonal Gaussian law, and \ref{post:empirical-concentration}, gives
\begin{equation}
    \E\!\left[
    b_i\,\frac{\chi_d}{1+b_i\chi_d}
    \right]
    =
    \E\!\left[
    \frac{\chi_d B_d}{1+\chi_d B_d}
    \right]
    +
    O(d^{-1/2}),
\end{equation}
and
\begin{equation}
    \E\!\left[
    a_i\,\frac{1}{k}\sum_{j\ne i}r_{j,i}
    \right]
    =
    \frac{1}{r}\E[A_d]\,
    \E\!\left[\phi_\beta''\!\bigl(\nu_d Z-\mu_d\bigr)\right]
    +O(d^{-1/2}).
\end{equation}
Substituting these two replacements into \eqref{eq:chi_trace_pre_expectation}, and replacing $n/d^2$ by $\alpha$ at the same error scale, yields \eqref{eq:chi_finite_d_fp}.
\end{proof}

Combining Propositions~\ref{prop:nu_finite_d} and~\ref{prop:chi_finite_d} gives the finite-dimensional residual system
\begin{equation}\label{eq:finite_d_map_residual}
\begin{aligned}
    \nu_d^2
    \left[
    \lambda+
    \frac{1}{r}\E[A_d]\,
    \E\!\left[\phi_\beta''\!\bigl(\nu_d Z-\mu_d\bigr)\right]
    \right]
    -
    \E[A_dS_d]
    &=
    O(d^{-1/2}),
    \\
    1-\alpha\lambda\chi_d
    -
    \alpha\,\E\!\left[
    \frac{\chi_d B_d}{1+\chi_d B_d}
    \right]
    -
    \frac{\alpha}{r}\E[A_d]\,
    \E\!\left[\phi_\beta''\!\bigl(\nu_d Z-\mu_d\bigr)\right]\chi_d
    &=
    O(d^{-1/2}).
\end{aligned}
\end{equation}
Here $\mu_d=\mu_r(\nu_d)$ is the deterministic threshold defined in \eqref{eq:mu_r_characterization}, evaluated at the finite-dimensional variance parameter $\nu_d$.
Thus $(\nu_d,\chi_d)$ is an \emph{approximate} solution of the nonlinear scalar equations whose zero-residual limit turns out to be exactly equal to the saddle-point equations for the potential $\mathcal P_{\alpha,r}$ in Definition~\ref{def:scalar_potentials}. This variational connection is established in the next proposition.

\begin{proposition}[Variational form of the scalar equations]\label{prop:scalar_variational}
For every $\nu>0$, $\chi\mapsto \mathcal P_{\alpha,r}(\nu,\chi)$ has a unique maximizer. Moreover, the reduced potential $\mathcal V_{\alpha,r}$ is $\lambda$-strongly convex in $\nu$. If $\nu_\ast$ is the unique minimizer of $\mathcal V_{\alpha,r}$ and $\chi_\ast$ is the unique maximizer of $\mathcal P_{\alpha,r}(\nu_\ast,\cdot)$, then the saddle-point equations of
\begin{equation}\label{eq:minimax_scalar_potential}
    \inf_{\nu>0}\sup_{\chi>0}\mathcal P_{\alpha,r}(\nu,\chi)
\end{equation}
are exactly the zero-residual limit of \eqref{eq:finite_d_map_residual}, with the finite-dimensional proximal channel replaced by \eqref{eq:scalar_channel_nu_chi} and with $c_r$ given by \eqref{eq:c_r_nu_def}.
\end{proposition}

\begin{proof}
See Appendix~\ref{app:scalar_variational}.
\end{proof}

\section{Ridgeless limit and phase transition}
\label{sec:tam_phase_transitions}

The scalar theory in \sref{tam_selfconsistent} gives a limiting
variational problem for each fixed ridge parameter $\lambda>0$ and
fixed smoothing parameter $\beta<\infty$. In this section we take the
sequential ridgeless limit: first $n,d\to\infty$ with
$n/d^2\to\alpha$ and $k/n\to r$, and then
$\lambda\downarrow0$ while keeping $\beta\in(0,\infty)$ fixed. The
limit separates two regimes. Below the critical load
$\alpha_c(r)$ defined in \eqref{eq:ridgeless_alpha_c} below, the scalar
norm $\nu$ escapes to infinity as $\lambda \downarrow 0$ and the limiting fitting loss vanishes; we
call this the satisfiable, or SAT, phase. Above this critical load the
scalar norm remains finite and the limiting fitting loss is positive;
we call this the unsatisfiable, or UNSAT, phase. After deriving the
transition from the variational principle, we describe the score,
margin, and percentile profiles on both sides.

All results in this section concern the scalar saddle problem of
Definition~\ref{def:scalar_potentials}. They translate into
statements about the original TAM minimizer $(\Wh,\muh)$ through the
scalar identification in Theorem~\ref{thm:tam_scalar}.

\subsection{The critical load}

\begin{definition}[Ridgeless threshold]\label{def:ridgeless_threshold}
Let $G\sim N(0,1)$, and define
\begin{equation}\label{eq:kappa_r_def}
    \kappa_r:=\frac{\varphi(\Phi^{-1}(1-r))}{r},
\end{equation}
where $\Phi$ and $\varphi$ are the CDF and density of a standard
normal random variable. The critical load is
\begin{equation}\label{eq:ridgeless_alpha_c}
    \alpha_c(r)
    =
    \frac{1}{\E[(\kappa_r-G)_+^2]}
    =
    \frac{1}{
        (1+\kappa_r^2)\Phi(\kappa_r)+\kappa_r\varphi(\kappa_r)
    }.
\end{equation}
\end{definition}

In what follows, we first explain the
origin and underlying ideas behind the transition at $\alpha_c(r)$. The formal proof is then given in
Proposition~\ref{prop:ridgeless_phase_transition}.

We start from the dual form of the reduced potential in
\eqref{eq:reduced_potential_dual}, writing the ridge parameter
explicitly:
\begin{equation}\label{eq:section5_dual_potential}
\begin{aligned}
    \mathcal V_{\alpha,r,\lambda}(\nu)
    &=
    \frac{\lambda}{2}\nu^2
    +
    \sup_{0\le a(\cdot)\le1}
    \Bigg\{
    \E\!\left[
    a(G)\bigl(c_r(\nu)-\nu G\bigr)-I(a(G))
    \right]
    \\
    &\hspace{3.5cm}
    -
    \frac{\nu}{\sqrt{\alpha}}
    \left(\E[a(G)^2]\right)^{1/2}
    \Bigg\},
\end{aligned}
\end{equation}
where $G\sim N(0,1)$ and
$I(a)=a\log a+(1-a)\log(1-a)$. The supremum is over measurable
functions $a:\R\to[0,1]$. We write
$\mathcal V_{\alpha,r,0}$ for the ridgeless potential obtained by
setting $\lambda=0$; explicitly,
\begin{equation}\label{eq:ridgeless_potential_explicit}
\begin{aligned}
    \mathcal V_{\alpha,r,0}(\nu)
    =
    \sup_{0\le a(\cdot)\le1}
    \Bigg\{
    \E\!\left[
    a(G)\bigl(c_r(\nu)-\nu G\bigr)-I(a(G))
    \right]
    -
    \frac{\nu}{\sqrt{\alpha}}
    \left(\E[a(G)^2]\right)^{1/2}
    \Bigg\}.
\end{aligned}
\end{equation}

The phase transition is determined by the large-$\nu$ behavior of the
ridgeless potential $\mathcal V_{\alpha,r,0}$. The only input from the
TAM threshold $c_r(\nu)$ needed for the leading-order calculation is
the asymptotic relation
\begin{equation}\label{eq:cr_over_nu_limit}
    \frac{c_r(\nu)}{\nu}\longrightarrow \kappa_r .
\end{equation}
We defer the verification of \eqref{eq:cr_over_nu_limit} to the
proof of Proposition~\ref{prop:ridgeless_phase_transition}.
Now set $m_r(\nu):=c_r(\nu)/\nu$. For fixed $a(\cdot)$, the
linear-in-$\nu$ part of \eqref{eq:section5_dual_potential} is
\begin{equation}
    \nu
    \left[
    \E\!\left[a(G)(m_r(\nu)-G)\right]
    -
    \frac{1}{\sqrt{\alpha}}\|a\|_2
    \right],
    \qquad
    \|a\|_2:=\left(\E[a(G)^2]\right)^{1/2}.
\end{equation}
The entropy term $I(a)$ is not multiplied by $\nu$.
Thus the large-norm behavior is governed by a simple question: can we
choose a nonzero dual profile $a$ so that the coefficient of the
leading linear term in $\nu$ is positive?

Letting $\nu\to\infty$ in this coefficient and using
\eqref{eq:cr_over_nu_limit}, the relevant expression becomes
\begin{equation}
    \E[a(G)(\kappa_r-G)]
    -
    \frac{1}{\sqrt{\alpha}}\|a\|_2 .
\end{equation}
Since $a\ge0$, the negative part of $\kappa_r-G$ can only hurt, so
the best profile should live on the event $\{G<\kappa_r\}$. More
precisely, by Cauchy--Schwarz, for any $0\le a\le1$,
\begin{equation}
    \E[a(G)(\kappa_r-G)]
    \le
    \E[a(G)(\kappa_r-G)_+]
    \le
    \|a\|_2\,
    \left(\E[(\kappa_r-G)_+^2]\right)^{1/2}.
\end{equation}
The upper bound is sharp despite the pointwise constraint $a\le1$.
Indeed, equality in Cauchy--Schwarz would ask for
$a$ proportional to $(\kappa_r-G)_+$, which is unbounded. We can
approximate this direction by the admissible truncations
\begin{equation}
    a_M(G):=\min\left\{\frac{(\kappa_r-G)_+}{M},1\right\},
    \qquad M\to\infty .
\end{equation}
The ratio
$\E[a_M(G)(\kappa_r-G)]/\|a_M\|_2$ converges to
$\left(\E[(\kappa_r-G)_+^2]\right)^{1/2}$. Hence a positive
large-$\nu$ direction exists exactly when
$\alpha>\alpha_c(r)$, with $\alpha_c(r)$ as in
\eqref{eq:ridgeless_alpha_c}. 

\begin{proposition}[Ridgeless scalar phase transition]\label{prop:ridgeless_phase_transition}
Fix $r\in(0,1)$ and $\beta\in(0,\infty)$. For each
$\lambda>0$, let $\nu_\lambda$ be the unique minimizer of
$\mathcal V_{\alpha,r,\lambda}$.
\begin{enumerate}[label=\textup{(\roman*)}]
    \item If $\alpha<\alpha_c(r)$, then
    \begin{equation}
        \nu_{\lambda}\to\infty
        \qquad\text{and}\qquad
        \mathcal V_{\alpha,r,0}(\nu_\lambda)
        \to0
    \end{equation}
    as $\lambda\downarrow0$.

    \item If $\alpha>\alpha_c(r)$, then $\nu_{\lambda}$ remains
    bounded as $\lambda\downarrow0$ and converges to the unique
    minimizer $\nu_0<\infty$ of $\mathcal V_{\alpha,r,0}$. The
    ridgeless fitting loss $\mathcal V_{\alpha,r,0}(\nu_0)$ is
    strictly positive.
\end{enumerate}
\end{proposition}

\begin{proof}
We first justify \eqref{eq:cr_over_nu_limit}. Rescaling the
definition of $c_r$ by writing $\mu=\nu t$ gives
\begin{equation}\label{eq:cr_over_nu_rescaled}
    \frac{c_r(\nu)}{\nu}
    =
    \min_{t\in\R}
    \left\{
    t+\frac{1}{r}\E\!\left[
    \frac{\phi_\beta(\nu(G-t))}{\nu}
    \right]
    \right\}.
\end{equation}
For each fixed $y$, the function
$\nu\mapsto \phi_\beta(\nu y)/\nu$ is nonincreasing and converges to
$y_+$. Thus the limiting problem is
\begin{equation}
    \min_{t\in\R}
    \left\{
    t+\frac{1}{r}\E[(G-t)_+]
    \right\}.
\end{equation}
Its minimizer is $t=\Phi^{-1}(1-r)$ and its value is $\kappa_r$, which proves
\eqref{eq:cr_over_nu_limit}. In fact the convergence is monotone
decreasing.

\emph{SAT case: $\alpha<\alpha_c(r)$.} By
\eqref{eq:cr_over_nu_limit} and the definition
\eqref{eq:ridgeless_alpha_c}, there are $\delta>0$ and $\nu_0<\infty$
such that, for all $\nu\ge\nu_0$,
\begin{equation}
    \left(\E[(m_r(\nu)-G)_+^2]\right)^{1/2}
    \le
    \frac{1}{\sqrt{\alpha}}-\delta .
\end{equation}
Therefore, uniformly over $0\le a\le1$,
\begin{equation}
    \E[a(G)(m_r(\nu)-G)]
    -
    \frac{1}{\sqrt{\alpha}}\|a\|_2
    \le
    -\delta\|a\|_2 .
\end{equation}
The remaining terms in the dual representation are lower order. More
precisely, for any $p\in(0,1)$ there is $C_p<\infty$ such that
$-I(x)\le C_p x^p$ for $0\le x\le1$, while
$\E[a^p]\le\|a\|_2^p$. Hence
\begin{equation}
\begin{aligned}
    \mathcal V_{\alpha,r,0}(\nu)
    &\le
    \sup_{0\le x\le1}
    \left\{
    -\delta\nu x+C_p x^p
    \right\}
    \longrightarrow0 \quad \text{as } \nu \to \infty.
\end{aligned}
\end{equation}
Since $a=0$ is admissible, $\mathcal V_{\alpha,r,0}(\nu)\ge0$, and
therefore $\mathcal V_{\alpha,r,0}(\nu)\to0$ as $\nu\to\infty$.

To identify the variational value with the scalar fitting loss, fix
$\nu$ and let $\chi(\nu)$ be the unique maximizer of
$\mathcal P_{\alpha,r,\lambda}(\nu,\cdot)$, whose existence and
uniqueness are given by Proposition~\ref{prop:scalar_variational}.
With the scalar proximal channel $S_{\nu,\chi}$ defined in
\eqref{eq:scalar_channel_nu_chi}, the identity
\begin{equation}
    \mathcal V_{\alpha,r,\lambda}(\nu)
    =
    \frac{\lambda}{2}\nu^2
    +
    \E\!\left[
    \ell(S_{\nu,\chi}-c_r(\nu))
    \right]
\end{equation}
follows from a short cancellation. Writing
$A:=\sigma(-S_{\nu,\chi}+c_r(\nu))$, the proximal equation gives
$S_{\nu,\chi}-\nu G=\chi A$, and hence
\begin{equation}
    M_r(\nu,\chi)
    =
    \frac{\chi}{2}\E[A^2]
    +
    \E\!\left[
    \ell(S_{\nu,\chi}-c_r(\nu))
    \right].
\end{equation}
At the maximizing $\chi$, the saddle equation
$\nu^2=\alpha\chi^2\E[A^2]$ gives
$\nu^2/(2\alpha\chi)=(\chi/2)\E[A^2]$, so the two quadratic terms in
$\mathcal P_{\alpha,r,\lambda}$ cancel. Thus the ridgeless potential
is the scalar fitting loss and is strictly positive at every finite
$\nu$.
For any $\varepsilon>0$, choose $N$ so large that
$\mathcal V_{\alpha,r,0}(N)\le\varepsilon$, and then choose
$\lambda$ so small that $\lambda N^2/2\le\varepsilon$. This gives
\begin{equation}
    \mathcal V_{\alpha,r,\lambda}(\nu_\lambda)
    \le
    2\varepsilon .
\end{equation}
Hence the optimal value tends to zero. If $\nu_\lambda$ remained
bounded along a subsequence, lower semicontinuity would force a
finite point with zero ridgeless fitting loss, contradicting strict
positivity. Therefore $\nu_\lambda\to\infty$, and the displayed
identity implies that the per-sample fitting loss tends to zero.

\emph{UNSAT case: $\alpha>\alpha_c(r)$.} By the sharpness of the
Cauchy--Schwarz bound, there is a bounded measurable
$a:\R\to[0,1]$ and a number $\delta>0$ such that
\begin{equation}
    \E[a(G)(\kappa_r-G)]
    -
    \frac{1}{\sqrt{\alpha}}\|a\|_2
    \ge
    2\delta .
\end{equation}
Using $m_r(\nu)\to\kappa_r$ and bounded convergence, the same $a$
satisfies, for all sufficiently large $\nu$,
\begin{equation}
    \E[a(G)(m_r(\nu)-G)]
    -
    \frac{1}{\sqrt{\alpha}}\|a\|_2
    \ge
    \delta .
\end{equation}
Substituting this fixed test function into
\eqref{eq:section5_dual_potential} gives
\begin{equation}
    \mathcal V_{\alpha,r,0}(\nu)
    \ge
    \delta\nu-O(1),
\end{equation}
so the ridgeless potential is coercive. The minimizers of
$\mathcal V_{\alpha,r,\lambda}
=\mathcal V_{\alpha,r,0}+\lambda\nu^2/2$ therefore remain in a
compact set as $\lambda\downarrow0$. Any limit point minimizes
$\mathcal V_{\alpha,r,0}$ by the standard variational argument:
compare $\nu_\lambda$ against an arbitrary fixed $\nu$ and let
$\lambda\downarrow0$.

We now show that the ridgeless minimizer is unique. For finite
$\beta$, the TAM threshold $c_r$ is strictly convex. Indeed, with
\begin{equation}
    w_\nu(Z):=\phi_\beta''\!\bigl(\nu Z-\mu_r(\nu)\bigr),
\end{equation}
the envelope theorem and the defining equation for $\mu_r(\nu)$ give
\begin{equation}
    c_r''(\nu)
    =
    \frac{1}{r}
    \left\{
    \E[w_\nu(Z)Z^2]
    -
    \frac{\E[w_\nu(Z)Z]^2}{\E[w_\nu(Z)]}
    \right\}
    >0 .
\end{equation}
For each finite $\nu$, the dual maximizer in
\eqref{eq:ridgeless_potential_explicit} has positive expectation. In
the primal variables this maximizer is
$a(G)=\sigma(-S_{\nu,\chi}+c_r(\nu))$, and hence
$0<a(G)<1$ almost surely. Therefore
$\mathcal V_{\alpha,r,0}$ is strictly convex: if
$\nu_\theta=\theta\nu_1+(1-\theta)\nu_2$ with
$0<\theta<1$ and $\nu_1\ne\nu_2$, evaluating the dual formula at the
maximizer for $\nu_\theta$ and using the strict convexity of $c_r$
gives
\begin{equation}
    \mathcal V_{\alpha,r,0}(\nu_\theta)
    <
    \theta\mathcal V_{\alpha,r,0}(\nu_1)
    +
    (1-\theta)\mathcal V_{\alpha,r,0}(\nu_2).
\end{equation}
Thus the ridgeless minimizer is unique; call it $\nu_0$. Hence
$\nu_\lambda\to\nu_0$. Since $\nu_0$ is finite, the corresponding
scalar logistic loss is strictly positive.
\end{proof}

\begin{remark}[No dependence on the smoothing parameter]
The critical load $\alpha_c(r)$ does not depend on the smoothing
parameter $\beta\in(0,\infty)$. The parameter $\beta$ affects the
finite-$\lambda$ scalar channel and the finite ridgeless minimizer in
the UNSAT phase, but it disappears from the large-$\nu$ limit
\eqref{eq:cr_over_nu_limit}, where the smoothed logistic tail
$\phi_\beta(\nu y)/\nu$ converges to $y_+$.
\end{remark}

\subsection{Score, margin, and percentile profiles}

We next translate the ridgeless phase transition into predictions for
the true scores, the competing scores, the margins, and the percentile
rank of the true score among its competitors. The two phases have
different scaling behavior. In the SAT phase,
Proposition~\ref{prop:ridgeless_phase_transition} shows that
$\nu_\lambda\to\infty$. Thus the true scores and the competitor
scores both diverge on the original scale. The meaningful objects are
the normalized scores, obtained by dividing by $\nu_\lambda$, and
these have a non-degenerate limit. In the UNSAT phase, by contrast,
$\nu_\lambda$ remains finite, so the ridgeless profiles live on the
original scale.

\subsubsection{The SAT phase}

We begin with the SAT phase, where the limiting normalized profiles
are explicit. The single parameter controlling them is the number
$\rho_\alpha>\kappa_r$, defined as the unique solution of
\begin{equation}\label{eq:rho_alpha_equation}
    \E[(\rho_\alpha-G)_+^2]
    =
    \frac{1}{\alpha},
    \qquad G\sim N(0,1).
\end{equation}
The left-hand side is strictly increasing in $\rho_\alpha$, and the
boundary value $\rho_\alpha=\kappa_r$ is exactly
\eqref{eq:ridgeless_alpha_c}. Hence $\rho_\alpha$ is monotone
decreasing in $\alpha$ and satisfies
$\rho_\alpha\downarrow\kappa_r$ as
$\alpha\uparrow\alpha_c(r)$.

\begin{proposition}[SAT score, margin, and percentile profiles]\label{prop:sat_profiles}
Assume $\alpha<\alpha_c(r)$, and let $\rho_\alpha>\kappa_r$ be the
unique solution of \eqref{eq:rho_alpha_equation}. In the sequential
ridgeless limit $\lambda\downarrow0$, the normalized true score
satisfies
\begin{equation}\label{eq:sat_score_profile_statement}
    \frac{
    S_{\nu_{\lambda},\chi_{\lambda}}
    }{\nu_{\lambda}}
    \Rightarrow
    \max\{G,\rho_\alpha\},
    \qquad G\sim N(0,1).
\end{equation}
Consequently, the normalized TAM margin satisfies
\begin{equation}\label{eq:sat_margin_profile}
    \frac{
    S_{\nu_{\lambda},\chi_{\lambda}}
    -
    c_r(\nu_{\lambda})
    }{\nu_{\lambda}}
    \Rightarrow
    \max\{G,\rho_\alpha\}-\kappa_r,
    \qquad G\sim N(0,1).
\end{equation}
Here and below, $\Rightarrow$ denotes weak convergence of laws as
$\lambda\downarrow0$. The associated percentile variable satisfies
\begin{equation}\label{eq:sat_percentile_profile}
    \Omega
    \Rightarrow
    \Phi\!\left(\max\{G,\rho_\alpha\}\right).
\end{equation}
Equivalently, if $p_\alpha:=\Phi(\rho_\alpha)$, then
\begin{equation}\label{eq:sat_percentile_cdf}
    \P(\Omega\le \omega)
    =
    \begin{cases}
        0, & \omega<p_\alpha,\\
        \omega, & p_\alpha\le \omega\le 1.
    \end{cases}
\end{equation}
\end{proposition}

\begin{proof}
By Proposition~\ref{prop:ridgeless_phase_transition}, in the SAT
phase $\nu_\lambda\to\infty$ and the scalar fitting loss tends to
zero. The proof has three steps. First, we rewrite the proximal
channel in terms of the normalized score and the correction term
created by the loss. Second, we introduce the scalar quantity
$\rho_\lambda$ and use the saddle-point equation to show that it is
bounded. Third, we identify every subsequential limit of
$\rho_\lambda$; uniqueness of the equation defining $\rho_\alpha$
then gives the stated weak convergence.

Set $c_\lambda:=c_r(\nu_\lambda)$. For each deterministic scalar
input $g\in\R$, define
\begin{equation}
    S_\lambda(g)
    :=
    c_\lambda+\Prox_{\chi_\lambda}(\nu_\lambda g-c_\lambda),
    \qquad
    X_\lambda(g):=\frac{S_\lambda(g)}{\nu_\lambda}.
\end{equation}
The proximal first-order condition is the pointwise identity
\begin{equation}\label{eq:sat_channel_scaled}
    S_\lambda(g)
    =
    \nu_\lambda g
    +
    \chi_\lambda
    \sigma(-S_\lambda(g)+c_\lambda).
\end{equation}
We also define the pointwise correction
\begin{equation}\label{eq:sat_B_lambda_def}
    B_\lambda(g)
    :=
    X_\lambda(g)-g
    =
    \frac{\chi_\lambda}{\nu_\lambda}
    \sigma(-S_\lambda(g)+c_\lambda)
    \ge0 .
\end{equation}
Only now do we insert the Gaussian variable: when $G\sim N(0,1)$,
the random variables in the scalar channel are
$S_\lambda(G)$, $X_\lambda(G)$, and $B_\lambda(G)$. The $\chi$
saddle-point equation for the scalar
potential in Definition~\ref{def:scalar_potentials}, recorded in
\eqref{eq:app_chi_saddle_first} of
Appendix~\ref{app:scalar_variational}, gives
\begin{equation}\label{eq:sat_B_norm}
    \E[B_\lambda(G)^2]
    =
    \frac{1}{\alpha}.
\end{equation}
Moreover, the zero-loss statement implies that the normalized score
does not fall below the TAM threshold. Indeed, with
$M_\lambda(g):=S_\lambda(g)-c_\lambda$, we have pointwise
\begin{equation}
    \left(
    \frac{c_\lambda}{\nu_\lambda}
    -
    X_\lambda(g)
    \right)_+
    =
    \frac{(-M_\lambda(g))_+}{\nu_\lambda}.
\end{equation}
For every $m\in\R$,
\begin{equation}
    (-m)_+
    \le
    \ell(m).
\end{equation}
Therefore, since the fitting loss
$\E[\ell(M_\lambda)]$ tends to zero and
$\nu_\lambda\to\infty$,
\begin{equation}
    \E\left[
    \left(
    \frac{c_\lambda}{\nu_\lambda}
    -
    X_\lambda(G)
    \right)_+
    \right]
    \longrightarrow0 .
\end{equation}
Since $c_\lambda/\nu_\lambda\to\kappa_r$ by
\eqref{eq:cr_over_nu_limit}, the $1$-Lipschitz property of
$x\mapsto x_+$ gives
\begin{equation}\label{eq:sat_loss_implies_above_threshold}
    \E\left[
    \left(
    \kappa_r
    -
    X_\lambda(G)
    \right)_+
    \right]
    \longrightarrow0 .
\end{equation}

We now identify the limit of the channel. The useful normalization is
the level at which the sigmoid correction in
\eqref{eq:sat_B_lambda_def} changes from negligible to order one.
Define
\begin{equation}\label{eq:rho_lambda_def}
    \rho_\lambda
    :=
    \frac{c_\lambda+\log(\chi_\lambda/\nu_\lambda)}
    {\nu_\lambda}.
\end{equation}
Combining this definition with
\eqref{eq:sat_B_lambda_def} and the identity
$\sigma(x)=(1+\exp\{-x\})^{-1}$ gives the exact formula
\begin{equation}\label{eq:B_lambda_exact_identity}
    B_\lambda(g)
    =
    \frac{\exp\{\nu_\lambda(\rho_\lambda-X_\lambda(g))\}}
    {1+\exp\{c_\lambda-\nu_\lambda X_\lambda(g)\}} .
\end{equation}
The sequence $\rho_\lambda$ is bounded. Indeed, if
$\rho_\lambda\to-\infty$ along a subsequence, then
$\chi_\lambda/\nu_\lambda
=\exp\{\nu_\lambda\rho_\lambda-c_\lambda\}\to0$, and hence
$B_\lambda(g)\le\chi_\lambda/\nu_\lambda\to0$ uniformly in $g$,
contradicting
\eqref{eq:sat_B_norm}. Next suppose that
$\rho_\lambda\to+\infty$ along a subsequence. We claim that, for all
sufficiently small $\lambda$,
\begin{equation}
    B_\lambda(g)\ge \frac{\rho_\lambda}{2}
    \qquad\text{for every }g\in[-1,0].
\end{equation}
Indeed, if the reverse inequality held at some $g\in[-1,0]$, then
$X_\lambda(g)=g+B_\lambda(g)\le\rho_\lambda/2$. Using
\eqref{eq:B_lambda_exact_identity} and the boundedness of
$c_\lambda/\nu_\lambda$, we would get
\begin{equation}
    B_\lambda(g)
    \ge
    \exp\{\nu_\lambda\rho_\lambda/3\}
\end{equation}
for all sufficiently small $\lambda$, contradicting
$B_\lambda(g)<\rho_\lambda/2$. Hence the claim holds. Since
$\P(-1\le G\le0)=\Phi(0)-\Phi(-1)>0$, the claim implies
\begin{equation}
    \E[B_\lambda(G)^2]
    \ge
    \bigl(\Phi(0)-\Phi(-1)\bigr)\frac{\rho_\lambda^2}{4}
    \longrightarrow\infty,
\end{equation}
contradicting \eqref{eq:sat_B_norm}.

It remains to identify the possible limits of $\rho_\lambda$. Take a
subsequence along which $\rho_\lambda\to\rho$. Passing to a further
subsequence, we may assume from
\eqref{eq:sat_loss_implies_above_threshold} that
\begin{equation}\label{eq:sat_pointwise_above_kappa}
    \liminf_{\lambda\downarrow0} X_\lambda(G)
    \ge
    \kappa_r
    \qquad\text{almost surely}.
\end{equation}
The rest of the argument is pointwise in $g$. First, if $g>\rho$,
then $B_\lambda(g)\to0$ and $X_\lambda(g)\to g$. Indeed, any positive
subsequential limit of $B_\lambda(g)$ would force
$X_\lambda(g)=g+B_\lambda(g)>\rho_\lambda$ eventually, and the
right-hand side of \eqref{eq:B_lambda_exact_identity} would then
converge to zero. This excludes $\rho<\kappa_r$, because on the
positive-probability event $\{\rho<G<\kappa_r\}$ it would imply
$X_\lambda(G)\to G<\kappa_r$, contradicting
\eqref{eq:sat_pointwise_above_kappa}. Hence $\rho\ge\kappa_r$.

We now identify the channel limit for any subsequential limit
$\rho\ge\kappa_r$. Along any bounded subsequence of $\rho_\lambda$,
\eqref{eq:B_lambda_exact_identity} gives the domination
\begin{equation}\label{eq:sat_B_domination}
    B_\lambda(g)^2
    \le
    1+
    \left(\sup_\lambda \rho_\lambda-g\right)_+^2 .
\end{equation}
Indeed, if $B_\lambda(g)\le1$ there is nothing to prove; if
$B_\lambda(g)>1$, then \eqref{eq:B_lambda_exact_identity} implies
$X_\lambda(g)\le\rho_\lambda$, and so
$B_\lambda(g)=X_\lambda(g)-g\le(\rho_\lambda-g)_+$.

For $g>\rho$, we already proved $X_\lambda(g)\to g$. Now fix
$g<\rho$. For almost every such $g$, every subsequential limit $x$ of
$X_\lambda(g)$ satisfies $x\ge\kappa_r$ by
\eqref{eq:sat_pointwise_above_kappa}. If $x<\rho$, then the
exponential numerator in \eqref{eq:B_lambda_exact_identity} grows like
$\exp\{\nu_\lambda(\rho-x+o(1))\}$, while the denominator grows at
most like $\exp\{\nu_\lambda(\kappa_r-x)_++o(\nu_\lambda)\}$. Since
$\rho\ge\kappa_r$ and $x\ge\kappa_r$, this is impossible. If
$x>\rho$, the correction term vanishes by \eqref{eq:B_lambda_exact_identity},
forcing $x=g$, also impossible because $g<\rho$. Hence $x=\rho$.
Thus, with $G\sim N(0,1)$,
\begin{equation}\label{eq:sat_score_profile}
    X_\lambda(G)
    \longrightarrow
    \max\{G,\rho\},
    \qquad
    B_\lambda(G)
    \longrightarrow
    (\rho-G)_+ .
\end{equation}
The domination bound \eqref{eq:sat_B_domination} allows us to pass to
the second moment in \eqref{eq:sat_B_norm}, so
\begin{equation}
    \E[(\rho-G)_+^2]
    =
    \frac{1}{\alpha}.
\end{equation}
By uniqueness of the solution to \eqref{eq:rho_alpha_equation}, every
subsequential limit is $\rho_\alpha$. Hence the whole sequence
satisfies
$X_\lambda(G)\Rightarrow\max\{G,\rho_\alpha\}$. Subtracting
$c_\lambda/\nu_\lambda\to\kappa_r$ proves
\eqref{eq:sat_score_profile_statement} and
\eqref{eq:sat_margin_profile}.

The percentile distribution is obtained by comparing the diagonal
score with the limiting competitor field $N(0,\nu_\lambda^2)$. Hence
\eqref{eq:sat_percentile_profile}. The CDF formula
\eqref{eq:sat_percentile_cdf} follows by noting that
$\Phi(\max\{G,\rho_\alpha\})$ has an atom of mass
$p_\alpha=\Phi(\rho_\alpha)$ at $p_\alpha$, followed by a uniform tail
on $(p_\alpha,1)$.
\end{proof}

\subsubsection{The UNSAT phase}

Assume $\alpha>\alpha_c(r)$. In this regime the ridgeless limit does
not escape to infinite norm. By
Proposition~\ref{prop:ridgeless_phase_transition}, the minimizers
$\nu_\lambda$ converge to the unique finite minimizer $\nu_0$ of the
ridgeless reduced potential. The next proposition records the
corresponding finite-scale profiles. We omit the proof, since the
statement is obtained by specializing the scalar channel from
\sref{tam_selfconsistent} at the ridgeless minimizer.

\begin{proposition}[UNSAT score, margin, and percentile profiles]\label{prop:unsat_profiles}
Let $\chi_0$ be the maximizer in the scalar saddle problem
corresponding to $\nu_0$, and define
\begin{equation}
    c_0:=c_r(\nu_0).
\end{equation}
Then the limiting true score is the finite-scale scalar random
variable
\begin{equation}\label{eq:unsat_score_channel}
    S_0
    :=
    S_{\nu_0,\chi_0}
    =
    c_0+\Prox_{\chi_0}\!\bigl(\nu_0G-c_0\bigr),
    \qquad G\sim N(0,1).
\end{equation}
The competing scores remain on the same finite scale, with limiting
Gaussian scale $\nu_0$. The limiting TAM margin and fitting loss are
\begin{equation}\label{eq:unsat_loss_margin}
    M_0:=S_0-c_0,
    \qquad
    L_0
    :=
    \E\!\left[\ell(M_0)\right]
    >
    0.
\end{equation}
Finally, the limiting percentile rank of the true score is
\begin{equation}\label{eq:unsat_percentile}
    \Omega_0
    :=
    \Phi\!\left(\frac{S_0}{\nu_0}\right).
\end{equation}
\end{proposition}

\begin{remark}
Unlike the phase boundary and the SAT normalized profiles, these
UNSAT quantities generally retain a dependence on the fixed value of
$\beta$ through the scalar functions $c_r(\nu)$ and
$S_{\nu,\chi}$.
\end{remark}

\subsection{Numerical validation}

We now give finite-dimensional numerical checks of the scalar
predictions. The most direct prediction is the phase boundary of
Proposition~\ref{prop:ridgeless_phase_transition}: as $(\alpha,r)$
crosses the curve $\alpha=\alpha_c(r)$, the optimized fitting loss
should separate from zero and the norm scale
$\nu=\|\Wh\|_{\mathsf F}/d$ should change sharply. Figure
\ref{fig:phase_diagram_validation} tests this prediction in a
two-dimensional sweep over the load $\alpha=n/d^2$ and the tail scale
$r=k/(n-1)$. The empirical transition extracted from the simulations is
closely aligned with the theoretical curve \eqref{eq:ridgeless_alpha_c},
both in the loss landscape and in the growth of the norm scale.

\begin{figure}[t]
    \centering
    \includegraphics[width=\textwidth]{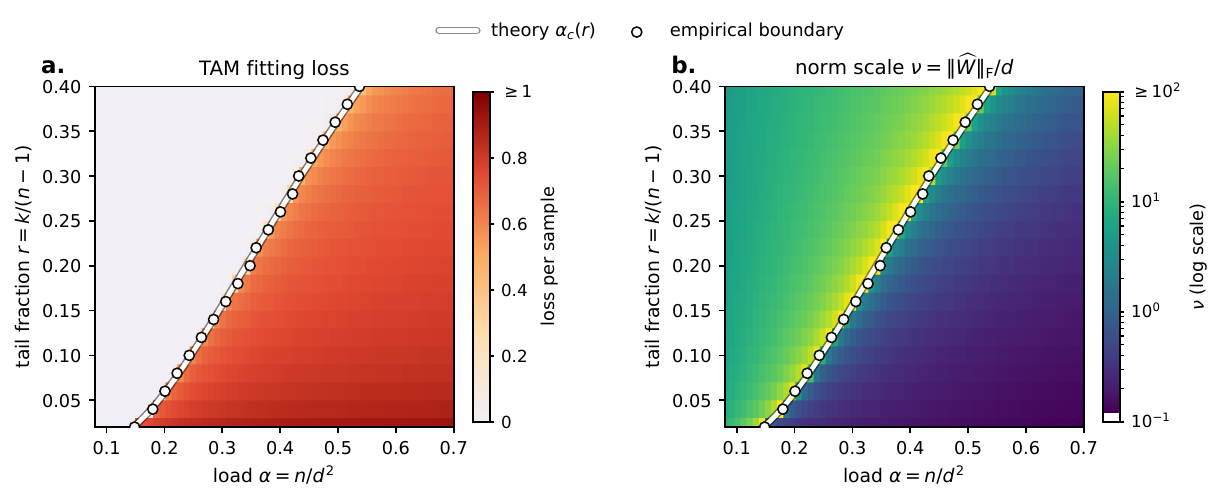}
    \caption{Finite-dimensional validation of the predicted phase
    boundary. The solid curve is the theoretical threshold
    $\alpha_c(r)$ from \eqref{eq:ridgeless_alpha_c}, and open circles
    show empirical boundary estimates from the numerical sweep.
    \textbf{a.} Optimized TAM fitting loss per sample, after
    subtracting the ridge contribution; colors are clipped at one.
    \textbf{b.} Norm scale
    $\nu=\|\Wh\|_{\mathsf F}/d$ on a logarithmic color scale; colors
    are clipped at $100$. Both panels use $d=400$ and small ridge
    regularization $\lambda = 1e-6$. The observed transition is sharply aligned with
    the scalar prediction.}
    \label{fig:phase_diagram_validation}
\end{figure}

The phase diagram gives a global view of the transition. We next look
more locally, along a one-dimensional slice in $\alpha$ (with $r = 0.15$), and compare
the finite-dimensional optimizer with the scalar predictions on both
sides of the transition. This second comparison provides more information: it
tests not only the location of the phase boundary, but also the
predicted loss, norm scale, margin profile, and percentile distribution.

\begin{figure}[tbp]
    \centering
    \includegraphics[width=\textwidth]{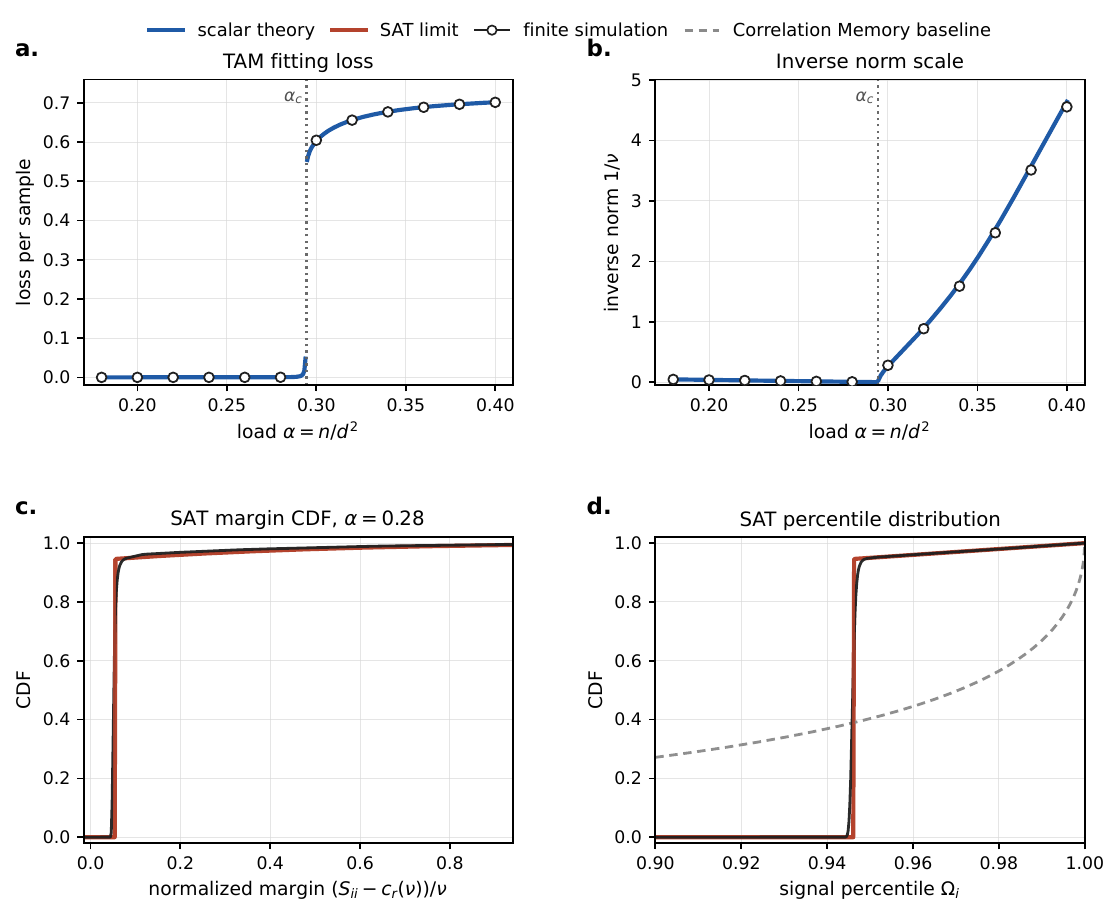}
    \caption{Scalar theory compared with finite-dimensional
    simulations along a one-dimensional slice at $r=0.15$.
    \textbf{a.} Optimized TAM fitting loss per sample as a function of
    the load $\alpha$, after subtracting the ridge contribution.
    \textbf{b.} Inverse norm scale $1/\nu$, showing the collapse of
    $1/\nu$ in the SAT phase and its positive value in the UNSAT
    phase. In both upper panels, the solid curve is obtained from the
    scalar prediction \eqref{eq:reduced_potential_def}, with the norm
    and loss read off through
    \eqref{eq:tam_order_parameter_prediction} and
    \eqref{eq:tam_loss_prediction}, at $\beta=30$ and ridge
    $\lambda=10^{-7}$. Open circles are finite-dimensional
    simulations at $d=600$, and the dashed vertical line marks the
    ridgeless threshold $\alpha_c(r)$. \textbf{c.} Empirical CDF of
    the normalized TAM margin $(S_{ii}-c_r(\nu))/\nu$ in the SAT
    phase, compared with \eqref{eq:sat_margin_profile}. \textbf{d.}
    Empirical CDF of the signal percentile $\Omega_i$, compared with
    \eqref{eq:sat_percentile_cdf} and the correlation-memory heuristic
    \eqref{eq:percentile-copula}.
    The lower panels use the representative SAT load $\alpha=0.28$.}
    \label{fig:scalar_theory_validation}
\end{figure}

Figure~\ref{fig:scalar_theory_validation} shows good agreement between
the scalar theory and the finite-dimensional simulations. The upper
panels compare two scalar observables: the fitting loss and the inverse
norm scale $1/\nu$. The loss remains close to zero throughout the SAT
phase and becomes positive beyond the predicted critical load. At the
same transition, $1/\nu$ collapses toward zero on the SAT side, while
it remains positive in the UNSAT phase, as predicted by
Proposition~\ref{prop:ridgeless_phase_transition}.

The lower panels compare distributional predictions in the SAT phase.
The empirical normalized margin distribution follows the limiting law
\eqref{eq:sat_margin_profile}, with the visible rounding caused by
finite dimension. The percentile panel
is particularly informative. The TAM prediction
\eqref{eq:sat_percentile_cdf} has a hard lower edge at
$p_\alpha=\Phi(\rho_\alpha)$, meaning that in the ridgeless SAT limit
the true item is never predicted to fall below this percentile among
its competitors. The finite-dimensional curve tracks this edge closely.
The same panel also shows the correlation-memory heuristic benchmark
\eqref{eq:percentile-copula}; compared with this non-optimized
baseline, the TAM optimizer shifts the signal percentiles sharply
toward the upper tail, reflecting the improvement obtained by learning
the memory matrix for the listwise objective.

\FloatBarrier

\section{Conclusion}
\label{sec:conclusion}

The capacity of a linear associative memory depends as much on how it is read
as on how many parameters it has. Under isotropic Gaussian embeddings, top-1
retrieval has a sharp threshold: the critical value of $d^2/(n\log n)$ is
$2$. If this ratio stays above $2$, our adaptive sparse-repair construction
retrieves all $n$ associations with high probability; if it stays below $2$,
no data-dependent linear memory can do so. The logarithmic factor is therefore
not a weakness of a particular storage rule. It is the price of requiring
every query to beat its strongest competitor. The critical value was
conjectured in an earlier version of this paper and independently in the
concurrent work of \citet{giorlandino2026factual}.

The picture changes when the correct target need only remain within a fixed
fraction of the candidates. TAM captures this regime by comparing the signal
with the average of its $k=\lceil r(n-1)\rceil$ strongest competitors. For
fixed $r\in(0,1)$, the list size is $k=\Theta(n)$, so this is a percentile-rank
guarantee rather than a sublinear shortlist. At any fixed finite smoothing
parameter and positive ridge, empirical risk minimization with the smoothed
TAM objective has a high-dimensional limit governed by two scalar parameters.
These parameters determine the limiting score, margin, and percentile laws.
Taking the high-dimensional limit first and then sending the ridge parameter
to zero gives the critical load $\alpha_c(r)$ for $\alpha=n/d^2$. Below this
load, the average fitting loss vanishes and the exact TAM criterion holds for
all but a vanishing fraction of keys; above it, the limiting average fitting
loss remains positive.

The TAM proof required a leave-one-out argument adapted to the strong
dependence among the losses. Removing one association does two things at once:
it changes that association's signal term, and it changes its role as a
competitor for every other query. Our coupled leave-one-out expansion tracks
both effects and reduces the leading response to an inverse Hessian applied to
a single rank-one feature. The accompanying stability, spectral, and
self-averaging estimates then lead to the scalar theory. This mechanism is not
specific to TAM and may also be useful for other matrix-valued empirical-risk
problems in which every sample appears throughout the loss.

Between the two retrieval regimes lies a natural open question: what happens
when the list fraction $r$ tends to zero with $n$? This limit should clarify how
the quadratic TAM scale crosses over to the extra logarithmic cost of top-1
retrieval. It also remains to prove the second-order correction proposed in
Conjecture~\ref{conj:top1-finite-size}, which appears to capture the slow
finite-size drift, and to understand how much of the theory persists beyond
isotropic Gaussian embeddings.

\subsection*{Disclosure of AI Assistance}
\addcontentsline{toc}{section}{Disclosure of AI Assistance}

We used AI tools, specifically ChatGPT 5.5 and 5.6, in this work.
The problem formulations and central ideas---including the
overall approach to top-1 retrieval, the Tail-Average
Margin criterion, and the leave-one-out framework for its analysis---originated with the authors. For top-1
retrieval, the authors first established nonmatching achievability and
impossibility bounds, based respectively on CMM and a trace comparison using
uniform Gaussian anticoncentration, and reported them in an earlier version
of this paper. Starting from these arguments, iterative collaboration with AI
tools helped sharpen both sides to the exact constant $2$. Building on our
earlier CMM analysis, the sharp achievability argument augments CMM with an
adaptive sparse correction that repairs the exceptional low-margin queries.
This construction, formalized in Proposition~\ref*{prop:optimal}, drew inspiration from
multiscale majority and deficit-repair constructions for random perceptrons
\citep{kim1998covering,abbe2022binary,ding2025capacity}; AI tools assisted in
refining the argument. For the converse, the
dual-certificate approach was motivated by classical separation and perceptron
methods \citep{wendel1962problem,cover1965geometrical}, and AI tools assisted
substantially in completing the technical arguments. We view this as an encouraging example of how human--AI collaboration can
contribute to mathematical research.

The TAM formulation and leave-one-out analysis, including the rank-one
reduction in Proposition~\ref*{prop:loo_rank_one_reduction} and the resulting
scalar characterization, were developed by the authors and reported in an
earlier version of this paper. AI tools assisted in completing and checking
several supporting technical arguments, particularly
the leave-one-out stability estimates in Proposition~\ref*{prop:loo_stability}
and the spectral and self-averaging concentration estimates in
Proposition~\ref*{prop:spectral_self_averaging}. AI tools also assisted with
numerical experiments and with editing and polishing the manuscript. The
authors carefully verified all results and take full responsibility for the
content of this paper.

\FloatBarrier
\bibliographystyle{plainnat}
\bibliography{refs}

\clearpage
\appendix

\section{Correlation matrix memory: score law and retrieval}
\label{app:correlation_memory}

This appendix gives the CMM score-law and retrieval calculations used
in Section~\ref{sec:superposition}.  Section~A.1 derives the exact
one-query law underlying~\eqref{eq:cmm-score-heuristic}, and
Section~A.2 proves Proposition~\ref{prop:sup-achievability} using a
conditional moment-generating function.

Under the Gaussian model~\eqref{eq:gaussian_model}, let
$U=[u_1\;\cdots\;u_n]$ and $V=[v_1\;\cdots\;v_n]$ in
$\R^{d\times n}$; their columns are mutually independent
$\mathcal N(0,I_d/d)$ vectors.  For the correlation matrix memory
$W_{\mathsf{CMM}}=UV^\top$ from~\eqref{eq:cmm}, throughout this appendix
we write
\[
    S
    \bydef
    U^\top W_{\mathsf{CMM}}V
    =U^\top U V^\top V
    \in\R^{n\times n},
    \qquad
    S_{j,i}=s_{j,i}=u_j^\top W_{\mathsf{CMM}}v_i.
\]
Thus column~$i$ contains all decoder scores for query~$v_i$:
$s_{i,i}$ is its signal score, while $(s_{j,i})_{j\ne i}$ are its
competitor scores.

Moreover, $\chi(k)$ denotes the chi distribution with
$k$ degrees of freedom.

\subsection{Law of one CMM score column}\label{app:lemma_1_proof}

\begin{lemma}[Exact law of one CMM score column]\label{lem:single-coord-law}
Assume $n\ge3$ and $d\ge2$.  Let $e_1^{(n)}$ be the first standard
basis vector in $\R^n$ and write
\[
    \begin{bmatrix}\theta\\ \xi\end{bmatrix}
    \bydef Se_1^{(n)},
    \qquad
    \theta=s_{1,1},
    \qquad
    \xi=(s_{2,1},\ldots,s_{n,1})^\top\in\R^{n-1}.
\]
Let $a \sim \mathcal{N}(0, 1)$,
$g \sim \mathcal{N}(0, I_{n-1})$, and
\begin{equation}\label{eq:rho-vars}
    \rho_1, \rho_2 \sim \frac{\chi(d)}{\sqrt{d}}, \quad
    \rho_3 \sim \frac{\chi(n-1)}{\sqrt{n}}, \quad
    \rho_4 \sim \frac{\chi(d-1)}{\sqrt{d}}, \quad
    \rho_5 \sim \frac{\chi(n-2)}{\sqrt{n}}.
\end{equation}
Assume that $a,g,\rho_1,\ldots,\rho_5$ are mutually independent. Then
\begin{equation}
    \begin{bmatrix} \theta \\ \xi \end{bmatrix}
    \;\eqd\;
    \rho_1
    \begin{bmatrix}
        \rho_2 \tau \\[2pt]
        \tfrac{\sigma}{\norm{g}/\sqrt{n}}\, g
    \end{bmatrix},
\end{equation}
where
\begin{equation}\label{eq:single_coor_law_parameters}
\tau = \rho_1 \rho_2 + \sqrt{\frac{n}{d^2}}\, \rho_3\, a, \qquad
    \sigma = \left[
        \frac{(a\tau/\sqrt{n} + \rho_3 \rho_4^2)^2}{d}
        + \frac{\tau^2 \rho_5^2}{d}
        + \frac{n}{d^2}\, \rho_3^2 \rho_4^2 \rho_5^2
    \right]^{1/2}.
\end{equation}
\end{lemma}

\begin{proof}

We derive the distributional representation~\eqref{eq:single_coor_law_parameters} for the first CMM score column. It is easy to see that, for any deterministic orthogonal matrix $O \in O(n-1)$,
\begin{equation} \label{eq:orthong_equiv_s_xi}
    \begin{bmatrix}
        \theta\\
        \xi
    \end{bmatrix} \overset{(d)}{=} \begin{bmatrix}
        \theta \\
        O \xi
    \end{bmatrix}.
\end{equation}
To see this, writing 
\begin{equation}
    Q := \begin{bmatrix}
        1 & 0 \\
        0 & O
    \end{bmatrix} \in O(n),
\end{equation}
we have $(UQ^\top, VQ^\top) \overset{(d)}{=} (U, V)$. Defining 
\[
w := U^\top U V^\top V e_1^{(n)}, \qquad w' := (UQ^\top)^\top (UQ^\top) (VQ^\top)^\top (VQ^\top) e_1^{(n)},
\]
it holds that $w' = Q w$ since $Q^\top e_1^{(n)} = e_1^{(n)}$. Since $w' \overset{(d)}{=} w$, we have $w \overset{(d)}{=} Qw$ and so \eqref{eq:orthong_equiv_s_xi} holds. From rotational invariance, it follows that
\begin{equation} \label{eq:rot_inv_conseq_xi}
    \begin{bmatrix}
        \theta\\
        \xi
    \end{bmatrix} \overset{(d)}{=} \begin{bmatrix}
        \theta\\
        \norm{\xi} \frac{g}{\norm{g}}
    \end{bmatrix},
\end{equation}
where $g \sim \mathcal{N}(0, I_{n-1})$ is a Gaussian vector independent of~$\theta$ and $\norm{\xi}$.

Let $v_1 \in \R^d$ denote the first column of $V$. It is straightforward to check that
\begin{equation}
    V^\top V e_1^{(n)} \overset{(d)}{=} \norm{v_1}\begin{bmatrix}
        \norm{v_1}\\
        z
    \end{bmatrix},
\end{equation}
where $z \sim \mathcal{N}(0, I_{n-1}/d)$ is independent of $v_1$. We may also write
\begin{equation}
    U = \begin{bmatrix}
        u_1 & H_{u_1} M H_z^\top
    \end{bmatrix} = H_{u_1} \begin{bmatrix}
        \norm{u_1} e_1^{(d)} & M H_z^\top
    \end{bmatrix},
\end{equation}
where $H_{u_1} \in \R^{d \times d}$ and $H_z \in \R^{(n-1) \times (n-1)}$ denote orthogonal matrices whose first columns are aligned with $u_1$ and $z$ respectively, and
\begin{equation}
    M = \begin{bmatrix}
        \tilde a & q^\top\\
        p & G
    \end{bmatrix} \in \R^{d \times (n-1)}, \qquad \tilde a \in \R,\; q \in \R^{n-2},\; p \in \R^{d-1},\; G \in \R^{(d-1) \times (n-2)}
\end{equation}
is a Gaussian matrix with i.i.d. $\mathcal{N}(0, 1/d)$ entries. Substituting the above representations in
\begin{equation}
    U^\top U V^\top V e_1^{(n)} = \begin{bmatrix}
        \theta\\
        \xi
    \end{bmatrix}
\end{equation}
yields
\begin{equation}
    \theta = \norm{v_1}\norm{u_1} \tau, \qquad\text{where}\qquad \tau \bydef \norm{u_1}\norm{v_1} + \tilde a \norm{z}.
\end{equation}
Moreover,
\begin{equation} \label{eq:xi_representation}
    \xi = \|v_1 \| H_z \begin{bmatrix}
        \tilde a \tau + \norm{z}\norm{p}^2\\
        \tau q + \|z\| G^\top p
    \end{bmatrix} \overset{(d)}{=} \|v_1 \| H_z \begin{bmatrix}
        \tilde a \tau + \norm{z}\norm{p}^2\\
        \sqrt{\tau^2 + \norm{z}^2\norm{p}^2} \tilde z
    \end{bmatrix},
\end{equation}
where $\tilde z \sim \mathcal{N}(0, I_{n-2}/d)$ is independent of $(\tau, z, p)$.
To reach the claim of the lemma, set
\begin{equation}
    \rho_1 = \norm{v_1}, \quad \rho_2 = \norm{u_1}, \quad \rho_3 = \sqrt{\frac d n} \norm{z}, \quad \rho_4 = \norm{p}, \quad \rho_5 = \sqrt{\frac d n} \norm{\tilde z}. 
\end{equation}
Moreover, the standard Gaussian variable $a$ in the statement of the lemma is $\tilde a \sqrt d$. The claim of the lemma then follows readily by applying \eqref{eq:rot_inv_conseq_xi} with \eqref{eq:xi_representation}.
\end{proof}

\begin{remark}[Gaussian consequence of the one-column law]
\label{rem:cmm-one-column-gaussian-limit}
Suppose that $n,d\to\infty$ with $n/d^2\to\alpha\in(0,\infty)$.
Then, for every fixed $m$,
\[
    (\theta,\xi_1,\ldots,\xi_m)
    \Longrightarrow
    \bigl(
        1+\sqrt\alpha\,G_0,
        \sqrt\alpha\,G_1,
        \ldots,
        \sqrt\alpha\,G_m
    \bigr),
\]
where $G_0,\ldots,G_m$ are independent standard Gaussians. To see
this, apply Lemma~\ref{lem:single-coord-law}: each $\rho_\ell$,
$\ell\in[5]$, and $\|g\|/\sqrt n$ converge in probability to $1$,
whereas $\sigma\to\sqrt\alpha$.

The same representation yields a more general standardized
formulation.  Set $\alpha_n\bydef n/d^2$.  If $d/n\to0$, then, for
every fixed $m$,
\[
    \left(
        \frac{\theta-1}{\sqrt{\alpha_n}},
        \frac{\xi_1}{\sqrt{\alpha_n}},
        \ldots,
        \frac{\xi_m}{\sqrt{\alpha_n}}
    \right)
    \Longrightarrow
    (G_0,G_1,\ldots,G_m),
\]
where $G_0,G_1,\ldots,G_m$ are independent standard Gaussians, and
the convergence is joint in $\R^{m+1}$.

Thus, for a fixed query, the signal score and any fixed number of
competitor scores have precisely the leading-order Gaussian form
suggested by \eqref{eq:cmm-score-heuristic}, including their
asymptotic independence. This finite-dimensional consequence does not
by itself control the maximum over all $n-1$ competitors; the uniform
tail estimate needed for perfect retrieval is proved in
Appendix~\ref{app:perfect_retrieval}.
\end{remark}

\subsection{CMM top-1 retrieval above constant 8}
\label{app:perfect_retrieval}

\begin{proof}[Proof of Proposition~\ref{prop:sup-achievability}]
For an ordered pair $i\ne j$, write
\[
    D_{j,i}
    \bydef
    s_{i,i}-s_{j,i},
    \qquad
    x_{j,i}\bydef u_i-u_j.
\]
The CMM succeeds precisely when $D_{j,i}>0$ for every ordered pair.
We bound the lower tail of one such margin directly.

\paragraph{Exact conditional moment-generating function.}
Fix $i\ne j$ and condition on $(u_i,u_j,v_i)$.  With
$x=x_{j,i}$ and $y=v_i$, expanding
$W_{\mathsf{CMM}}=\sum_{k=1}^n u_kv_k^\top$ gives
\begin{equation}\label{eq:cmm-pairwise-margin-decomposition}
    D_{j,i}
    =
    a_{j,i}
    +b_{j,i}\langle v_j,y\rangle
    +\sum_{k\notin\{i,j\}}
      \langle x,u_k\rangle\langle v_k,y\rangle,
\end{equation}
where
\[
    a_{j,i}\bydef\langle x,u_i\rangle\|y\|^2,
    \qquad
    b_{j,i}\bydef\langle x,u_j\rangle.
\]
The terms after $a_{j,i}$ in
\eqref{eq:cmm-pairwise-margin-decomposition} are conditionally
independent and centered.  Moreover,
$\langle v_j,y\rangle\sim\mathcal N(0,\|y\|^2/d)$, while for
$k\notin\{i,j\}$ the two factors
$\langle x,u_k\rangle$ and $\langle v_k,y\rangle$ are independent
centered Gaussians with variances $\|x\|^2/d$ and $\|y\|^2/d$.
Thus, if $R_{j,i}\bydef D_{j,i}-a_{j,i}$, then
\begin{equation}\label{eq:cmm-pairwise-margin-mgf}
\E\left[e^{\lambda R_{j,i}}\mid u_i,u_j,v_i\right]
=
\exp\left(
    \frac{\lambda^2b_{j,i}^2\|y\|^2}{2d}
\right)
\left(
    1-\frac{\lambda^2\|x\|^2\|y\|^2}{d^2}
\right)^{-(n-2)/2}
\end{equation}
whenever $\lambda^2\|x\|^2\|y\|^2<d^2$.  This identity uses
$\E e^{tGH}=(1-t^2\sigma_G^2\sigma_H^2)^{-1/2}$ for independent
centered Gaussians $G,H$.

\paragraph{Uniform concentration.}
The condition $\liminf d^2/(n\log n)>8$ implies
$d/\log n\to\infty$.
Standard chi-square and Gaussian tail bounds, followed by a union
bound, therefore give a deterministic sequence $\varepsilon_n\to0$
such that, with probability $1-o(1)$,
\begin{equation}\label{eq:cmm-pairwise-good-event}
\begin{split}
    &\max_{i\le n}
      \left(
        \bigl|\|u_i\|^2-1\bigr|
        \vee
        \bigl|\|v_i\|^2-1\bigr|
      \right)
      \le\varepsilon_n,\\
    &\hspace{35mm}
      \max_{i\ne j}|\langle u_i,u_j\rangle|
      \le\varepsilon_n.
\end{split}
\end{equation}
Call this event $\mathcal H_n$, and let $\mathcal H_{j,i}$ denote
the corresponding restrictions involving only $(u_i,u_j,v_i)$.
On $\mathcal H_{j,i}$, uniformly over $i\ne j$,
\begin{equation}\label{eq:cmm-pairwise-coefficients}
    a_{j,i}=1+o(1),
    \qquad
    b_{j,i}=-1+o(1),
    \qquad
    \|x_{j,i}\|^2\|v_i\|^2=2+o(1).
\end{equation}

It remains to apply Chernoff's bound to
\eqref{eq:cmm-pairwise-margin-mgf}.  We do so along an arbitrary
subsequence.  Passing to a further subsequence, either
$d/n\to0$, or $d/n$ is bounded below by a positive constant.

Suppose first that $d/n\to0$, and set
$\alpha_n=n/d^2$.  Conditional on
$(u_i,u_j,v_i)\in\mathcal H_{j,i}$, choose
$\lambda=a_{j,i}/(2\alpha_n)$.  Then
$\lambda/d=O(d/n)=o(1)$, and
$-\log(1-t)=t+O(t^2)$ in
\eqref{eq:cmm-pairwise-margin-mgf} yields, uniformly over $i\ne j$,
\[
    \log\E\left[e^{-\lambda R_{j,i}}
        \mid u_i,u_j,v_i\right]
    =
    (1+o(1))\lambda^2\alpha_n.
\]
Here the Gaussian term in
\eqref{eq:cmm-pairwise-margin-mgf} is negligible because
$1/d=o(\alpha_n)$.  Since $R_{j,i}$ is symmetric, Chernoff's bound
and \eqref{eq:cmm-pairwise-coefficients} give
\begin{equation}\label{eq:cmm-pairwise-sharp-tail}
    \sup_{(u_i,u_j,v_i)\in\mathcal H_{j,i}}
    \P\left(D_{j,i}\le0\mid u_i,u_j,v_i\right)
    \le
    \exp\left(-\frac{1-o(1)}{4\alpha_n}\right).
\end{equation}
Choose $\eta>0$ so that
$d^2/(n\log n)\ge8+\eta$ eventually.  The exponent on the
right-hand side of \eqref{eq:cmm-pairwise-sharp-tail} is then at
least $(2+c_\eta)\log n$ for some $c_\eta>0$ and all sufficiently
large $n$.

Suppose instead that $d/n\ge c_0>0$.  Taking $\lambda=\kappa d$ in
\eqref{eq:cmm-pairwise-margin-mgf}, with $\kappa>0$ sufficiently
small depending only on $c_0$, gives
\begin{equation}\label{eq:cmm-pairwise-high-dimensional-tail}
    \sup_{(u_i,u_j,v_i)\in\mathcal H_{j,i}}
    \P\left(D_{j,i}\le0\mid u_i,u_j,v_i\right)
    \le e^{-c_1d}
\end{equation}
for some $c_1>0$.  Indeed, the Chernoff exponent contributes
$-\kappa d(1+o(1))$, whereas the logarithm of the moment-generating
function is $O(\kappa^2(d+n))=O(\kappa^2d)$.  Since
$d/\log n\to\infty$, the right-hand side of
\eqref{eq:cmm-pairwise-high-dimensional-tail} is $o(n^{-2})$.

Finally, a union bound over the $n(n-1)$ ordered pairs gives
\[
\begin{split}
    \P\bigl(\mathsf{PR}(W_{\mathsf{CMM}};U,V)^c\bigr)
    &\le
    \P(\mathcal H_n^c)
    +\sum_{i\ne j}
      \P\bigl(D_{j,i}\le0,\mathcal H_{j,i}\bigr)\\
    &=o(1)
\end{split}
\]
along either type of further subsequence.  Every subsequence admits
such a further subsequence, so the same convergence holds along the
original sequence.
\end{proof}

\section{Taylor expansion of the interior remainder}
\label{app:tam_taylor}

This appendix carries out the first-order Taylor expansion of the interior remainder $\mathcal{R}^{\setminus h}$ defined in \eqref{eq:R_def} of \sref{tam_loo}, derives the explicit form of the residual $\hat E$ appearing in the main identity \eqref{eq:H_LOO_identity}, and establishes the entry-size bounds \eqref{eq:Ehat_entry_sizes} used there. All notation is as in \sref{tam_loo}: $(\Wh, \muh)$ and $(\WLOO, \muLOO)$ are the full and leave-one-out minimizers, $\Delta S := U^\top(\Wh - \WLOO)V$, $(\Delta \mu)_i := \hat \mu_i - \mu_i^{\setminus h}$; a hat denotes a quantity evaluated at $(\Wh, \muh)$ and a superscript $\setminus h$ one evaluated at $(\WLOO, \muLOO)$. The analysis is conditional on the LOO postulates (A1)--(A3).

\subsection{Taylor identities}
\label{app:tam_taylor_identities}

We record three Taylor identities around the LOO point. Throughout, for a scalar quantity $X(W, \mu)$ we write $\Delta X := X(\Wh, \muh) - X(\WLOO, \muLOO)$.

\paragraph{Distractor sigmoid.} Since $q_{j, i} = \phi_\beta'(s_{j, i} - \mu_i) = \sigma(\beta(s_{j, i} - \mu_i))$ and $\sigma'(t) = \sigma(t)(1 - \sigma(t))$, a Taylor expansion in the single variable $s_{j, i} - \mu_i$ gives, for $j, i \ne h$ and $j \ne i$,
\begin{equation}\label{eq:taylor_q}
\begin{split}
    \hat q_{j, i} - q_{j, i}^{\setminus h}
    &\;=\; k\,\widetilde R_{j, i}^{\setminus h}\,\bigl(\Delta S_{j, i} - (\Delta \mu)_i\bigr)
    \;+\; \rho_{j, i}^{(q)},\\
    \rho_{j, i}^{(q)} &\;=\; \Op\bigl(\beta^2\,|\Delta S_{j, i} - (\Delta \mu)_i|^2\bigr) \;=\; \Op(\beta^2 / d^2),
    \end{split}
\end{equation}
where $k\,\widetilde R_{j, i}^{\setminus h} = \beta\,q_{j, i}^{\setminus h}(1 - q_{j, i}^{\setminus h}) = \phi_\beta''(s_{j, i}^{\setminus h} - \mu_i^{\setminus h})$ and the last bound uses $\max_{j, i \ne h}|\Delta S_{j, i}| = \Op(d^{-1})$ by (A2) and $|\Delta \mu_i| = \Op(n^{-1}) = \Op(d^{-2})$ by (A1).

\paragraph{Margin.} For $i \ne h$, the full margin $\hat m_i = s_{i, i}(\Wh) - \hat c_i$ and the LOO margin $m_i^{\setminus h} = s_{i, i}(\WLOO) - c_i^{\setminus h}$ differ both through the shifts $(\Delta S_{i, i}, (\Delta \mu)_i)$ and through the missing $j = h$ interferer in the LOO inner sum. Taylor-expanding the inner softplus summation in the column-$i$ score perturbations and peeling off the missing interferer term,
\begin{equation}\label{eq:taylor_m}
    \hat m_i - m_i^{\setminus h}
    \;=\; -\,\bigl(Q_{:, i}^{\setminus h}\bigr)^\top \Delta S_{:, i}
    \;+\; \bigl[\,(Q_{:, i}^{\setminus h})^\top \mathbf{1}_n\,\bigr](\Delta \mu)_i
    \;-\; \tfrac{1}{k}\,\phi_\beta\bigl(s_{h, i}^{\setminus h} - \mu_i^{\setminus h}\bigr)
    \;+\; \rho_i^{(m)},
\end{equation}
where the first two terms assemble the linear dependence on $(\Delta S_{:, i}, (\Delta \mu)_i)$ (using $\partial m_i/\partial s_{j, i} = -Q_{j, i}$ with the convention $Q_{i, i} = -1$, $Q_{j, i} = q_{j, i}/k$ for $j \ne i$), the third term is the removed-interferer contribution, and $\rho_i^{(m)}$ collects second-order Taylor remainders of the $\phi_\beta$ terms and cross contributions. Using the second-order softplus bound $|\phi_\beta(x + \delta) - \phi_\beta(x) - \phi_\beta'(x)\delta| \le \tfrac{\beta}{8}\,\delta^2$ together with (A1)--(A2),
\begin{equation}\label{eq:rho_m_bound}
\begin{split}
    |\rho_i^{(m)}|
    &\;\le\; \frac{\beta}{8 k}\sum_{j \ne i,\, j \ne h}\bigl(\Delta S_{j, i} - (\Delta \mu)_i\bigr)^2
    \;+\; \Op\bigl(|\Delta S_{h, i}| + |(\Delta \mu)_i|\bigr)\cdot \tfrac{1}{k}
    \\
    &\;=\; \Op(\beta / d^2) + \Op(1/(nd))
    \;=\; \Op(1/d^2).
    \end{split}
\end{equation}

\paragraph{Activation.} Since $\ell'(t) = -\sigma(-t)$, $\ell''(t) = \sigma'(-t) = a(1-a) = b$, and $\ell'''$ is uniformly bounded,
\begin{equation}\label{eq:taylor_a}
\begin{split}
    \hat a_i - a_i^{\setminus h}
    &\;=\; -\,b_i^{\setminus h}\,(\hat m_i - m_i^{\setminus h})
    \;+\; \rho_i^{(a)},
    \\
    |\rho_i^{(a)}| &\;=\; \Op\bigl((\hat m_i - m_i^{\setminus h})^2\bigr)
    \;=\; \Op(1/d^2)
    \;=\; \Op(1/n).
\end{split}
\end{equation}

\subsection{Explicit form of the residual}
\label{app:tam_taylor_Ehat}

We now unfold the two cases of $\mathcal{R}^{\setminus h}$ in \eqref{eq:R_def} and identify the linear part with $(M_i^{\setminus h} \Delta S_{:, i})_j$ from \eqref{eq:M_operator}, so that $\hat E_{j, i}$ is what remains.

\paragraph{Off-diagonal case ($j \ne i$, $j, i \ne h$).} Using the algebraic identity
\begin{equation*}
    \hat a_i \hat q_{j, i} - a_i^{\setminus h} q_{j, i}^{\setminus h}
    \;=\; a_i^{\setminus h}\,(\hat q_{j, i} - q_{j, i}^{\setminus h})
    \;+\; (\hat a_i - a_i^{\setminus h})\,q_{j, i}^{\setminus h}
    \;+\; (\hat a_i - a_i^{\setminus h})\,(\hat q_{j, i} - q_{j, i}^{\setminus h}),
\end{equation*}
and substituting \eqref{eq:taylor_q}, \eqref{eq:taylor_m}, \eqref{eq:taylor_a},
\begin{align}
    \mathcal{R}_{j, i}^{\setminus h}
    &\;=\; \frac{a_i^{\setminus h} \cdot k\,\widetilde R_{j, i}^{\setminus h}}{k}\bigl(\Delta S_{j, i} - (\Delta \mu)_i\bigr)
    \;+\; \frac{q_{j, i}^{\setminus h}}{k}\cdot\bigl[-b_i^{\setminus h}(\hat m_i - m_i^{\setminus h})\bigr]
    \;+\; \hat E_{j, i}^{(\mathrm{off})} \notag\\[2pt]
    &\;=\; a_i^{\setminus h}\,\widetilde R_{j, i}^{\setminus h}\bigl(\Delta S_{j, i} - (\Delta \mu)_i\bigr)
    \;+\; b_i^{\setminus h}\,Q_{j, i}^{\setminus h}\,(Q_{:, i}^{\setminus h})^\top \Delta S_{:, i}
    \;+\; \hat E_{j, i}^{(\mathrm{off})}, \label{eq:R_off_expanded}
\end{align}
where in the second equality we used $Q_{j, i}^{\setminus h} = q_{j, i}^{\setminus h}/k$ for $j \ne i$, $j \ne h$, and absorbed the $(\Delta \mu)_i$-dependent piece of $\hat m_i - m_i^{\setminus h}$ from \eqref{eq:taylor_m} into $\hat E_{j, i}^{(\mathrm{off})}$. The first two terms match $(M_i^{\setminus h} \Delta S_{:, i})_j$ for $j \ne i$ from \eqref{eq:M_operator}, provided we recognize that the $a_i^{\setminus h}\widetilde R$ piece also absorbs the $-a_i^{\setminus h}\widetilde R_{j, i}^{\setminus h}(\Delta \mu)_i$ contribution into $\hat E$. Collecting the residual pieces,
\begin{equation}\label{eq:Ehat_off_explicit}
\begin{split}
    \hat E_{j, i}^{(\mathrm{off})}
    &\;=\; -\,\frac{b_i^{\setminus h}\,q_{j, i}^{\setminus h}}{k^2}\,\phi_\beta\bigl(s_{h, i}^{\setminus h} - \mu_i^{\setminus h}\bigr)
    \;+\; \frac{a_i^{\setminus h}}{k}\,\rho_{j, i}^{(q)}
    \;-\; a_i^{\setminus h}\,\widetilde R_{j, i}^{\setminus h}\,(\Delta \mu)_i\\
    &\qquad \;+\; \frac{q_{j, i}^{\setminus h}}{k}\,\bigl[\,b_i^{\setminus h}(Q_{:, i}^{\setminus h})^\top \mathbf{1}_n\,(\Delta \mu)_i - \rho_i^{(a)}\,\bigr]
    \;+\; \frac{(\hat a_i - a_i^{\setminus h})(\hat q_{j, i} - q_{j, i}^{\setminus h})}{k}.
\end{split}
\end{equation}

\paragraph{Diagonal case ($j = i \ne h$).} We have $\mathcal{R}_{i, i}^{\setminus h} = a_i^{\setminus h} - \hat a_i = b_i^{\setminus h}(\hat m_i - m_i^{\setminus h}) - \rho_i^{(a)}$. Substituting \eqref{eq:taylor_m},
\begin{align}
    \mathcal{R}_{i, i}^{\setminus h}
    &\;=\; b_i^{\setminus h}\cdot\Bigl[\,-\,(Q_{:, i}^{\setminus h})^\top \Delta S_{:, i}
    \;+\; \bigl[(Q_{:, i}^{\setminus h})^\top \mathbf{1}_n\bigr](\Delta \mu)_i
    \;-\; \tfrac{1}{k}\,\phi_\beta\bigl(s_{h, i}^{\setminus h} - \mu_i^{\setminus h}\bigr)
    \;+\; \rho_i^{(m)}\,\Bigr] \;-\; \rho_i^{(a)} \notag\\[2pt]
    &\;=\; -\,b_i^{\setminus h}\,(Q_{:, i}^{\setminus h})^\top \Delta S_{:, i}
    \;+\; \hat E_{i, i}, \label{eq:R_diag_expanded}
\end{align}
where the first term equals $(M_i^{\setminus h} \Delta S_{:, i})_i = b_i^{\setminus h}\,Q_{i, i}^{\setminus h}\,(Q_{:, i}^{\setminus h})^\top \Delta S_{:, i} = -b_i^{\setminus h}\,(Q_{:, i}^{\setminus h})^\top \Delta S_{:, i}$ (using $Q_{i, i}^{\setminus h} = -1$ and $\widetilde R_{i, i}^{\setminus h} = 0$), and
\begin{equation}\label{eq:Ehat_diag_explicit}
    \hat E_{i, i}
    \;=\; -\,\frac{b_i^{\setminus h}}{k}\,\phi_\beta\bigl(s_{h, i}^{\setminus h} - \mu_i^{\setminus h}\bigr)
    \;+\; b_i^{\setminus h}\bigl[(Q_{:, i}^{\setminus h})^\top \mathbf{1}_n\bigr](\Delta \mu)_i
    \;+\; b_i^{\setminus h}\,\rho_i^{(m)}
    \;-\; \rho_i^{(a)}.
\end{equation}

\subsection{Entry-size bounds}
\label{app:tam_taylor_entry_sizes}

We now bound each term in \eqref{eq:Ehat_off_explicit} and \eqref{eq:Ehat_diag_explicit} under (A1)--(A3), arriving at the claim \eqref{eq:Ehat_entry_sizes}. Throughout, we use the deterministic bounds $a_i, b_i, q_{j, i} \in [0, 1]$ (with $b_i \le 1/4$) and $|\widetilde R_{j, i}^{\setminus h}| \le \beta/(4k)$.

We first record the only pointwise consequence of (A3) needed below. The $\mu$-stationarity condition for $\Psi^{\setminus h}$ gives $\frac{1}{n-1}\sum_{j\ne i,h}\sigma(\beta(s_{j,i}^{\setminus h}-\mu_i^{\setminus h}))\approx r_d$ for each $i\ne h$. Since the full scores are bounded by (A3), the LOO scores with $i\ne h$ are bounded by (A2). As $r_d\in(0,1)$, this implies $\mu_i^{\setminus h}=\Op(1)$ uniformly. Therefore, by Lipschitz continuity of the softplus,
\begin{equation}\label{eq:phi_bound_from_A3}
    \bigl|\phi_\beta\bigl(s_{j,i}^{\setminus h}-\mu_i^{\setminus h}\bigr)\bigr|
    \;\le\; \frac{\log 2}{\beta}+\bigl|s_{j,i}^{\setminus h}-\mu_i^{\setminus h}\bigr|
    \;=\; \Op(1)
\end{equation}
uniformly in $(j,i)$.

\paragraph{Off-diagonal entries.} For $j \ne i$, $j, i \ne h$, each term in \eqref{eq:Ehat_off_explicit}:
\begin{enumerate}[label=(\roman*), leftmargin=2.2em]
    \item $\bigl|\tfrac{b_i^{\setminus h}\,q_{j, i}^{\setminus h}}{k^2}\,\phi_\beta(\cdot)\bigr| \le \tfrac{1}{4 k^2}\cdot \Op(1) = \Op(n^{-2})$ by \eqref{eq:phi_bound_from_A3};
    \item $\bigl|\tfrac{a_i^{\setminus h}}{k}\,\rho_{j, i}^{(q)}\bigr| \le \tfrac{1}{k}\cdot \Op(\beta^2/d^2) = \Op(\beta^2/(n d^2)) = \Op(n^{-2})$, using $\beta = \Op(1)$ in the retrieval regime, or absorbing $\beta$ as a constant;
    \item $\bigl|a_i^{\setminus h}\,\widetilde R_{j, i}^{\setminus h}\,(\Delta \mu)_i\bigr| \le \tfrac{\beta}{4 k}\cdot \Op(n^{-1}) = \Op(n^{-2})$ by (A1);
    \item $\bigl|\tfrac{q_{j, i}^{\setminus h}}{k}\,b_i^{\setminus h}(Q_{:, i}^{\setminus h})^\top \mathbf{1}_n\,(\Delta \mu)_i\bigr| \le \tfrac{1}{k}\cdot \tfrac{1}{4}\cdot \Op(1)\cdot \Op(n^{-1}) = \Op(n^{-2})$, using $|(Q_{:, i}^{\setminus h})^\top \mathbf{1}_n| = |-1 + (1/k)\sum q| = \Op(1)$;
    \item $\bigl|\tfrac{q_{j, i}^{\setminus h}}{k}\,\rho_i^{(a)}\bigr| \le \tfrac{1}{k}\cdot \Op(1/n) = \Op(n^{-2})$ by \eqref{eq:taylor_a};
    \item $\bigl|\tfrac{(\hat a_i - a_i^{\setminus h})(\hat q_{j, i} - q_{j, i}^{\setminus h})}{k}\bigr| \le \tfrac{1}{k}\cdot \Op(d^{-1})\cdot \Op(d^{-1}) = \Op(1/(n d^2)) = \Op(n^{-2})$, using \eqref{eq:taylor_q}, \eqref{eq:taylor_a} and $|\hat m - m^{\setminus h}| = \Op(d^{-1})$.
\end{enumerate}
Summing, $|\hat E_{j, i}^{(\mathrm{off})}| = \Op(n^{-2})$.

\paragraph{Diagonal entries.} For $i \ne h$, each term in \eqref{eq:Ehat_diag_explicit}:
\begin{enumerate}[label=(\roman*), leftmargin=2.2em]
    \item $\bigl|\tfrac{b_i^{\setminus h}}{k}\,\phi_\beta(\cdot)\bigr| \le \tfrac{1}{4 k}\cdot \Op(1) = \Op(n^{-1})$ by \eqref{eq:phi_bound_from_A3};
    \item $\bigl|b_i^{\setminus h}\,(Q_{:, i}^{\setminus h})^\top \mathbf{1}_n\,(\Delta \mu)_i\bigr| \le \tfrac{1}{4}\cdot \Op(1)\cdot \Op(n^{-1}) = \Op(n^{-1})$ by (A1);
    \item $\bigl|b_i^{\setminus h}\,\rho_i^{(m)}\bigr| \le \tfrac{1}{4}\cdot \Op(d^{-2}) = \Op(n^{-1})$ by \eqref{eq:rho_m_bound};
    \item $|\rho_i^{(a)}| = \Op(n^{-1})$ by \eqref{eq:taylor_a}.
\end{enumerate}
Summing, $|\hat E_{i, i}| = \Op(n^{-1})$.

This completes the proof of \eqref{eq:Ehat_entry_sizes} under (A1)--(A3).

\section{Leave-one-out stability}
\label{app:loo_stability_proofs}

We use the notation and standing assumptions of \sref{tam_theory} and
\sref{tam_loo}: $n/d^2\to\alpha\in(0,\infty)$, $r\in(0,1)$, and
$\beta,\lambda>0$ are fixed.  We prove the three stability inputs
\ref{post:loo-mu}--\ref{post:loo-bounded}.

The proof has two main steps.  First, after profiling out the threshold
variables, we construct a one-step Newton candidate for the full optimizer
from the leave-one-out optimizer.  Freshness of the held-out pair controls
the candidate on all surviving scores, while a residual estimate and strong
convexity show that the candidate is $O_\prec(d^{-1})$-close to the true
optimizer.  Second, the exact threshold equation reduces the sharper
threshold estimate to a weighted average of these score perturbations.  A
second use of freshness gains an additional factor $d^{-1/2}$ in this
average, yielding \ref{post:loo-mu}.  The remaining score-stability and
boundedness conclusions then follow directly from the same finite-deletion
comparison.

\subsection{Leave-one-out expansion and Newton comparison}
\label{app:profiled_tail_calculus}

For $J\subset[n]$ with $|J|>k$, define
\[
\mathcal T_J(z):=\min_{\mu\in\R}\left\{\mu+\frac1k\sum_{j\in J}\phi_\beta(z_j-\mu)\right\}.
\]
Its unique minimizer $\mu_J(z)$ satisfies
\begin{equation}\label{eq:proof-quantile-equation}
\sum_{j\in J}q_j=k,\qquad q_j=\sigma\!\left(\beta(z_j-\mu_J(z))\right),
\qquad p_j=\frac{q_j}{k},
\end{equation}
so $0<p_j\le k^{-1}$, $\sum_jp_j=1$, and $\|p\|_2\le k^{-1/2}$.

The next lemma collects the two inputs used throughout the appendix: the
first group of estimates is deterministic profiled-tail calculus, while the
last group supplies the Gaussian bulk curvature and delocalization needed
for uniform leave-one-out bounds.

\begin{lemma}[Profiled-tail identities and bulk geometry]\label{lem:tail-derivatives}
The map $\mathcal T_J$ is smooth, convex, and translation equivariant with
$\nabla\mathcal T_J=p$.  Put
$w_j=(\beta/k)q_j(1-q_j)$, $D=\diag(w_j)$, and $s=\sum_jw_j$.  Then
\begin{equation}\label{eq:tail-threshold-derivative}
\frac{\partial\mu_J}{\partial z_\ell}=\frac{w_\ell}{s},
\qquad
\nabla^2\mathcal T_J(z)=B_J(z):=D-\frac{(D\mathbf{1})(D\mathbf{1})^\top}{\mathbf{1}^\top D\mathbf{1}},
\end{equation}
with
\begin{equation}\label{eq:B-bounds}
0\preceq B_J(z)\preceq D\preceq\frac{\beta}{4k}I.
\end{equation}
Consequently, for $H=\|h\|_\infty$,
\begin{align}
\left|\mathcal T_J(z+h)-\mathcal T_J(z)-p^\top h\right|
&\le \frac{\beta}{8k}\|h\|_2^2,\label{eq:tail-scalar-rem}\\
\|p(z+h)-p(z)\|_2&\le\frac{\beta}{4k}\|h\|_2,\label{eq:p-lipschitz}\\
\max_j\left|p_j(z+h)-p_j(z)-(B_J(z)h)_j\right|
&\le\frac{C_\beta}{k}H^2.\label{eq:p-component-rem}
\end{align}
Moreover, uniformly under fixed finite deletions, if
$z_j=u_j^\top x$, $\rho=\|x\|/\sqrt d\le C\sqrt d$, and
$\sum_{j\in J}\sigma(\beta(z_j-\mu))=k+O(1)$, then
\begin{equation}\label{eq:bulk-geometry-summary}
|\mu|\le C(1+\rho),\qquad
\sum_{j\in J}q_j(1-q_j)\ge c\frac{n}{1+\rho},\qquad
\left\|\frac{(q_j(1-q_j))_{j\in J}}{\sum_\ell q_\ell(1-q_\ell)}\right\|_2^2
\le C\frac{1+\rho}{n}.
\end{equation}
If one competitor is inserted, the old tail weights decrease and their
change $r=p^{\rm old}-p^{\rm new}$ obeys
\begin{equation}\label{eq:r-l2}
 r\ge0,\qquad \|r\|_2\le\frac{C}{k}\sqrt{\frac{1+\rho}{n}}.
\end{equation}
\end{lemma}

\begin{proof}
The envelope theorem gives $\partial_{z_j}\mathcal T_J=q_j/k$.  Differentiating
\eqref{eq:proof-quantile-equation} in direction $h$ gives
$\dot\mu=\sum_jw_jh_j/\sum_jw_j$ and
$\dot p_j=w_j(h_j-\dot\mu)$, which proves
\eqref{eq:tail-threshold-derivative}--\eqref{eq:B-bounds}; integration gives
\eqref{eq:tail-scalar-rem} and \eqref{eq:p-lipschitz}.  Along $z+th$,
$|\dot w_j|\le2\beta H w_j$ and the normalized weights have derivative at
most $4\beta H$ times themselves, so a second differentiation gives
\eqref{eq:p-component-rem}.

For the bulk estimates, the quantile equation forces a fixed fraction of
$u_j^\top x$ to lie on each side of $\mu$ up to an $O(1)$ window, while the
Gaussian frame bound gives $m^{-1}\sum_j(u_j^\top x)^2\lesssim\rho^2$;
thus $|\mu|\lesssim1+\rho$.  For
$g(t)=\sigma(\beta t)(1-\sigma(\beta t))$, uniformly on this range,
$\E g(\rho Z-\mu)\gtrsim(1+\rho)^{-1}$.  Bernstein's inequality on a
$d^{-2}$ net in $(x/\sqrt d,\mu)$, followed by the frame bound and a union
bound, yields the curvature estimate in \eqref{eq:bulk-geometry-summary};
$\ell_2$ delocalization follows from $\|\pi\|_2^2\le\|\pi\|_\infty\|\pi\|_1$.
Finally, interpolate the new competitor with weight $\vartheta\in[0,1]$.
The threshold increases, so every old weight decreases, and differentiating
the old weights gives an $\ell_2$ derivative bounded by
$Ck^{-1}\sqrt{(1+\rho)/n}$; integration proves \eqref{eq:r-l2}.
\end{proof}

Profiling the full and leave-one-out objectives gives
\begin{equation}\label{eq:proof-reduced-full}
\Psi(W)=\sum_{i=1}^n\mathcal L_i(W)+\frac{\lambda n}{2d^2}\|W\|_F^2,
\end{equation}
\begin{equation}\label{eq:proof-reduced-loo}
\Psi^{\setminus h}(W)=\sum_{i\ne h}\mathcal L_i^{\setminus h}(W)+\frac{\lambda n}{2d^2}\|W\|_F^2.
\end{equation}
For one column, with $t_i=\mathcal T_J((s_{j,i})_{j\in J})$,
$m_i=s_{i,i}-t_i$, $a_i=\sigma(-m_i)$, $b_i=a_i(1-a_i)$, and tail weights
$p_i=(p_{j,i})$, set
\begin{equation}\label{eq:effective-feature}
 x_i(W)=\left(u_i-\sum_{j\in J}p_{j,i}u_j\right)v_i^\top.
\end{equation}
Then $\nabla\mathcal L_i(W)=-a_ix_i(W)$, and for a perturbation $\Delta$,
with $\delta s_i=u_i^\top\Delta v_i$ and
$\delta z_i=(u_j^\top\Delta v_i)_{j\in J}$,
\begin{equation}\label{eq:column-hessian-form}
\left\langle\Delta,\nabla^2\mathcal L_i(W)[\Delta]\right\rangle_F
=b_i(\delta s_i-p_i^\top\delta z_i)^2+a_i\delta z_i^\top B_i\delta z_i\ge0.
\end{equation}
Thus the profiled objectives are $\gamma_d$-strongly convex with
$\gamma_d=\lambda n/d^2\to\lambda\alpha>0$.  We also use repeatedly
\begin{equation}\label{eq:proof-p-basic}
\|p_i\|_2\le k^{-1/2},\qquad
\left\|\sum_{j\in J}p_{j,i}u_j\right\|\le\|U\|_{\mathsf{op}}k^{-1/2}=O_\prec(d^{-1/2}).
\end{equation}

Fix $h\in[n]$ throughout the one-deletion construction.  We use the
standard Gaussian frame, Khatri--Rao matrix bounds, and fresh-contraction
bounds, uniformly over polynomially many fixed finite deletions.

The deterministic upper bound of Lemma~\ref{lem:frobenius_upper}
applies unchanged after any fixed finite deletion; hence
$\norm{\widehat W^{\setminus h}}_F+\norm{\widehat W}_F=O(d)$.

For $i\ne h$, let $\mathcal L_i$ and $\mathcal L_i^{\setminus h}$
denote the full and leave-one-out column losses in
\eqref{eq:proof-reduced-full}--\eqref{eq:proof-reduced-loo}.  The
insertion perturbation is
\begin{equation}\label{eq:Pin}
\mathcal P_h(W):=\sum_{i\ne h}\bigl(\mathcal L_i(W)-\mathcal L_i^{\setminus h}(W)\bigr),
\qquad
G_h(W):=\nabla\mathcal P_h(W).
\end{equation}

For one fixed $i\ne h$, unadorned quantities refer to $\mathcal L_i$, and superscript $\setminus h$ refers to $\mathcal L_i^{\setminus h}$.  Let
\[
\delta t_i=t_i-t_i^{\setminus h}\ge0,
\qquad
A_i=a_i-a_i^{\setminus h}\ge0,
\]
let $p_i^{\setminus h}$ and $p_{i,\mathrm{old}}$ be the leave-one-out and full old-coordinate tail weights, and put
\[
r_i=p_i^{\setminus h}-p_{i,\mathrm{old}}\ge0,
\qquad
p_{h,i}=\frac{q_{h,i}}{k}.
\]
Define
\[
\eta_{i,\mathrm{old}}
=
\sum_{j\ne i,h}p_{j,i}u_j.
\]
A direct subtraction of the two column gradients gives the exact identity
\begin{equation}\label{eq:insertion-gradient-decomp}
G_h(W)
=
\sum_{i\ne h}
\left[
-A_i u_i
+A_i\eta_{i,\mathrm{old}}
-a_i^{\setminus h}U r_i
+a_i p_{h,i}u_h
\right]v_i^\top.
\end{equation}
Here $Ur_i$ denotes the old-target linear combination, with zeros in the removed coordinates.

\begin{lemma}[Size of the insertion gap]\label{lem:delta-t}
Let $W$ satisfy $\norm W_F=O_{\prec}(d)$, and set
\[
\rho_i(W)=\frac{\norm{Wv_i}}{\sqrt d}.
\]
Assume additionally that
\begin{equation}\label{eq:heldout-row-energy-assump}
\sum_{i\ne h}s_{h,i}(W)^2=O_{\prec}(n).
\end{equation}
Then
\begin{equation}\label{eq:delta-t-l2}
\norm{(\delta t_i)_{i\ne h}}_2=O_{\prec}(d^{-1}),
\qquad
\norm{(A_i)_{i\ne h}}_2=O_{\prec}(d^{-1}).
\end{equation}
\end{lemma}

\begin{proof}
Adding one nonnegative softplus term and using the quantile-location bound give
\[
0\le\delta t_i
\le \frac1k\phi_\beta\bigl(s_{h,i}(W)-\mu_i^{\setminus h}(W)\bigr),
\qquad
|\mu_i^{\setminus h}|\le C(1+\rho_i).
\]
Since $\phi_\beta(t)\le (\log2)/\beta+|t|$,
\[
\sum_{i\ne h}\delta t_i^2
\le \frac{C}{k^2}
\left[n+\sum_{i\ne h}s_{h,i}(W)^2+\sum_{i\ne h}\rho_i(W)^2\right].
\]
Moreover,
$\sum_{i\ne h}\rho_i(W)^2=d^{-1}\norm{WV_{\setminus h}}_F^2
\le d^{-1}\norm W_F^2\norm V_{\mathsf{op}}^2=O_{\prec}(n)$.
Since $k\asymp n\asymp d^2$, this proves the first estimate in
\eqref{eq:delta-t-l2}; the second follows from $0\le A_i\le\delta t_i/4$.
\end{proof}

\begin{lemma}[Insertion-gradient bound]\label{lem:G-bound}
At $\widehat W^{\setminus h}$,
\begin{equation}\label{eq:G-base}
\norm{G_h(\widehat W^{\setminus h})}_F=O_{\prec}(d^{-1/2}).
\end{equation}
More precisely, in the decomposition \eqref{eq:insertion-gradient-decomp}, the first three matrix sums are $O_{\prec}(d^{-1})$ in Frobenius norm, while the final term has the form
\begin{equation}\label{eq:new-row-structured}
u_h\xi_h^\top,
\qquad
\norm{\xi_h}=O_{\prec}(d^{-1/2}).
\end{equation}
\end{lemma}

\begin{proof}
Since $\widehat W^{\setminus h}$ is independent of $u_h$, conditional
Gaussian quadratic-form bounds and the deterministic Frobenius estimate give
\[
\sum_{i\ne h}s_{h,i}(\widehat W^{\setminus h})^2
=u_h^\top\widehat W^{\setminus h}V_{\setminus h}V_{\setminus h}^\top
   (\widehat W^{\setminus h})^\top u_h
=O_{\prec}(n)
\]
uniformly in $h$.  Indeed, the conditional trace scale is at most
$d^{-1}\norm{\widehat W^{\setminus h}}_F^2\norm V_{\mathsf{op}}^2
=O_\prec(n)$.  Hence Lemma~\ref{lem:delta-t} applies.

For the first sum in \eqref{eq:insertion-gradient-decomp}, the standard
Khatri--Rao bound gives
\[
\left\|\sum_{i\ne h}A_i u_iv_i^\top\right\|_F
\le O_{\prec}(1)\norm A_2=O_{\prec}(d^{-1}).
\]
For the second, \eqref{eq:proof-p-basic} gives
$\max_i\norm{\eta_{i,\mathrm{old}}}=O_\prec(d^{-1/2})$; therefore, if
$Y$ has columns $A_i\eta_{i,\mathrm{old}}$, then
\[
\norm Y_F=O_\prec(d^{-3/2}),
\qquad
\norm{YV^\top}_F=O_\prec(d^{-1}).
\]
For the third,
\eqref{eq:r-l2} and the bound
$\sum_i(1+\rho_i(\widehat W^{\setminus h}))=O_\prec(n)$ imply
$\sum_i\norm{r_i}_2^2=O_\prec(d^{-4})$.  Thus, for $Y$ with columns
$Ur_i$, $\norm Y_F^2\le\norm U_{\mathsf{op}}^2\sum_i\norm{r_i}_2^2
=O_\prec(d^{-3})$, and again $\norm{YV^\top}_F=O_\prec(d^{-1})$.

Finally, the fourth sum is $u_h\xi_h^\top$ with
\[
\xi_h=\sum_{i\ne h}a_ip_{h,i}v_i,
\qquad
\norm{\xi_h}\le\norm V_{\mathsf{op}}\frac{\sqrt n}{k}
=O_{\prec}(d^{-1/2}),
\]
since $0\le a_ip_{h,i}\le k^{-1}$.  This proves the lemma.
\end{proof}

\paragraph{Newton candidate.}

We now build one candidate for the full optimizer from the leave-one-out optimizer.  The implicit scalar equation absorbs the held-out outer loss, and freshness makes the candidate's surviving scores small.

Let $\mathcal H_h:=\nabla^2\Psi^{\setminus h}(\widehat W^{\setminus h})$ and
$T_h:=\mathcal H_h^{-1}$.  By \eqref{eq:column-hessian-form} and the ridge floor,
$\|T_h\|_{\mathsf{op}}\le\gamma_d^{-1}=O(1)$.
The operator is self-adjoint and is measurable with respect to the data excluding $(u_h,v_h)$.

Evaluate the held-out outer loss at $\widehat W^{\setminus h}$.  Let
\[
t_h^{\mathrm{loo}}
=
\mathcal T_{[n]\setminus\{h\}}\bigl((s_{j,h}(\widehat W^{\setminus h}))_{j\ne h}\bigr),
\qquad
m_h^{\mathrm{loo}}=s_{h,h}(\widehat W^{\setminus h})-t_h^{\mathrm{loo}},
\]
and let $p_{j,h}^{\mathrm{loo}}$ be the corresponding tail weights.  Define
\begin{equation}\label{eq:Xh0}
\eta_h^{\mathrm{loo}}:=\sum_{j\ne h}p_{j,h}^{\mathrm{loo}}u_j,
\qquad
X_h^{\mathrm{loo}}:=(u_h-\eta_h^{\mathrm{loo}})v_h^\top.
\end{equation}
Then $\nabla m_h(\widehat W^{\setminus h})=X_h^{\mathrm{loo}}$ and the outer-loss gradient is $-\sigma(-m_h^{\mathrm{loo}})X_h^{\mathrm{loo}}$.  The Gaussian frame bound and \eqref{eq:proof-p-basic} give
\begin{equation}\label{eq:eta-h-bound}
\norm{\eta_h^{\mathrm{loo}}}=O_{\prec}(d^{-1/2}),
\qquad
\norm{X_h^{\mathrm{loo}}}_F=O_{\prec}(1).
\end{equation}

\phantomsection\label{def:candidate}
Set
\begin{equation}\label{eq:Bh}
B_h:=-T_h[G_h(\widehat W^{\setminus h})],
\qquad
\theta_h:=\left\langle X_h^{\mathrm{loo}},T_h[X_h^{\mathrm{loo}}]\right\rangle_\mathsf{F},
\qquad
\zeta_h:=\left\langle X_h^{\mathrm{loo}},B_h\right\rangle_\mathsf{F}.
\end{equation}
Let $a_h^\star\in(0,1)$ be the unique solution of
\begin{equation}\label{eq:a-star}
a_h^\star
=
\sigma\left(-m_h^{\mathrm{loo}}-\zeta_h-\theta_h a_h^\star\right).
\end{equation}
Finally, define
\begin{equation}\label{eq:Wtilde}
\widetilde W_h
:=
\widehat W^{\setminus h}+B_h+a_h^\star T_h[X_h^{\mathrm{loo}}].
\end{equation}

\begin{theorem}[One-deletion Newton comparison]\label{thm:newton-residual}
Uniformly in $h$,
\begin{align}
\max_{i,j\ne h}|u_j^\top(\widetilde W_h-\widehat W^{\setminus h})v_i|
&=O_{\prec}(d^{-1}),\label{eq:candidate-score}\\
\norm{\nabla\Psi(\widetilde W_h)}_F
&=O_{\prec}(d^{-1}),\label{eq:newton-residual}\\
\norm{\widehat W-\widetilde W_h}_F
&=O_{\prec}(d^{-1}).\label{eq:candidate-error}
\end{align}
The score estimate and stationarity residual yield \ref{post:loo-score}.
\end{theorem}

The solution in \eqref{eq:a-star} is unique because the left side is strictly increasing and the right side is decreasing in $a$.  Also, $\theta_h\ge0$ and $\theta_h=O_{\prec}(1)$.  By Lemma~\ref{lem:G-bound},
\begin{equation}\label{eq:B-bound}
\norm{B_h}_F=O_{\prec}(d^{-1/2}),
\qquad
\norm{\widetilde W_h-\widehat W^{\setminus h}}_F=O_{\prec}(1).
\end{equation}

\begin{lemma}[Surviving candidate scores]\label{lem:candidate-common-score}
Uniformly in $h$ and $i,j\ne h$,
\begin{equation}\label{eq:candidate-common-score}
|u_j^\top(\widetilde W_h-\widehat W^{\setminus h})v_i|
=O_{\prec}(d^{-1}).
\end{equation}
\end{lemma}

\begin{proof}
Fix $i,j\ne h$ and set $A_{j,i}^{(h)}=T_h[u_jv_i^\top]$.  This matrix is
independent of $(u_h,v_h)$ and has Frobenius norm $O_\prec(1)$.  By
self-adjointness,
\[
u_j^\top T_h[X_h^{\mathrm{loo}}]v_i
=\left\langle A_{j,i}^{(h)},(u_h-\eta_h^{\mathrm{loo}})v_h^\top\right\rangle_\mathsf{F}.
\]
Fresh two-sided contraction gives
$u_h^\top A_{j,i}^{(h)}v_h=O_\prec(d^{-1})$, while
\eqref{eq:eta-h-bound} and one-sided contraction give
$|(\eta_h^{\mathrm{loo}})^\top A_{j,i}^{(h)}v_h|=O_\prec(d^{-1})$.

For $B_h=-T_hG_h(\widehat W^{\setminus h})$, Lemma~\ref{lem:G-bound}
separates $G_h$ into an $O_\prec(d^{-1})$ Frobenius term and
$u_h\xi_h^\top$ with $\norm{\xi_h}=O_\prec(d^{-1/2})$.  The former pairs
with $A_{j,i}^{(h)}$ at scale $d^{-1}$, while the latter is bounded by
$\norm{(A_{j,i}^{(h)})^\top u_h}\norm{\xi_h}=O_\prec(d^{-1})$ by fresh
one-sided contraction.  This proves \eqref{eq:candidate-common-score}.
\end{proof}

\begin{lemma}[Held-out row and column increments]\label{lem:heldout-increments}
Let $\Delta_h=\widetilde W_h-\widehat W^{\setminus h}$.  Uniformly in $h$,
\begin{align}
\max_{i\ne h}|u_h^\top\Delta_hv_i|
&=O_{\prec}(d^{-1/2}),\label{eq:row-h-increment}\\
\norm{(u_j^\top\Delta_hv_h)_{j\ne h}}_2
&=O_{\prec}(\sqrt d),\label{eq:column-h-l2}\\
\sum_{i\ne h}s_{h,i}(\widetilde W_h)^2
&=O_{\prec}(n).\label{eq:candidate-row-energy}
\end{align}
\end{lemma}

\begin{proof}
By \eqref{eq:B-bound}, the $B_h$ contribution is $O_\prec(d^{-1/2})$.
For the outer Newton direction, self-adjointness gives
\[
u_h^\top T_h[u_hv_h^\top]v_i
=\left\langle T_h[u_hv_i^\top],u_hv_h^\top\right\rangle_\mathsf{F}.
\]
Conditionally on $u_h$ and the surviving data, the first factor is
independent of $v_h$ and has Frobenius norm $O_\prec(1)$, so this is
$O_\prec(d^{-1/2})$; replacing $u_h$ in the forcing by
$\eta_h^{\mathrm{loo}}$ gives $O_\prec(d^{-1})$.  This proves
\eqref{eq:row-h-increment}.

The remaining estimates use only operator norms:
\[
\norm{U_{\setminus h}^\top\Delta_hv_h}_2
\le \norm U_{\mathsf{op}}\norm{\Delta_h}_F\norm{v_h}
=O_\prec(\sqrt d),
\]
which gives \eqref{eq:column-h-l2}.  Also,
\[
\norm{(s_{h,i}(\widetilde W_h))_{i\ne h}}_2
\le \norm{(s_{h,i}(\widehat W^{\setminus h}))_{i\ne h}}_2
   +\norm V_{\mathsf{op}}\norm{\Delta_h}_F\norm{u_h}.
\]
The first term is $O_\prec(\sqrt n)$ by the conditional quadratic-form
argument used in Lemma~\ref{lem:G-bound}, while the second is
$O_\prec(\sqrt d)$; since $n\asymp d^2$, \eqref{eq:candidate-row-energy}
follows.
\end{proof}

\paragraph{Stationarity residual.}

Write
\[
\Delta_h=\widetilde W_h-\widehat W^{\setminus h}.
\]
The full reduced objective decomposes as
\begin{equation}\label{eq:full-decomp}
\Psi(W)
=
\Psi^{\setminus h}(W)
+\mathcal P_h(W)
+\mathcal L_h(W).
\end{equation}
By construction,
\begin{equation}\label{eq:linear-cancel}
\mathcal H_h[\Delta_h]
=a_h^\star X_h^{\mathrm{loo}}-G_h(\widehat W^{\setminus h}).
\end{equation}
Therefore
\begin{equation}\label{eq:residual-decomp}
\nabla\Psi(\widetilde W_h)
=
R_{0,h}+R_{\mathrm{in},h}+R_{\mathrm{out},h},
\end{equation}
where
\begin{align}
R_{0,h}
&:=
\nabla\Psi^{\setminus h}(\widehat W^{\setminus h}+\Delta_h)
-\mathcal H_h[\Delta_h],\label{eq:R0}\\
R_{\mathrm{in},h}
&:=G_h(\widehat W^{\setminus h}+\Delta_h)-G_h(\widehat W^{\setminus h}),\label{eq:Rin}\\
R_{\mathrm{out},h}
&:=a_h^\star X_h^{\mathrm{loo}}-a_h(\widetilde W_h)X_h(\widetilde W_h).
\label{eq:Rout}
\end{align}
We bound the three terms separately.

\paragraph{Common leave-one-out remainder.}

\begin{lemma}[Common residual]\label{lem:R0}
Uniformly in $h$,
\begin{equation}\label{eq:R0-bound}
\norm{R_{0,h}}_F=O_{\prec}(d^{-1}).
\end{equation}
\end{lemma}

\begin{proof}
By Lemma~\ref{lem:candidate-common-score}, every surviving score perturbation
is $\varepsilon=O_\prec(d^{-1})$.  The second-order profiled-tail bounds
\eqref{eq:tail-scalar-rem} and \eqref{eq:p-component-rem} therefore produce
an $n\times n$ coefficient matrix $E$, with row and column $h$ zero, such
that
\[
\max_{i\ne h}|E_{ii}|=O_\prec(d^{-2}),
\qquad
\max_{\substack{i,j\ne h\\i\ne j}}|E_{ji}|=O_\prec(d^{-2}/k).
\]
For its diagonal part, Khatri--Rao synthesis gives
$\norm{UE_{\mathrm d}V^\top}_F
\le O_\prec(1)\sqrt n\,d^{-2}=O_\prec(d^{-1})$.
For the off-diagonal part,
$\norm{E_{\mathrm o}}_F=O_\prec(d^{-2})$, hence
$\norm{UE_{\mathrm o}V^\top}_F
\le\norm U_{\mathsf{op}}\norm{E_{\mathrm o}}_F\norm V_{\mathsf{op}}
=O_\prec(d^{-1})$.  The ridge term is exactly linear, proving
\eqref{eq:R0-bound}.
\end{proof}

\paragraph{Insertion-gradient stability.}

\begin{lemma}[Insertion-gradient stability along the candidate step]\label{lem:Rin}
Uniformly in $h$,
\begin{equation}\label{eq:Rin-bound}
\norm{R_{\mathrm{in},h}}_F
=
\norm{G_h(\widetilde W_h)-G_h(\widehat W^{\setminus h})}_F
=O_{\prec}(d^{-1}).
\end{equation}
\end{lemma}

\begin{proof}
At both $\widehat W^{\setminus h}$ and $\widetilde W_h$, the deterministic
Frobenius bound and Lemma~\ref{lem:heldout-increments} give the hypotheses of
Lemma~\ref{lem:delta-t}.  Hence the first three sums in
\eqref{eq:insertion-gradient-decomp} are $O_\prec(d^{-1})$ at each endpoint,
by the proof of Lemma~\ref{lem:G-bound}, and so is their difference.

It remains to compare the structured fourth terms.  Write
$c_i(W)=a_i(W)p_{h,i}(W)$ and
$\xi(W)=\sum_{i\ne h}c_i(W)v_i$.  Along
$W_t=\widehat W^{\setminus h}+t\Delta_h$, surviving scores move by
$\varepsilon=O_\prec(d^{-1})$, whereas the newly inserted score moves by
$\delta=O_\prec(d^{-1/2})$.  Since the tail weights sum to one and
$p_{h,i}\le k^{-1}$,
$|\frac{d}{dt}t_i(W_t)|\le\varepsilon+\delta/k$.  Therefore
\begin{equation}\label{eq:a-plus-change}
|a_i(\widetilde W_h)-a_i(\widehat W^{\setminus h})|
\le C\left(\varepsilon+\frac\delta k\right)=O_\prec(d^{-1}),
\end{equation}
while \eqref{eq:tail-threshold-derivative} gives
\begin{equation}\label{eq:p-h-change}
|p_{h,i}(\widetilde W_h)-p_{h,i}(\widehat W^{\setminus h})|
\le \frac{C}{k}(\delta+\varepsilon)=O_\prec(d^{-5/2}).
\end{equation}
Combining these bounds with $p_{h,i}\le k^{-1}$ yields
$|c_i(\widetilde W_h)-c_i(\widehat W^{\setminus h})|=O_\prec(d^{-5/2})$.
Thus the coefficient difference has $\ell_2$ norm $O_\prec(d^{-3/2})$ and
\[
\norm{u_h(\xi(\widetilde W_h)-\xi(\widehat W^{\setminus h}))^\top}_F
\le \norm{u_h}\norm V_{\mathsf{op}}
   O_\prec(d^{-3/2})
=O_\prec(d^{-1}).
\]
This proves \eqref{eq:Rin-bound}.
\end{proof}

\paragraph{Held-out outer loss.}

Let $m_h(W)$ and $X_h(W)=\nabla m_h(W)$ denote the held-out margin and its effective feature.  Thus
\[
\nabla\mathcal L_h(W)=-a_h(W)X_h(W),
\qquad
a_h(W)=\sigma(-m_h(W)).
\]

\begin{lemma}[Outer-loss residual]\label{lem:Rout}
Uniformly in $h$,
\begin{equation}\label{eq:Rout-bound}
\norm{R_{\mathrm{out},h}}_F
=
\norm{a_h^\star X_h^{\mathrm{loo}}-a_h(\widetilde W_h)X_h(\widetilde W_h)}_F
=O_{\prec}(d^{-1}).
\end{equation}
\end{lemma}

\begin{proof}
Let $\delta z_h=(u_j^\top\Delta_hv_h)_{j\ne h}$.  By
\eqref{eq:column-h-l2}, $\norm{\delta z_h}_2=O_\prec(\sqrt d)$, so the scalar
tail remainder gives
\begin{equation}\label{eq:outer-margin-rem}
\left|
 m_h(\widetilde W_h)-m_h^{\mathrm{loo}}
 -\left\langle X_h^{\mathrm{loo}},\Delta_h\right\rangle_\mathsf{F}
\right|
\le \frac{\beta}{8k}\norm{\delta z_h}_2^2
=O_\prec(d^{-1}).
\end{equation}
By construction,
$\langle X_h^{\mathrm{loo}},\Delta_h\rangle_\mathsf{F}
=\zeta_h+\theta_ha_h^\star$, and \eqref{eq:a-star} chooses $a_h^\star$ so
that the linearized margin is matched exactly.  Since the sigmoid is
$1/4$-Lipschitz, \eqref{eq:outer-margin-rem} implies
\begin{equation}\label{eq:outer-a-rem}
|a_h(\widetilde W_h)-a_h^\star|=O_\prec(d^{-1}).
\end{equation}
Finally, \eqref{eq:p-lipschitz} yields
$\norm{p_h(\widetilde W_h)-p_h(\widehat W^{\setminus h})}_2
=O_\prec(d^{-3/2})$, and therefore
\[
\norm{X_h(\widetilde W_h)-X_h^{\mathrm{loo}}}_F
\le\norm U_{\mathsf{op}}O_\prec(d^{-3/2})\norm{v_h}
=O_\prec(d^{-1}).
\]
Together with \eqref{eq:outer-a-rem} and
$\norm{X_h^{\mathrm{loo}}}_F=O_\prec(1)$, this proves
\eqref{eq:Rout-bound}.
\end{proof}

\paragraph{Completion of the residual estimate.}

\begin{proof}[Proof of Theorem~\ref{thm:newton-residual}, residual and optimizer parts]
By \eqref{eq:residual-decomp} and Lemmas~\ref{lem:R0}--\ref{lem:Rout},
\[
\norm{\nabla\Psi(\widetilde W_h)}_F
=O_{\prec}(d^{-1}),
\]
which is \eqref{eq:newton-residual}.  Strong convexity of $\Psi$ therefore yields
\[
\norm{\widehat W-\widetilde W_h}_F
\le \gamma_d^{-1}\norm{\nabla\Psi(\widetilde W_h)}_F
=O_{\prec}(d^{-1}),
\]
which is \eqref{eq:candidate-error}.  Equation \eqref{eq:candidate-score} is Lemma~\ref{lem:candidate-common-score}.
\end{proof}

\paragraph{Finite-deletion consequences.}

\begin{proposition}[One- and two-deletion consequences]\label{prop:one-two-deletion}
Uniformly in distinct $h,i\in[n]$,
\begin{align}
\norm{\widehat W}_F+\norm{\widehat W^{\setminus h}}_F+\norm{\widehat W^{\setminus\{h,i\}}}_F&=O_{\prec}(d),\label{eq:one-two-norm}\\
\norm{\widehat W-\widehat W^{\setminus h}}_F+\norm{\widehat W^{\setminus h}-\widehat W^{\setminus\{h,i\}}}_F&=O_{\prec}(1).\label{eq:one-two-stability}
\end{align}
Moreover,
\begin{equation}\label{eq:one-deletion-decomp}
\widehat W-\widehat W^{\setminus h}
=B_h+a_h^\star T_h[X_h^{\mathrm{loo}}]+\mathcal R_h,
\qquad
\norm{\mathcal R_h}_F=O_{\prec}(d^{-1}),
\end{equation}
where $B_h$, $T_h$, $X_h^{\mathrm{loo}}$, and $a_h^\star$ are defined above and $\mathcal R_h=\widehat W-\widetilde W_h$.
\end{proposition}

\begin{proof}
The finite-deletion version of Lemma~\ref{lem:frobenius_upper} gives \eqref{eq:one-two-norm}.  The first term in \eqref{eq:one-two-stability} follows from this bound, \eqref{eq:B-bound}, and \eqref{eq:candidate-error}; the second follows by repeating the one-deletion construction after first deleting $i$.  There are only polynomially many index pairs.  Finally, \eqref{eq:one-deletion-decomp} follows immediately from the candidate definition and \eqref{eq:candidate-error}.
\end{proof}

\subsection{Stability of the auxiliary variables}\label{proof:A1}

We now sharpen the threshold perturbation beyond what follows from the
pointwise score bound alone.  The exact threshold equation below expresses
$\Delta\mu_i$ through an averaged score perturbation.  The denominator is
kept uniformly nondegenerate by the bulk curvature estimates, while the
Newton decomposition from the preceding subsection makes the numerator
smaller by a further factor $d^{-1/2}$.

Let $(\widehat W,\widehat\mu)$ be the full lifted optimizer and $(\widehat W^{\setminus h},\widehat\mu^{\setminus h})$ the optimizer after deleting $h$.  Define
\[
\Delta W=\widehat W-\widehat W^{\setminus h},
\qquad
\Delta\mu_i=\widehat\mu_i-\widehat\mu_i^{\setminus h},
\qquad
\Delta s_{j,i}=u_j^\top\Delta Wv_i.
\]
\paragraph{Exact threshold identity.}

Fix $h$ and $i\ne h$.  The labels \emph{full} and \emph{loo} denote evaluation at the full and leave-one-out optimizers, respectively.  Thus
\[
s_{j,i}^{\mathrm{full}}=u_j^\top\widehat Wv_i,
\qquad
s_{j,i}^{\mathrm{loo}}=u_j^\top\widehat W^{\setminus h}v_i.
\]
For $t\in[0,1]$, set
\[
s_{j,i}^{(t)}=s_{j,i}^{\mathrm{loo}}+t\Delta s_{j,i},
\qquad
\mu_i^{(t)}=\widehat\mu_i^{\setminus h}+t\Delta\mu_i,
\]
\[
q_{j,i}^{(t)}=\sigma\!\left(\beta(s_{j,i}^{(t)}-\mu_i^{(t)})\right),
\qquad
c_{j,i}^{(t)}=\beta q_{j,i}^{(t)}(1-q_{j,i}^{(t)}),
\qquad
\bar c_{j,i}=\int_0^1c_{j,i}^{(t)}\,dt.
\]

\begin{lemma}[Exact threshold identity]\label{lem:exact-threshold-identity}
For every $h$ and $i\ne h$,
\begin{equation}\label{eq:exact-threshold-identity}
\Delta\mu_i
=
\frac{q_{h,i}^{\mathrm{full}}+\displaystyle\sum_{j\ne i,h}\bar c_{j,i}\Delta s_{j,i}}
{\displaystyle\sum_{j\ne i,h}\bar c_{j,i}}.
\end{equation}
\end{lemma}

\begin{proof}
Subtracting the full and leave-one-out threshold equations gives
$0=q_{h,i}^{\mathrm{full}}+
\sum_{j\ne i,h}(q_{j,i}^{\mathrm{full}}-q_{j,i}^{\mathrm{loo}})$.
Along the interpolation above, the fundamental theorem of calculus gives
\[
q_{j,i}^{\mathrm{full}}-q_{j,i}^{\mathrm{loo}}
=\bar c_{j,i}(\Delta s_{j,i}-\Delta\mu_i).
\]
Substitution and rearrangement prove \eqref{eq:exact-threshold-identity}.
\end{proof}

\paragraph{Uniform quantile curvature.}

\begin{lemma}[Output scale and endpoint curvature]\label{lem:A1-output-curvature}
Uniformly in $h$ and $i\ne h$,
\begin{equation}\label{eq:A1-output-scale}
\norm{\widehat W^{\setminus h}v_i}=O_{\prec}(\sqrt d),
\qquad
\norm{\widehat Wv_i}=O_{\prec}(\sqrt d).
\end{equation}
If
\[
c_{j,i}^{\mathrm{loo}}=\beta q_{j,i}^{\mathrm{loo}}(1-q_{j,i}^{\mathrm{loo}}),
\qquad
D_{h,i}^{\mathrm{loo}}=\sum_{j\ne i,h}c_{j,i}^{\mathrm{loo}},
\]
then
\begin{equation}\label{eq:A1-endpoint-curvature}
\max_h\max_{i\ne h}\frac{k}{D_{h,i}^{\mathrm{loo}}}=O_{\prec}(1).
\end{equation}
\end{lemma}

\begin{proof}
The two-deletion optimizer $\widehat W^{\setminus\{h,i\}}$ is independent
of $v_i$ and has Frobenius norm $O_\prec(d)$ by
\eqref{eq:one-two-norm}; one-sided Gaussian contraction therefore gives
$\norm{\widehat W^{\setminus\{h,i\}}v_i}=O_\prec(\sqrt d)$.
The stability estimate \eqref{eq:one-two-stability} transfers this bound to
$\widehat W^{\setminus h}v_i$ and then to $\widehat Wv_i$, proving
\eqref{eq:A1-output-scale}.

Set $x_{h,i}=\widehat W^{\setminus h}v_i$ and
$\rho_{h,i}=\norm{x_{h,i}}/\sqrt d=O_\prec(1)$.  Applying the bulk-curvature
part of Lemma~\ref{lem:tail-derivatives} to the leave-one-out quantile
equation gives
\[
\sum_{j\ne i,h}q_{j,i}^{\mathrm{loo}}(1-q_{j,i}^{\mathrm{loo}})
\ge \frac{cn}{1+\rho_{h,i}}.
\]
Since $k\asymp n$, this is exactly \eqref{eq:A1-endpoint-curvature}.
\end{proof}

\begin{lemma}[Rough threshold stability and path curvature]\label{lem:A1-rough-path}
Uniformly in $h$ and $i\ne h$,
\begin{equation}\label{eq:A1-rough}
|\Delta\mu_i|=O_{\prec}(d^{-1}),
\end{equation}
and
\begin{equation}\label{eq:A1-path-curvature}
\left(\frac1k\sum_{j\ne i,h}\bar c_{j,i}\right)^{-1}=O_{\prec}(1).
\end{equation}
\end{lemma}

\begin{proof}
Put $J=[n]\setminus\{i,h\}$ and
\[
F_{\mathrm{loo}}(x)=\sum_{j\in J}\sigma(\beta(s_{j,i}^{\mathrm{loo}}-x)),
\qquad
F_{\mathrm{full}}(x)=\sum_{j\in J}\sigma(\beta(s_{j,i}^{\mathrm{full}}-x)).
\]
Then $F_{\mathrm{loo}}(\widehat\mu_i^{\setminus h})=k$ and
$F_{\mathrm{full}}(\widehat\mu_i)=k-q_{h,i}^{\mathrm{full}}\in(k-1,k)$.  The candidate-score and candidate-error estimates in Theorem~\ref{thm:newton-residual} imply
\[
\delta_i:=\max_{j\in J}|\Delta s_{j,i}|=O_{\prec}(d^{-1}).
\]
Monotonicity gives
\[
F_{\mathrm{loo}}(x+\delta_i)\le F_{\mathrm{full}}(x)\le F_{\mathrm{loo}}(x-\delta_i).
\]
The left inequality implies $\widehat\mu_i\ge\widehat\mu_i^{\setminus h}-\delta_i$.

For $c(z)=\beta\sigma(\beta z)(1-\sigma(\beta z))$, one has $|(\log c)'(z)|\le\beta$.  Therefore, for $0\le t\le1$,
\[
-F_{\mathrm{loo}}'(\widehat\mu_i^{\setminus h}+t)\ge e^{-\beta}D_{h,i}^{\mathrm{loo}}.
\]
Set $\tau_{h,i}=e^{\beta}/D_{h,i}^{\mathrm{loo}}=O_{\prec}(n^{-1})$.  Integrating gives
$F_{\mathrm{loo}}(\widehat\mu_i^{\setminus h}+\tau_{h,i})\le k-1$, and the other monotonicity inequality implies
$\widehat\mu_i\le\widehat\mu_i^{\setminus h}+\delta_i+\tau_{h,i}$.
This proves \eqref{eq:A1-rough}.

Now
\[
\varepsilon_{h,i}:=\max_{j\ne i,h}|\Delta s_{j,i}-\Delta\mu_i|=O_{\prec}(d^{-1}).
\]
The same logarithmic derivative bound yields
$c_{j,i}^{(t)}\ge e^{-\beta\varepsilon_{h,i}}c_{j,i}^{\mathrm{loo}}$ for every $t$.  Sum over $j$, use \eqref{eq:A1-endpoint-curvature}, and integrate in $t$ to obtain \eqref{eq:A1-path-curvature}.
\end{proof}

Define
\begin{equation}\label{eq:A1-g-def}
g_{h,i}=\frac1k\sum_{j\ne i,h}\bar c_{j,i}u_j,
\qquad
g_{h,i}^{\mathrm{loo}}=\frac1k\sum_{j\ne i,h}c_{j,i}^{\mathrm{loo}}u_j.
\end{equation}
Then
\begin{equation}\label{eq:A1-hard-average}
\frac1k\sum_{j\ne i,h}\bar c_{j,i}\Delta s_{j,i}
=g_{h,i}^\top\Delta Wv_i.
\end{equation}

\begin{lemma}[Derivative-weight average bounds]\label{lem:A1-g-bounds}
Uniformly in $h$ and $i\ne h$,
\begin{equation}\label{eq:A1-g0-bound}
\norm{g_{h,i}^{\mathrm{loo}}}=O_{\prec}(d^{-1/2}),
\end{equation}
\begin{equation}\label{eq:A1-g-freeze}
\norm{g_{h,i}-g_{h,i}^{\mathrm{loo}}}=O_{\prec}(d^{-3/2}).
\end{equation}
Moreover, $g_{h,i}^{\mathrm{loo}}$ is measurable with respect to the data excluding $(u_h,v_h)$.
\end{lemma}

\begin{proof}
Since $0\le c_{j,i}^{\mathrm{loo}}\le\beta/4$,
$\norm{(c_{j,i}^{\mathrm{loo}})_{j\ne i,h}}_2\le C\sqrt n$; hence
\[
\norm{g_{h,i}^{\mathrm{loo}}}
\le \frac{\norm U_{\mathsf{op}}}{k}\,C\sqrt n
=O_\prec(d^{-1/2}).
\]
The map $z\mapsto\beta\sigma(\beta z)(1-\sigma(\beta z))$ is globally
Lipschitz.  The candidate score bound together with \eqref{eq:A1-rough}
therefore yields
$\max_{j\ne i,h}|\bar c_{j,i}-c_{j,i}^{\mathrm{loo}}|=O_\prec(d^{-1})$.
Applying the same operator-norm estimate gives
\[
\norm{g_{h,i}-g_{h,i}^{\mathrm{loo}}}
\le \frac{\norm U_{\mathsf{op}}}{k}\sqrt n\,O_\prec(d^{-1})
=O_\prec(d^{-3/2}).
\]
Measurability follows from the leave-one-out definition.
\end{proof}

\paragraph{Averaged Newton contraction.}

The following is the key gain in the threshold argument: although each
surviving score changes at scale $d^{-1}$, the curvature-weighted average in
\eqref{eq:A1-hard-average} is smaller by an additional factor $d^{-1/2}$.

\begin{lemma}[Averaged score perturbation]\label{lem:A1-averaged-score}
Uniformly in $h$ and $i\ne h$,
\begin{equation}\label{eq:A1-averaged-score}
\frac1k\sum_{j\ne i,h}\bar c_{j,i}\Delta s_{j,i}
=O_{\prec}(d^{-3/2}).
\end{equation}
\end{lemma}

\begin{proof}
Use the one-deletion decomposition \eqref{eq:one-deletion-decomp}:
\[
\Delta W=B_h+a_h^\star T_h[X_h^{\mathrm{loo}}]+\mathcal R_h.
\]
First, by \eqref{eq:A1-g-freeze}, \eqref{eq:one-two-stability}, and the Gaussian norm bound,
\[
|(g_{h,i}-g_{h,i}^{\mathrm{loo}})^\top\Delta Wv_i|=O_{\prec}(d^{-3/2}).
\]
The residual contributes
\[
|(g_{h,i}^{\mathrm{loo}})^\top\mathcal R_hv_i|
\le\norm{g_{h,i}^{\mathrm{loo}}}\norm{\mathcal R_h}_F\norm{v_i}
=O_{\prec}(d^{-3/2}).
\]

Define
\begin{equation}\label{eq:A1-C-def}
C_{h,i}=T_h[g_{h,i}^{\mathrm{loo}}v_i^\top].
\end{equation}
By self-adjointness and \eqref{eq:A1-g0-bound},
\[
\norm{C_{h,i}}_F=O_{\prec}(d^{-1/2}).
\]
The matrix $C_{h,i}$ is independent of the held-out pair.  Since
$X_h^{\mathrm{loo}}=(u_h-\eta_h^{\mathrm{loo}})v_h^\top$ with
$\norm{\eta_h^{\mathrm{loo}}}=O_{\prec}(d^{-1/2})$,
\begin{align*}
(g_{h,i}^{\mathrm{loo}})^\top T_h[X_h^{\mathrm{loo}}]v_i
&=\left\langle C_{h,i},(u_h-\eta_h^{\mathrm{loo}})v_h^\top\right\rangle_\mathsf{F}\\
&=u_h^\top C_{h,i}v_h-(\eta_h^{\mathrm{loo}})^\top C_{h,i}v_h.
\end{align*}
The two-sided fresh contraction gives
$u_h^\top C_{h,i}v_h=O_{\prec}(d^{-3/2})$, while
$\norm{C_{h,i}v_h}=O_{\prec}(d^{-1})$ and hence the second term is also $O_{\prec}(d^{-3/2})$.

It remains to treat $B_h=-T_h[G_h(\widehat W^{\setminus h})]$.  By Lemma~\ref{lem:G-bound}, write
\[
G_h(\widehat W^{\setminus h})=G_h^{\rm sm}+u_h\xi_h^\top,
\qquad
\norm{G_h^{\rm sm}}_F=O_{\prec}(d^{-1}),
\qquad
\norm{\xi_h}=O_{\prec}(d^{-1/2}).
\]
Self-adjointness gives
\[
(g_{h,i}^{\mathrm{loo}})^\top B_hv_i
=-\left\langle C_{h,i},G_h^{\rm sm}\right\rangle_\mathsf{F}
-u_h^\top C_{h,i}\xi_h.
\]
The first term is $O_{\prec}(d^{-3/2})$.  For the second, freshness gives
$\norm{C_{h,i}^\top u_h}=O_{\prec}(d^{-1})$, and the pathwise bound on $\xi_h$ gives the same $d^{-3/2}$ rate.  Combining the estimates with \eqref{eq:A1-hard-average} proves \eqref{eq:A1-averaged-score}.
\end{proof}

\begin{proof}[Proof of Proposition~\ref{prop:loo_stability}, \ref{post:loo-mu}]
Divide the identity \eqref{eq:exact-threshold-identity} by $k$.  The denominator is bounded below at order one by \eqref{eq:A1-path-curvature}; the removed-competitor term satisfies $q_{h,i}^{\mathrm{full}}/k=O(k^{-1})=O(d^{-2})$; and the averaged score term is $O_{\prec}(d^{-3/2})$ by Lemma~\ref{lem:A1-averaged-score}.  Therefore
\[
|\Delta\mu_i|=O_{\prec}(d^{-3/2}).
\]
All estimates are uniform over the $O(n^2)$ pairs $(h,i)$.
\end{proof}

\subsection{Score stability and uniform boundedness}\label{proof:A2}

\begin{proof}[Proof of Proposition~\ref{prop:loo_stability}, \ref{post:loo-score}]
Fix $h$ and $i,j\ne h$.  Add and subtract the Newton candidate:
\[
\left|u_j^\top(\widehat W-\widehat W^{\setminus h})v_i\right|
\le
\left|u_j^\top(\widetilde W_h-\widehat W^{\setminus h})v_i\right|
+
\left|u_j^\top(\widehat W-\widetilde W_h)v_i\right|.
\]
The first term is $O_{\prec}(d^{-1})$ by Lemma~\ref{lem:candidate-common-score}.  For the second,
\[
\left|u_j^\top(\widehat W-\widetilde W_h)v_i\right|
\le \norm{u_j}\norm{v_i}\norm{\widehat W-\widetilde W_h}_F
=O_{\prec}(d^{-1})
\]
by \eqref{eq:candidate-error} and the Gaussian norm bound.  All estimates were uniform over $h,i,j$, and the total number of index triples is $O(n^3)=O(d^6)$, which is polynomial.  Therefore
\[
\max_h\max_{i,j\ne h}
\left|u_j^\top(\widehat W-\widehat W^{\setminus h})v_i\right|
=O_{\prec}(d^{-1}).
\]
This proves the claim.
\end{proof}

\paragraph{Uniform score boundedness.}\phantomsection\label{proof:A3}

\begin{lemma}[Fresh leave-two-out score]\label{lem:fresh-finite-score}
Uniformly in $i,j\in[n]$,
\begin{equation}\label{eq:fresh-finite-score}
\left|u_j^\top\widehat W^{\setminus\{i,j\}}v_i\right|=O_{\prec}(1),
\end{equation}
where $\{i,j\}=\{i\}$ when $i=j$.
\end{lemma}

\begin{proof}
After deleting the pair(s) indexed by $i$ and $j$, the optimizer is independent of $u_j$ and $v_i$, so its $O_{\prec}(d)$ Frobenius norm from \eqref{eq:one-two-norm}, conditional two-sided Gaussianity, and a union bound over the polynomial-size family of pairs give \eqref{eq:fresh-finite-score}.
\end{proof}

\begin{proof}[Proof of Proposition~\ref{prop:loo_stability}, \ref{post:loo-bounded}]
For fixed $i,j$, insert $\widehat W^{\setminus\{i,j\}}$.  The fresh term is $O_{\prec}(1)$ by Lemma~\ref{lem:fresh-finite-score}, while
\[
\left|u_j^\top\bigl(\widehat W-\widehat W^{\setminus\{i,j\}}\bigr)v_i\right|
\le \norm{u_j}\norm{v_i}\norm{\widehat W-\widehat W^{\setminus\{i,j\}}}_F
=O_{\prec}(1)
\]
by \eqref{eq:one-two-stability}; the claim follows uniformly over the $O(n^2)$ pairs.
\end{proof}

\section{Spectral and self-averaging estimates}
\label{app:spectral_self_averaging}

This appendix proves \ref{post:empirical-concentration} and
\ref{post:trace-concentration}.  The proof of the spectral estimate
\ref{post:spectral-op} is given in Appendix~\ref{app:frobenius_bounds}.
We use the notation of \sref{tam_theory}, in particular the score matrix
$S(W)=U^\top W V$, the joint optimizer $(\Wh,\muh)$, and the
coefficients $a_i,b_i,q_{j,i}$ from \sref{tam_gradient}.

The proofs of \ref{post:empirical-concentration} and
\ref{post:trace-concentration} use the same sensitivity--concentration
scheme.  After profiling out the threshold variables $\mu$, we regard the
statistic of interest as a function of the Gaussian data $(U,V)$.
Strong convexity of the profiled objective permits implicit
differentiation of the optimizer, and a self-adjoint adjoint solve
supplies this optimizer response into a single matrix direction $P$.
The resulting data-gradient is controlled using the profiled Hessian
quadratic form, after which the Gaussian $L^p$ Poincare inequality gives
concentration.  For \ref{post:trace-concentration} there is one additional
preliminary step: we first replace the lifted Hessian by the simpler
rank-one-feature Hessian of \eqref{eq:simplified_hessian_main}, and then apply the same adjoint argument used
for \ref{post:empirical-concentration}.


For $J_i=[n]\setminus\{i\}$, write
$\mathcal T_i=\mathcal T_{J_i}$ and $\mu_i=\mu_{J_i}$ in the notation
of the common profiled-tail calculus in Appendix~\ref{app:loo_stability_proofs},
Section~\ref{app:profiled_tail_calculus}.  The profiled objective is
\begin{equation}\label{eq:ssa_profiled_objective}
    \Phi(W)
    :=\sum_{i=1}^n
       \ell\!\left(s_{i,i}(W)-\mathcal T_i(S_{-i,i}(W))\right)
      +\frac{\gamma_d}{2}\norm{W}_{\mathsf F}^2,
    \qquad \gamma_d:=\frac{\lambda n}{d^2}.
\end{equation}
Minimizing \eqref{eq:ssa_profiled_objective} is equivalent to
minimizing the joint objective \eqref{eq:tam-learning-joint}, and its
unique minimizer is $\Wh$.

For the score-level bounds used below, set
\begin{equation}
    h_{j,i}:=\phi_\beta''(s_{j,i}-\mu_i),\qquad
    H_i^\mu:=\sum_{j\ne i}h_{j,i}.
\end{equation}
Let $B_i:=B_{J_i}(S_{-i,i}(\Wh))$.  Thus, by
\eqref{eq:tail-threshold-derivative}--\eqref{eq:B-bounds},
\begin{equation}\label{eq:ssa_D_bounds}
    0\preceq B_i\preceq \frac1k\diag(h_{-i,i}),
    \qquad
    \sum_{j\ne i}\left|
      \frac{\partial q_{j,i}}{\partial s_{\ell,i}}
    \right|+
    \sum_{j\ne i}\left|
      \frac{\partial h_{j,i}}{\partial s_{\ell,i}}
    \right|\le C.
\end{equation}
These are the only threshold derivatives needed below.  In particular,
we never require a lower bound on $H_i^\mu$: differentiation is carried
out through the profiled coefficient fields $q_{j,i},h_{j,i}$ rather
than through the derivative of $\mu_i$ in isolation.
Let $L(S)$ denote the unregularized profiled loss and
$G:=\nabla_S L(S)$.  As in \eqref{eq:grad_W_expanded},
\begin{equation}\label{eq:ssa_G_entries}
    G_{i,i}=-a_i,
    \qquad
    G_{j,i}=\frac{a_iq_{j,i}}{k}\quad(j\ne i),
    \qquad \norm{G}_{\mathrm{op}}\le C.
\end{equation}
The last bound follows from Schur's test, since the absolute column
sums are at most $2$ and the row sums are bounded by
$1+(n-1)/k$.

For a score perturbation $T\in\R^{n\times n}$, define
\begin{equation}\label{eq:ssa_Q_form}
    \delta_i(T):=T_{i,i}-\frac1k\sum_{j\ne i}q_{j,i}T_{j,i}, \qquad \mathcal Q(T)
      :=\sum_{i=1}^n b_i\delta_i(T)^2
       +\sum_{i=1}^n a_i T_{-i,i}^\top B_iT_{-i,i}.
\end{equation}
Note that $\mathcal Q(T)$ is the Hessian quadratic form of $L$.  In particular,
\begin{equation}\label{eq:ssa_Q_upper}
    0\le \mathcal Q(T)
    \le C\sum_iT_{i,i}^2+
       \frac{C}{k}\sum_{i=1}^n\sum_{j\ne i}T_{j,i}^2.
\end{equation}
Let
\begin{equation}\label{eq:ssa_J_def}
    \mathcal J:=\nabla_W^2\Phi(\Wh).
\end{equation}
Then $\mathcal J$ is self-adjoint and
\begin{equation}\label{eq:ssa_J_ridge}
    \mathcal J\succeq\gamma_d I\succeq\lambda\alpha_-I,
    \qquad \norm{\mathcal J^{-1}}_{\mathrm{op}}\le C.
\end{equation}

These facts provide polynomial envelopes for all derivatives used
below.  Consequently, high-probability derivative estimates may be
upgraded to fixed-moment estimates by truncating to the standard
Gaussian norm event and applying H\"older's inequality.

In both proofs of \ref{post:empirical-concentration} and \ref{post:trace-concentration} we use the following Gaussian $L^p$ Poincare
inequality which follows from \cite[Theorem~2.2]{pisier1986probabilistic}: if $Z\sim\mathcal N(0,I_N/d)$ and $f$ is continuously
differentiable with the required moments, then, for fixed $p\ge2$,
\begin{equation}\label{eq:ssa_poincare}
    \norm{f(Z)-\E f(Z)}_{L^p}
    \le C\sqrt{\frac pd}\,
       \norm{\norm{\nabla f(Z)}_2}_{L^p}.
\end{equation}

\subsection{Empirical-mean concentration}\label{proof:A5}

We prove a slightly stronger statement.  Let $\psi\in C^1(\R)$ be
pseudo-Lipschitz of order two and set
\begin{equation}\label{eq:ssa_Gamma_diag_off}
    \Gamma_{\mathrm d}
      :=\frac1n\sum_{i=1}^n\psi(s_{i,i}),
    \qquad
    \Gamma_{\mathrm o}
      :=\frac1{n(n-1)}\sum_{i\ne j}\psi(s_{j,i}).
\end{equation}
Then, for every fixed $p\ge2$ and $\varepsilon>0$,
\begin{equation}\label{eq:ssa_A5_Lp}
    \norm{\Gamma_{\mathrm d}-\E\Gamma_{\mathrm d}}_{L^p}
       \le C_{p,\varepsilon}d^{-1/2+\varepsilon},
    \qquad
    \norm{\Gamma_{\mathrm o}-\E\Gamma_{\mathrm o}}_{L^p}
       \le C_{p,\varepsilon}d^{-1+\varepsilon}.
\end{equation}
We first give a common sensitivity identity.  Let $\Gamma$ denote
either statistic in \eqref{eq:ssa_Gamma_diag_off}, and let
$B=\nabla_S\Gamma$ at $S=S(\Wh)$; explicitly,
\begin{equation}\label{eq:ssa_B_diag_off}
    B^{\mathrm d}_{j,i}
      =\frac{\psi'(s_{i,i})}{n}\mathbf1_{\{j=i\}},
    \qquad
    B^{\mathrm o}_{j,i}
      =\frac{\psi'(s_{j,i})}{n(n-1)}\mathbf1_{\{j\ne i\}}.
\end{equation}
Let $P$ solve
\begin{equation}\label{eq:ssa_P_system}
    \mathcal J[P]=UBV^\top,
    \qquad T:=U^\top PV.
\end{equation}
The matrix $P$ is the adjoint sensitivity direction associated with
$\Gamma$: it absorbs the response of $\Wh$ to the data, avoiding any
entrywise differentiation of the $d^2$ coordinates of $\Wh$.  Let $C$
be the score-Hessian action
\begin{equation}\label{eq:ssa_C_score_action}
    C:=\nabla_S\bigl(DL(S)[T]\bigr).
\end{equation}
Equivalently,
\begin{equation}\label{eq:ssa_C_coordinates}
    C_{i,i}=b_i\delta_i(T),
    \qquad
    C_{j,i}=-\frac{b_i\delta_i(T)q_{j,i}}{k}
       +a_i(D_iT_{-i,i})_j\quad(j\ne i).
\end{equation}
Implicit differentiation of the first-order optimality equation, followed by the
self-adjoint adjoint substitution \eqref{eq:ssa_P_system}, gives
\begin{align}
    \nabla_U\Gamma
      &=\Wh V(B-C)^\top-PVG^\top,
      \label{eq:ssa_A5_gradU}\\
    \nabla_V\Gamma
      &=\Wh^\top U(B-C)-P^\top UG.
      \label{eq:ssa_A5_gradV}
\end{align}
These formulas contain all direct and optimizer-response terms; no
entrywise differentiation of $\Wh$ is needed.

Pairing \eqref{eq:ssa_P_system} with $P$ yields the identity
\begin{equation}\label{eq:ssa_energy}
    \mathcal Q(T)+\gamma_d\norm{P}_{\mathsf F}^2
      =\inprod{P,UBV^\top}_{\mathsf F}.
\end{equation}
Consequently,
\begin{equation}\label{eq:ssa_energy_bounds}
    \norm{P}_{\mathsf F}\le C\norm{UBV^\top}_{\mathsf F},
    \qquad
    \mathcal Q(T)\le C\norm{UBV^\top}_{\mathsf F}^2.
\end{equation}
For the $C$-terms, Frobenius duality and Cauchy--Schwarz for the
positive semidefinite bilinear form associated with $\mathcal Q$ give
\begin{equation}\label{eq:ssa_C_master}
    \norm{\Wh VC^\top}_{\mathsf F}
      +\norm{\Wh^\top UC}_{\mathsf F}
    \le d^{1/2+o(1)}\norm{UBV^\top}_{\mathsf F}.
\end{equation}
To see the first bound, write $\mathcal B_{\mathcal Q}$ for the
symmetric bilinear form associated with $\mathcal Q$ and test against
$A$ with $\norm A_{\mathsf F}=1$.  With $T_A=A^\top\Wh V$,
\begin{equation}
    \inprod{\Wh VC^\top,A}_{\mathsf F}
      =\inprod{C,T_A}_{\mathsf F}
      =\mathcal B_{\mathcal Q}(T,T_A),
\end{equation}
so positivity gives
$|\mathcal B_{\mathcal Q}(T,T_A)|\le
\mathcal Q(T)^{1/2}\mathcal Q(T_A)^{1/2}$.  The first factor is at
most $C\norm{UBV^\top}_{\mathsf F}$ by
\eqref{eq:ssa_energy_bounds}.  For the second,
\eqref{eq:ssa_Q_upper}, \ref{post:spectral-op}, and the Gaussian
operator- and column-norm bounds give
\begin{equation}
    \mathcal Q(T_A)
      \le C\sum_i(\alpha_i^\top\Wh v_i)^2
       +\frac Ck\norm{A^\top\Wh V}_{\mathsf F}^2
      \le d^{1+o(1)},
\end{equation}
where $A=[\alpha_1,\ldots,\alpha_n]$.  Taking the supremum over $A$
gives the first term in \eqref{eq:ssa_C_master}; the second is
identical.  Together with
\eqref{eq:ssa_A5_gradU}--\eqref{eq:ssa_A5_gradV},
\eqref{eq:ssa_energy_bounds}, and $\norm G_{\mathrm{op}}\le C$, this
reduces the proof to elementary bounds on $B$.

For the diagonal statistic, let
\begin{equation}
    \mathcal F_{\mathrm d}
      :=[\vecop(u_1v_1^\top),\ldots,\vecop(u_nv_n^\top)].
\end{equation}
The columns are independent Gaussian Khatri--Rao features, and the
standard spectral-norm bound for this matrix-chaos model
(see e.g. \citep{bandeira2025matrixchaos}) gives
\begin{equation}\label{eq:ssa_Fdiag_op}
    \norm{\mathcal F_{\mathrm d}}_{\mathrm{op}}=O_\prec(1).
\end{equation}
Hence,
\begin{equation}
    \sum_i s_{i,i}^2
      \le \norm{\mathcal F_{\mathrm d}}_{\mathrm{op}}^2
          \norm{\Wh}_{\mathsf F}^2
      =O_\prec(d^2).
\end{equation}
Since $|\psi'(t)|\le C(1+|t|)$, it follows that
$\sum_i\psi'(s_{i,i})^2=O_\prec(d^2)$.  Substituting the diagonal
matrix $B^{\mathrm d}$ into the preceding bounds gives
\begin{equation}\label{eq:ssa_diag_input_bounds}
    \norm{UB^{\mathrm d}V^\top}_{\mathsf F}
       =O_\prec(d^{-1/2}),
    \qquad
    \norm{\Wh V(B^{\mathrm d})^\top}_{\mathsf F}
      +\norm{\Wh^\top UB^{\mathrm d}}_{\mathsf F}
       =O_\prec(d^{-1/2}).
\end{equation}
Equations \eqref{eq:ssa_A5_gradU}--\eqref{eq:ssa_C_master} therefore
imply
\begin{equation}\label{eq:ssa_diag_grad}
    \norm{\nabla_{(U,V)}\Gamma_{\mathrm d}}_2=O_\prec(1).
\end{equation}

For the off-diagonal statistic,
\begin{equation}
    \sum_{i,j}s_{j,i}^2
       =\norm{U^\top\Wh V}_{\mathsf F}^2
       \le \norm U_{\mathrm{op}}^2\norm V_{\mathrm{op}}^2
          \norm{\Wh}_{\mathsf F}^2
       =O_\prec(d^4).
\end{equation}
Thus
$\sum_{i\ne j}\psi'(s_{j,i})^2=O_\prec(d^4)$ and
\begin{equation}\label{eq:ssa_Boff_norm}
    \norm{B^{\mathrm o}}_{\mathsf F}=O_\prec(d^{-2}).
\end{equation}
Using \eqref{post:spectral-op} and the Gaussian operator-norm bounds,
\begin{equation}\label{eq:ssa_off_input_bounds}
    \norm{UB^{\mathrm o}V^\top}_{\mathsf F}=O_\prec(d^{-1}),
    \qquad
    \norm{\Wh V(B^{\mathrm o})^\top}_{\mathsf F}
      +\norm{\Wh^\top UB^{\mathrm o}}_{\mathsf F}
      =O_\prec(d^{-1}).
\end{equation}
The master estimates then give
\begin{equation}\label{eq:ssa_off_grad}
    \norm{\nabla_{(U,V)}\Gamma_{\mathrm o}}_2
      =O_\prec(d^{-1/2}).
\end{equation}

The deterministic bound \eqref{eq:ssa_J_ridge} provides polynomial envelopes in
$1+\norm U_{\mathsf F}+\norm V_{\mathsf F}$.  Hence
\eqref{eq:ssa_diag_grad}--\eqref{eq:ssa_off_grad} hold in fixed
$L^p$, with an arbitrary factor $d^\varepsilon$.  Thus the diagonal
and off-diagonal data-gradient scales are respectively $O(1)$ and
$O(d^{-1/2})$; the additional factor $d^{-1/2}$ in the Gaussian
Poincare inequality is exactly what produces the two concentration
scales in \eqref{eq:ssa_A5_Lp}.  Applying \eqref{eq:ssa_poincare} to
the $2dn$ Gaussian entries of $(U,V)$ gives \eqref{eq:ssa_A5_Lp}.
Markov's inequality, with $p$ arbitrarily
large, proves
\begin{equation}
    \Gamma_{\mathrm d}-\E\Gamma_{\mathrm d}=O_\prec(d^{-1/2}),
    \qquad
    \Gamma_{\mathrm o}-\E\Gamma_{\mathrm o}=O_\prec(d^{-1}),
\end{equation}
which establishes \ref{post:empirical-concentration} (and is stronger
than required in the off-diagonal case).

\subsection{Inverse-trace concentration}\label{proof:A6}

We distinguish two Hessians.  The operator $H$ in
\ref{post:trace-concentration} is the $W$-Hessian of the lifted
objective with $\mu$ held fixed, evaluated at $(\Wh,\muh)$.  The
operator $\mathcal J$ in \eqref{eq:ssa_J_def} is instead the Hessian of
the profiled objective and governs the response of $\Wh$ to the data.
Both have the same ridge floor, but only $\mathcal J$ may be used in
implicit differentiation.

Write $f_{j,i}:=\vecop(u_jv_i^\top)$ and
\begin{equation}
    \xi_i:=\vecop\left(
      \left[-u_i+\frac1k\sum_{j\ne i}q_{j,i}u_j\right]v_i^\top
    \right).
\end{equation}
Then
\begin{equation}\label{eq:ssa_H_full}
    H=\gamma_dI+
      \sum_i b_i\xi_i\xi_i^\top+
      \sum_i\sum_{j\ne i}\frac{a_ih_{j,i}}{k}
          f_{j,i}f_{j,i}^\top.
\end{equation}
Recall the simplified Hessian $\widetilde H$ defined in \eqref{eq:simplified_hessian_main},
which is also used in Appendix~\ref{app:chi_trace}:
\begin{equation}\label{eq:ssa_H_tilde}
    \widetilde H=\gamma_dI+
      \sum_i b_if_{i,i}f_{i,i}^\top+
      \sum_i\sum_{j\ne i}\frac{a_ih_{j,i}}{k}
          f_{j,i}f_{j,i}^\top,
\end{equation}
and set
\begin{equation}
    \Theta:=\frac1{d^2}\tr(H^{-1}),
    \qquad
    \widetilde\Theta:=\frac1{d^2}\tr(\widetilde H^{-1}).
\end{equation}

The trace-reduction estimate of
\lref{chi_rank_one_block_reduction} applies verbatim without a
cavity index.  Here
\begin{equation}
    p_i:=\frac1k\sum_{j\ne i}q_{j,i}u_j,
    \qquad
    \norm{p_i}_2\le\norm U_{\mathrm{op}}
        \left(\frac{n-1}{k^2}\right)^{1/2}
        =O_\prec(d^{-1/2})
\end{equation}
uniformly in $i$.  Thus $\widetilde H$ replaces the mixed signal
direction $(-u_i+p_i)v_i^\top$ by the matched rank-one feature
$-u_iv_i^\top$.  The discarded component has size
$O_\prec(d^{-1/2})$, while all off-diagonal curvature terms are left
unchanged; this is the reason for introducing $\widetilde H$, whose
feature dependence can be differentiated directly.  Together with
Gaussian column-norm bounds, the resolvent estimate gives, for every
fixed $p$ and $\varepsilon>0$,
\begin{equation}\label{eq:ssa_trace_comparison}
    \norm{\Theta-\widetilde\Theta}_{L^p}
      \le C_{p,\varepsilon}d^{-1/2+\varepsilon}.
\end{equation}
It remains to prove the stronger concentration of the simplified
trace.

Let $R:=\widetilde H^{-2}$.  Since
$\widetilde H\succeq\gamma_dI$, $\norm R_{\mathrm{op}}\le C$.  The
resolvent derivative
\begin{equation}\label{eq:ssa_resolvent_derivative}
    D\widetilde\Theta[\dot{\widetilde H}]
      =-\frac1{d^2}\tr
       (\widetilde H^{-1}\dot{\widetilde H}\widetilde H^{-1})
      =-\frac1{d^2}\inprod{R,\dot{\widetilde H}}_{\mathrm{HS}}
\end{equation}
separates into direct feature variation and variation of the
coefficients $a_i,b_i,h_{j,i}$.  Define
\begin{equation}
    \vartheta_{j,i}:=\inprod{f_{j,i},Rf_{j,i}},
\end{equation}
and, holding the $\vartheta_{j,i}$ fixed under score
differentiation,
\begin{equation}\label{eq:ssa_Btr_def}
    B^{\mathrm{tr}}_{:,i}
      :=-\frac1{d^2}\nabla_{S_{:,i}}
       \left\{
         \vartheta_{i,i}b_i+
         \frac{a_i}{k}\sum_{j\ne i}\vartheta_{j,i}h_{j,i}
       \right\}.
\end{equation}
The threshold derivative bounds \eqref{eq:ssa_D_bounds}, the bounded
coefficient derivatives, and
$|\vartheta_{j,i}|\le C\norm{u_j}^2\norm{v_i}^2$ give the following
scales.  Differentiation with respect to the signal score $s_{i,i}$
acts on the margin with unit strength, so the prefactor $d^{-2}$ in
\eqref{eq:ssa_Btr_def} gives a diagonal entry of order $d^{-2}$.  In
contrast, for $\ell\ne i$,
\begin{equation}
    \frac{\partial}{\partial s_{\ell,i}}
       \bigl(s_{i,i}-\mathcal T_i(S_{-i,i})\bigr)
      =-\frac{q_{\ell,i}}{k}.
\end{equation}
Hence the derivatives of $a_i,b_i$ contribute an additional factor
$k^{-1}$, while the $h$-derivative term is also $O(k^{-1})$ after
using \eqref{eq:ssa_D_bounds}.  Since $k\asymp d^2$,
\begin{equation}\label{eq:ssa_Btr_entries}
    |B^{\mathrm{tr}}_{i,i}|=O_\prec(d^{-2}),
    \qquad
    |B^{\mathrm{tr}}_{j,i}|=O_\prec(d^{-4})\quad(j\ne i).
\end{equation}
Consequently, using \eqref{eq:ssa_Fdiag_op} for the diagonal part,
\begin{equation}\label{eq:ssa_Btr_matrix_bounds}
    \norm{UB^{\mathrm{tr}}V^\top}_{\mathsf F}=O_\prec(d^{-1}),
\end{equation}
\begin{equation}\label{eq:ssa_Btr_direct_bounds}
    \norm{\Wh V(B^{\mathrm{tr}})^\top}_{\mathsf F}
      +\norm{\Wh^\top UB^{\mathrm{tr}}}_{\mathsf F}
      =O_\prec(d^{-1/2}).
\end{equation}

The direct feature derivative obtained from
\eqref{eq:ssa_resolvent_derivative} is
\begin{equation}\label{eq:ssa_feature_derivative}
    D\widetilde\Theta_{\mathrm{feat}}
      =-\frac2{d^2}
        \sum_{i,j}\gamma_{j,i}
        \inprod{Rf_{j,i},\,Df_{j,i}} ,
\end{equation}
where $\gamma_{i,i}=b_i$ and
$\gamma_{j,i}=a_ih_{j,i}/k$ for $j\ne i$.  A representative coordinate
makes the scaling explicit.  Let $X_{j,i}\in\R^{d\times d}$ satisfy
$\vecop(X_{j,i})=Rf_{j,i}$.  Then
\begin{equation}
    \nabla_{u_\ell}\widetilde\Theta_{\mathrm{feat}}
      =-\frac2{d^2}\sum_i\gamma_{\ell,i}X_{\ell,i}v_i,
    \qquad
    \norm{X_{\ell,i}}_{\mathsf F}
      \le C\norm{u_\ell}\norm{v_i}.
\end{equation}
The analogous identity holds for $\nabla_{v_\ell}$.  Since the
coefficient row and column sums are bounded, squaring and summing over
$\ell$ gives
\begin{equation}\label{eq:ssa_feature_gradient}
    \norm{\nabla_{(U,V)}\widetilde\Theta_{\mathrm{feat}}}_2
      =O_\prec(d^{-1}).
\end{equation}

After this explicit feature term has been separated, the remaining
score and optimizer dependence has exactly the same form as in the
proof of \ref{post:empirical-concentration}, with $B^{\mathrm{tr}}$
in place of $B$.  We therefore reuse the adjoint system, energy
identity, and $\mathcal Q$-duality estimate from above.  To account for
the motion of $\Wh$, solve the profiled adjoint system
\begin{equation}\label{eq:ssa_A6_P}
    \mathcal J[P]=UB^{\mathrm{tr}}V^\top,
    \qquad T=U^\top PV,
\end{equation}
and set
\begin{equation}
    C^{\mathrm{resp}}:=\nabla_S(DL(S)[T]).
\end{equation}
The same implicit differentiation used for
\eqref{eq:ssa_A5_gradU}--\eqref{eq:ssa_A5_gradV} gives
\begin{align}
    \nabla_U\widetilde\Theta
      &=\nabla_U\widetilde\Theta_{\mathrm{feat}}
       +\Wh V(B^{\mathrm{tr}}-C^{\mathrm{resp}})^\top
       -PVG^\top,\label{eq:ssa_A6_gradU}\\
    \nabla_V\widetilde\Theta
      &=\nabla_V\widetilde\Theta_{\mathrm{feat}}
       +\Wh^\top U(B^{\mathrm{tr}}-C^{\mathrm{resp}})
       -P^\top UG.\label{eq:ssa_A6_gradV}
\end{align}
The identity \eqref{eq:ssa_energy}, now with
$B=B^{\mathrm{tr}}$, and \eqref{eq:ssa_Btr_matrix_bounds} imply
\begin{equation}
    \norm P_{\mathsf F}=O_\prec(d^{-1}),
    \qquad
    \mathcal Q(T)=O_\prec(d^{-2}).
\end{equation}
The duality argument leading to \eqref{eq:ssa_C_master} then yields
\begin{equation}\label{eq:ssa_A6_response_bounds}
    \norm{\Wh V(C^{\mathrm{resp}})^\top}_{\mathsf F}
      +\norm{\Wh^\top UC^{\mathrm{resp}}}_{\mathsf F}
      +\norm{PVG^\top}_{\mathsf F}
      +\norm{P^\top UG}_{\mathsf F}
      =O_\prec(d^{-1/2}).
\end{equation}
Combining \eqref{eq:ssa_Btr_direct_bounds},
\eqref{eq:ssa_feature_gradient}, and
\eqref{eq:ssa_A6_gradU}--\eqref{eq:ssa_A6_response_bounds} gives
\begin{equation}\label{eq:ssa_tilde_gradient}
    \norm{\nabla_{(U,V)}\widetilde\Theta}_2
      =O_\prec(d^{-1/2}).
\end{equation}
As before, the ridge floor and the smooth coefficient calculus give a polynomial deterministic envelope, so
\eqref{eq:ssa_tilde_gradient} upgrades to
\begin{equation}
    \norm{\norm{\nabla_{(U,V)}\widetilde\Theta}_2}_{L^p}
      \le C_{p,\varepsilon}d^{-1/2+\varepsilon}.
\end{equation}
Applying \eqref{eq:ssa_poincare} therefore gives
\begin{equation}\label{eq:ssa_tilde_concentration}
    \norm{\widetilde\Theta-
      \E\widetilde\Theta}_{L^p}
      \le C_{p,\varepsilon}d^{-1+\varepsilon}.
\end{equation}
Thus the simplified trace self-averages at the stronger scale
$d^{-1}$; the final $d^{-1/2}$ rate is set by the trace-reduction error
\eqref{eq:ssa_trace_comparison}, not by the concentration of
$\widetilde\Theta$ itself.  Finally,
\begin{equation}
    \Theta-\E\Theta
      =(\Theta-\widetilde\Theta)
       +(\widetilde\Theta-\E\widetilde\Theta)
       +\E(\widetilde\Theta-\Theta).
\end{equation}
Equations \eqref{eq:ssa_trace_comparison} and
\eqref{eq:ssa_tilde_concentration} imply
\begin{equation}
    \norm{\Theta-\E\Theta}_{L^p}
      \le C_{p,\varepsilon}d^{-1/2+\varepsilon}.
\end{equation}
Markov's inequality with arbitrarily large fixed $p$ proves
\begin{equation}
    \frac1{d^2}\tr(H^{-1})-
    \E\left[\frac1{d^2}\tr(H^{-1})\right]
      =O_\prec(d^{-1/2}),
\end{equation}
which establishes \ref{post:trace-concentration}.  Repeating the
Poincare estimate \eqref{eq:ssa_poincare} for a fixed finite-removal objective and using its
stability gives the corresponding estimate for every fixed finite LOO Hessian.

\section{Bounds on the norms of $\Wh$}
\label{app:frobenius_bounds}

This appendix establishes two-sided control on $\|\Wh\|_\mathsf{F}/d$ for the joint TAM minimizer of \ref{eq:tam-learning-joint} as well as \ref{post:spectral-op}. For the former, the upper bound is deterministic and uniform in the data; the lower bound holds with overwhelming probability ($1 - d^{-D}$ for every $D > 0$) in the regime $n/d^2 \to \alpha \in (0, \infty)$. Together, they pin
\begin{equation}\label{eq:nud_bracket}
    c_0 - O(d^{-D}) \;\le\; \nu_d \;\le\; C_0,
\end{equation}
for deterministic constants $0 < c_0 \le C_0 < \infty$ depending only on the parameters $(\alpha, r, \beta, \lambda)$, where $\nu_d^2 = \E[\|\Wh\|_\mathsf{F}^2/d^2]$ as in \eqref{eq:nu_d_def}.

The two Frobenius bounds provide the a priori scale for $\|\Wh\|_\mathsf{F}$ needed in \sref{tam_selfconsistent} and \corref{bilinear_gauss_coupling_rand}. The corresponding Frobenius concentration estimate is treated there as a consequence of the empirical concentration postulate \ref{post:empirical-concentration}, applied to the observables in the energy identity. The spectral estimate \ref{post:spectral-op} is proved at the end of this appendix.

Throughout the appendix, $C, c$ denote positive constants that may depend on $(\alpha, r, \beta, \lambda)$ but not on $d$, and may vary line to line.

\subsection{Deterministic Frobenius upper bound}
\label{app:frobenius_upper}

\begin{lemma}[Deterministic upper bound]\label{lem:frobenius_upper}
For any $\lambda > 0$, $\beta > 0$, $r \in (0, 1)$, and any data matrices $U, V \in \R^{d \times n}$,
\begin{equation}\label{eq:frob_upper}
    \frac{\|\Wh\|_\mathsf{F}^2}{d^2} \;\le\; \frac{2}{\lambda}\Big[\,\log 2 \;+\; \tfrac{\log 2}{r\beta}\,\Big] \;=:\; C_0.
\end{equation}
In particular, $\nu_d^2 \le C_0$.
\end{lemma}

\begin{proof}
Evaluate $\Psi$ at the test point $W = 0$, $\mu = 0$. All scores vanish, so
\begin{equation}\label{eq:psi_at_zero}
    \Psi(0, 0) \;=\; n\cdot\ell\Big(-\tfrac{1}{k}\,(n-1)\,\phi_\beta(0)\Big) \;=\; n\cdot\ell\Big(-\tfrac{\log 2}{r\beta}\Big),
\end{equation}
using $\phi_\beta(0) = (\log 2)/\beta$ and $r = k/(n-1)$. The function $\ell$ is convex with $|\ell'| \le 1$, so
\begin{equation}\label{eq:ell_bound}
    \ell(t) \;\le\; \ell(0) + |t| \;=\; \log 2 + |t|.
\end{equation}
Applying \eqref{eq:ell_bound} at $t = -(\log 2)/(r\beta)$ gives
\begin{equation}\label{eq:psi_at_zero_bound}
    \Psi(0, 0) \;\le\; n\Big[\,\log 2 + \tfrac{\log 2}{r\beta}\,\Big].
\end{equation}
By optimality of $(\Wh, \muh)$ and nonnegativity of $\ell$,
\begin{equation}\label{eq:opt_to_norm}
    \frac{\lambda n}{2 d^2}\,\|\Wh\|_\mathsf{F}^2 \;\le\; \Psi(\Wh, \muh) \;\le\; \Psi(0, 0).
\end{equation}
Combining \eqref{eq:psi_at_zero_bound} and \eqref{eq:opt_to_norm} and dividing by $\lambda n/(2d^2)$ yields \eqref{eq:frob_upper}. The bound on $\nu_d^2$ follows from $\nu_d^2 = \E\|\Wh\|_\mathsf{F}^2/d^2$. \qed
\end{proof}

\begin{remark}[Dependence on $\beta$]
The dependence on $\beta$ enters as $1/\beta$, which vanishes in the large-$\beta$ limit, leaving the cleaner asymptotic bound $C_0 \to 2\log 2/\lambda$.
\end{remark}

\subsection{High-probability Frobenius lower bound}
\label{app:frobenius_lower}

\begin{lemma}[High probability lower bound]\label{lem:frobenius_lower}
Fix $r, \alpha\in (0, 1) \times (0, \infty)$ and $\beta, \lambda > 0$, and consider the regime $n/d^2 \to \alpha$. There exists a constant $c_0 > 0$ such that, for every $D > 0$, there is $d_0 = d_0(D)$ for which
\begin{equation}\label{eq:frob_lower}
    \P\!\left(\frac{\|\Wh\|_\mathsf{F}}{d} \;\ge\; c_0\right) \;\ge\; 1 - d^{-D}, \qquad d \ge d_0.
\end{equation}
In particular, $\nu_d^2 \ge c_0^2 - O(d^{-D})$ for every $D > 0$.
\end{lemma}

\begin{remark}[Strategy]
The argument is by objective gap. We exhibit a one-dimensional family $W_t = t\,W^\times$ along the direction
\begin{equation}\label{eq:Wcross_def}
    W^\times \;:=\; U V^\top
\end{equation}
along which $\Psi(W_t, \mu_0\mathbf{1})$ decreases by $\Theta(n) = \Theta(d^2)$ at small $t > 0$, and complement this with the convex linearization $\Psi(W, \mu_0\mathbf{1}) \ge \Psi(0, \mu_0\mathbf{1}) + \langle\nabla_W\Psi(0, \mu_0\mathbf{1}), W\rangle$, whose linear term is bounded by $\|\nabla_W\Psi(0, \mu_0\mathbf{1})\|_\mathsf{F}\|W\|_\mathsf{F} = O(d\,\|W\|_\mathsf{F})$. Comparing the two excludes a Frobenius ball of radius $\eta d$ for small $\eta$.
\end{remark}

\begin{proof}
The variable $\mu$ is unregularized; we choose a fixed $\mu_0 \in \R$ and track the objective along the slice $\mu = \mu_0\mathbf{1}_n$ throughout. Let
\begin{equation}\label{eq:mu0_def}
    \mu_0 \;:=\; -\tfrac{1}{\beta}\log\tfrac{r}{1-r}, \qquad q_0 := \sigma(-\beta\mu_0) = r,
\end{equation}
so that, with $f(\mu) := -\mu - \tfrac{1}{r}\phi_\beta(-\mu)$, we have $f'(\mu_0) = -1 + \tfrac{1}{r}q_0 = 0$. Define $a_0 := \sigma(-f(\mu_0)) \in (0, 1)$, the activation associated with the test profile $(0, \mu_0\mathbf{1}_n)$.

\subparagraph*{Step 1: Gradient at the test point.}
By the gradient formula \eqref{eq:grad_W_expanded} evaluated at $W = 0$, $\mu = \mu_0\mathbf{1}_n$, all $a_i = a_0$ and $q_{j,i} = q_0 = r$, hence
\begin{equation}\label{eq:grad_at_test}
    \nabla_W \Psi(0, \mu_0\mathbf{1}_n) \;=\; -a_0\,U\,M\,V^\top,
    \qquad
    M_{ii} = 1, \quad M_{ij} = -\tfrac{1}{n-1}\ \text{for }j \ne i.
\end{equation}
The matrix $M$ is symmetric with eigenvalues $0$ (on $\mathbf{1}_n$) and $n/(n-1)$ (on $\mathbf{1}_n^\perp$, with multiplicity $n-1$); in particular, $\|M\|_\mathsf{F}^2 = n^2/(n-1)$ and $\|M\|_\mathrm{op} = n/(n-1)$.

By independence of $U$ and $V$ together with $\E[U^\top U] = \E[V^\top V] = I_n$ (since each column has variance $I_d/d$),
\begin{equation}\label{eq:E_grad_F2}
    \E\,\|U M V^\top\|_\mathsf{F}^2 \;=\; \tr\!\bigl(M\,\E[U^\top U]\,M\,\E[V^\top V]\bigr) \;=\; \tr(M^2) \;=\; \tfrac{n^2}{n-1}.
\end{equation}
Hanson--Wright applied to the Gaussian quadratic form $\|UMV^\top\|_\mathsf{F}^2$, viewed as a degree-$4$ polynomial in the i.i.d. $\mathcal{N}(0, 1)$ entries of $\sqrt d\,U$ and $\sqrt d\,V$, has variance $O(n)$ and operator-norm scale $O(1)$, yielding sub-exponential concentration: for any $\delta > 0$,
\begin{equation}\label{eq:grad_concentration}
    \P\!\bigl(\|U M V^\top\|_\mathsf{F}^2 > (1+\delta)\,\tr(M^2)\bigr) \;\le\; \exp(-c\,d^2),
\end{equation}
for a constant $c > 0$ depending only on $\delta$. Hence there exists a constant $K = K(\alpha) > 0$ such that, for every $D > 0$, the event
\begin{equation}\label{eq:grad_bound}
    \Omega_d^{(1)} \;:=\; \bigl\{\,\|\nabla_W \Psi(0, \mu_0\mathbf{1}_n)\|_\mathsf{F} \;\le\; a_0\,K\,d\,\bigr\}
\end{equation}
satisfies $\P(\Omega_d^{(1)}) \ge 1 - d^{-D}$ for all sufficiently large $d$.

\subparagraph*{Step 2: Directional derivative along $W^\times$.}
Define $g(t) := \Psi(tW^\times, \mu_0\mathbf{1}_n)$, viewed as a one-dimensional convex function of $t \ge 0$. By \eqref{eq:grad_at_test} and the symmetry of $M$,
\begin{equation}\label{eq:gprime_zero}
    g'(0) \;=\; \langle \nabla_W \Psi(0, \mu_0\mathbf{1}_n),\, W^\times\rangle \;=\; -a_0\,\tr\!\bigl(M\,V^\top V\,U^\top U\bigr).
\end{equation}
The quantity $\tr(M V^\top V U^\top U)$ is a degree-$4$ polynomial in the i.i.d. $\mathcal{N}(0, 1)$ entries of $(\sqrt d\,U, \sqrt d\,V)$, with mean $\tr(M) = n$ (by independence of $U, V$ and $\E[U^\top U] = \E[V^\top V] = I_n$) and variance $O(n)$. Sub-exponential concentration via Hanson--Wright (or, equivalently, Gaussian hypercontractivity for finite-degree chaos) yields
\begin{equation}\label{eq:gprime_concentration}
    \Omega_d^{(2)} \;:=\; \bigl\{\,g'(0) \;\le\; -\tfrac{1}{2}\,a_0\,n\,\bigr\}, \qquad \P(\Omega_d^{(2)}) \;\ge\; 1 - d^{-D},
\end{equation}
for every $D > 0$ and all sufficiently large $d$.

\subparagraph*{Step 3: Curvature along the path.}
Computing the curvature $g''(t)$ for the smooth function $\Psi(\,\cdot\,, \mu_0\mathbf{1}_n)$ and bounding $|\ell''| \le 1/4$, $|\phi_\beta''| \le \beta/4$:
\begin{equation}\label{eq:gpp_bound}
    g''(t) \;\le\; \tfrac{1}{4}\sum_{i=1}^n (t_i'(t))^2 \;+\; \tfrac{\beta}{4 k}\sum_{i=1}^n\sum_{j\ne i}(S^\times_{ji})^2 \;+\; \tfrac{\lambda n}{d^2}\,\|W^\times\|_\mathsf{F}^2,
\end{equation}
where $t_i(t)$ is the loss argument at sample $i$ along the path. Using $|t_i'(t)|^2 \le 2(S^\times_{ii})^2 + (2/(rk))\sum_{j\ne i}(S^\times_{ji})^2$, the first two terms collapse into
\begin{equation}\label{eq:gpp_collapse}
    \tfrac{1}{2}\sum_i (S^\times_{ii})^2 \;+\; \Bigl(\tfrac{1}{2 r k} + \tfrac{\beta}{4 k}\Bigr)\sum_{i,j\ne i}(S^\times_{ji})^2.
\end{equation}
The summed second moments of $S^\times$ admit the matching expectations
\begin{equation}\label{eq:S_moments}
    \E\!\sum_i (S^\times_{ii})^2 \;\asymp\; n,
    \qquad
    \E\!\sum_{i, j\ne i}(S^\times_{ji})^2 \;\asymp\; n^2,
\end{equation}
both inherited from the Wishart structure of $U^\top U$ and $V^\top V$. The ridge contribution $\|W^\times\|_\mathsf{F}^2 = \|UV^\top\|_\mathsf{F}^2$ has mean $n$ (by the same calculation as \eqref{eq:E_grad_F2} with $M = I$), so that $(\lambda n/d^2)\|W^\times\|_\mathsf{F}^2 \asymp \lambda \alpha n$. Each of the three sums is a polynomial of bounded degree (at most $8$) in the i.i.d. $\mathcal{N}(0,1)$ entries of $\sqrt d\,U$ and $\sqrt d\,V$ with mean of order $n^a$ for $a \in \{1, 2, 1\}$, respectively, and variance of strictly smaller order. Gaussian hypercontractivity for finite-degree chaos (or, equivalently, Latala's moment inequalities) yields stretched-exponential concentration $\P(|X - \E X| > \delta\,\E X) \le \exp(-c\,d^{2/q})$ for any constant $\delta > 0$ and degree-$q$ chaos. Combining the three concentration statements via a union bound, there is a constant $L = L(\alpha, r, \beta, \lambda) > 0$ such that, for every $D > 0$, the event
\begin{equation}\label{eq:gpp_uniform}
    \Omega_d^{(3)} \;:=\; \Bigl\{\,\sup_{t \in [0, 1]} g''(t) \;\le\; L\,n\,\Bigr\}
\end{equation}
satisfies $\P(\Omega_d^{(3)}) \ge 1 - d^{-D}$ for all sufficiently large $d$. (The supremum over $t$ is harmless: the data-dependent factors in \eqref{eq:gpp_collapse} are independent of $t$, while the loss curvatures $|\ell''|, |\phi_\beta''|$ admit deterministic bounds $1/4, \beta/4$ that hold uniformly in $t$.)

\subparagraph*{Step 4: Objective gap.}
Set $t_\star := a_0/(2L) \wedge 1$. By Taylor's theorem with remainder,
\begin{equation}\label{eq:Taylor_g}
    g(t_\star) \;\le\; g(0) \;+\; g'(0)\,t_\star \;+\; \tfrac{1}{2}\sup_{[0, t_\star]} g''\cdot t_\star^2.
\end{equation}
On $\Omega_d^{(2)} \cap \Omega_d^{(3)}$, the bounds \eqref{eq:gprime_concentration} and \eqref{eq:gpp_uniform} reduce \eqref{eq:Taylor_g} to
\begin{equation}\label{eq:gap}
    g(t_\star) \;\le\; g(0) \;-\; \tfrac{1}{2}\,a_0\,n\,t_\star \;+\; \tfrac{1}{2}\,L\,n\,t_\star^2 \;\le\; g(0) \;-\; \tfrac{1}{4}\,a_0\,n\,t_\star \;=:\; g(0) - c_1 n,
\end{equation}
with $c_1 := a_0\,t_\star/4 > 0$.

\subparagraph*{Step 5: Convex linearization and contradiction.}
By optimality of $(\Wh, \muh)$ in $\Psi$ and convexity of $\Psi(\,\cdot\,, \mu_0\mathbf{1}_n)$ in $W$,
\begin{align}
    g(0) - c_1 n
    \;&\ge\; g(t_\star) \;\ge\; \Psi(\Wh, \muh) \;\ge\; \Psi(0, \mu_0\mathbf{1}_n) \;+\; \langle \nabla_W \Psi(0, \mu_0\mathbf{1}_n),\, \Wh\rangle \notag \\
    \;&\ge\; g(0) \;-\; \|\nabla_W \Psi(0, \mu_0\mathbf{1}_n)\|_\mathsf{F}\cdot\|\Wh\|_\mathsf{F},
    \label{eq:lower_chain}
\end{align}
where the third inequality follows from the convex subgradient inequality around $(0,\mu_0\mathbf{1}_n)$, and the optimality $\nabla_\mu \Psi(0, \mu_0\mathbf{1}_n) = 0$ at $\mu_0$ (which follows from $f'(\mu_0) = 0$) eliminates the $\mu$-component. Rearranging, on $\Omega_d := \bigcap_{j=1}^3 \Omega_d^{(j)}$,
\begin{equation}\label{eq:Wh_F_lower}
    \|\Wh\|_\mathsf{F} \;\ge\; \frac{c_1 n}{\|\nabla_W \Psi(0, \mu_0\mathbf{1}_n)\|_\mathsf{F}} \;\ge\; \frac{c_1 n}{a_0 K d} \;=:\; c_0\,d,
\end{equation}
using \eqref{eq:grad_bound} and $n \asymp \alpha d^2$, with $c_0 := c_1\alpha/(a_0 K) > 0$. This is \eqref{eq:frob_lower}.

A union bound over $\Omega_d^{(1)}, \Omega_d^{(2)}, \Omega_d^{(3)}$ gives $\P(\Omega_d^c) \le 3 d^{-D}$, which is at most $d^{-(D-1)}$ for $d \ge 3$; since $D$ is arbitrary, the rate $d^{-D}$ is preserved.

For the conclusion on $\nu_d^2$: combining \lref{frobenius_upper} with the bound just established,
\begin{equation}\label{eq:nud2_lower}
    \|\Wh\|_\mathsf{F}^2/d^2 \;\ge\; c_0^2\,\mathbf{1}_{\Omega_d}, \qquad \|\Wh\|_\mathsf{F}^2/d^2 \;\le\; C_0 \text{ everywhere},
\end{equation}
so $\nu_d^2 \ge c_0^2\,\P(\Omega_d) \ge c_0^2(1 - d^{-D})$ for every $D > 0$.
\end{proof}

Together, \lref{frobenius_upper} and \lref{frobenius_lower} certify that $\nu_d$ is bounded away from $0$ and $\infty$ for all sufficiently large $d$. This is the \emph{a priori} input that subsequent analyses---in particular the bilinear coupling \corref{bilinear_gauss_coupling_rand}---use implicitly. The Frobenius-norm concentration statement $\|\Wh\|_\mathsf{F}^2/d^2 - \nu_d^2 = \Op(d^{-1/2})$ is not a separate postulate; it is folded into (A5) through the energy identity in \sref{tam_selfconsistent}.

\subsection{Spectral norm bound}
\label{proof:A4}

The Frobenius bound of Lemma \ref{lem:frobenius_upper} alone yields only
$\norm{\Wh}_{\mathrm{op}}=O_\prec(d)$.  To recover the square-root
scale, we profile the optimization problem along every pair of test
directions, thereby separating the adaptive score-gradient matrix
from the Gaussian coordinates being tested.  Recall that the ridge
coefficient in the profiled objective is
$\gamma_d=\lambda n/d^2\asymp1$.

Fix deterministic $x,y\in S^{d-1}$ and define the orthogonal
subspaces
\begin{equation}\label{eq:ssa_C_L}
    \mathcal C_{x,y}:=
       \{C:x^\top C=0,\ Cy=0\},
    \qquad
    \mathcal L_{x,y}:=\mathcal C_{x,y}^\perp.
\end{equation}
The orthogonal projection onto $\mathcal L_{x,y}$ is
\begin{equation}\label{eq:ssa_projection}
    \Pi_{x,y}M=xx^\top M+Myy^\top-xx^\top Myy^\top.
\end{equation}
In particular, $xy^\top\in\mathcal L_{x,y}$.

Let $C_{x,y}$ minimize $\Phi$ over $\mathcal C_{x,y}$ and consider the
profile
\begin{equation}
    F_{x,y}(L):=\min_{C\in\mathcal C_{x,y}}\Phi(C+L),
    \qquad L\in\mathcal L_{x,y}.
\end{equation}
For $C\perp L$, the ridge term splits as
\[
    \frac{\gamma_d}{2}\norm{C+L}_{\mathsf F}^2
    =\frac{\gamma_d}{2}\norm C_{\mathsf F}^2
     +\frac{\gamma_d}{2}\norm L_{\mathsf F}^2.
\]
Thus partial minimization preserves convexity while retaining the
$\gamma_d$-quadratic term in $L$, so $F_{x,y}$ is
$\gamma_d$-strongly convex.  If
$\Wh=C^\star+L^\star$ is the corresponding orthogonal decomposition,
then global optimality of $\Wh$ gives
$L^\star=\Pi_{x,y}\Wh=\argmin F_{x,y}$.  The envelope theorem gives
\begin{equation}
    \nabla F_{x,y}(0)
      =\Pi_{x,y}\nabla\Phi(C_{x,y})
      =\Pi_{x,y}\bigl[UG(C_{x,y})V^\top\bigr],
\end{equation}
since $\Pi_{x,y}C_{x,y}=0$.  Moreover
$\nabla F_{x,y}(L^\star)=0$, and strong monotonicity yields
\[
    \gamma_d\norm{L^\star}_{\mathsf F}^2
    \le \left\langle \nabla F_{x,y}(0),-L^\star\right\rangle
    \le \norm{\nabla F_{x,y}(0)}_{\mathsf F}\norm{L^\star}_{\mathsf F}.
\]
Hence $\norm{L^\star}_{\mathsf F}\le
\gamma_d^{-1}\norm{\nabla F_{x,y}(0)}_{\mathsf F}$.  Since
$xy^\top\in\mathcal L_{x,y}$ and $\norm{xy^\top}_{\mathsf F}=1$,
\begin{equation}\label{eq:ssa_profile_ineq}
    |x^\top\Wh y|
    =\bigl|\langle \Pi_{x,y}\Wh,xy^\top\rangle\bigr|
    \le \norm{\Pi_{x,y}\Wh}_{\mathsf F}
    \le \gamma_d^{-1}
      \norm{\Pi_{x,y}[UG(C_{x,y})V^\top]}_{\mathsf F}.
\end{equation}

It remains to bound the right-hand side.  Decompose
\begin{equation}\label{eq:ssa_residual_decomp}
    u_j=p_jx+\bar u_j,
    \qquad
    v_i=r_i y+\bar v_i,
\end{equation}
where $p,r\sim\mathcal N(0,I_n/d)$ are mutually independent and
independent of $(\bar U,\bar V)$.  If $C\in\mathcal C_{x,y}$, then
\[
    u_j^\top Cv_i=\bar u_j^\top C\bar v_i.
\]
Thus every score in the objective restricted to $\mathcal C_{x,y}$
depends only on $(\bar U,\bar V)$.  By uniqueness of the restricted
minimizer, both $C_{x,y}$ and $G_{x,y}:=G(C_{x,y})$ are measurable with
respect to $(\bar U,\bar V)$ and hence independent of the fresh
Gaussian variables $(p,r)$.  This is the key decoupling in the
argument.  By \eqref{eq:ssa_G_entries},
$\norm{G_{x,y}}_{\mathrm{op}}\le C$ deterministically.

Writing $M=UG_{x,y}V^\top$, \eqref{eq:ssa_projection} gives
\begin{equation}\label{eq:ssa_projected_three}
    \norm{\Pi_{x,y}M}_{\mathsf F}
    \le 3|p^\top G_{x,y}r|
      +\norm{p^\top G_{x,y}\bar V^\top}_2
      +\norm{\bar U G_{x,y}r}_2.
\end{equation}
Let $\mathcal E_d$ denote the standard event
\begin{equation}\label{eq:ssa_gaussian_event}
    \norm{U}_{\mathrm{op}}+\norm{V}_{\mathrm{op}}\le C\sqrt d,
    \qquad
    \norm{U}_{\mathsf F}+\norm{V}_{\mathsf F}\le C\sqrt n,
\end{equation}
which fails with probability at most $e^{-cd}$.  Since we condition on
$(\bar U,\bar V)$ below, note that $\mathcal E_d$ implies the
residual bounds
$\norm{\bar U}_{\mathrm{op}}+\norm{\bar V}_{\mathrm{op}}\le C\sqrt d$
and
$\norm{\bar U}_{\mathsf F}+\norm{\bar V}_{\mathsf F}\le C\sqrt n$,
as well as $\norm{p}_2+\norm{r}_2\le C\sqrt d$.

Condition on $(\bar U,\bar V)$, so that $G_{x,y}$ is fixed.  For
$A:=G_{x,y}\bar V^\top$, the preceding residual bounds give
$\norm{A}_{\mathsf F}\le C\sqrt n$ and
$\norm{A}_{\mathrm{op}}\le C\sqrt d$.  Writing
$p=d^{-1/2}g$ with $g\sim\mathcal N(0,I_n)$,
\begin{equation}\label{eq:ssa_vector_moments}
    \E\!\left[\norm{p^\top A}_2^2\mid\bar U,\bar V\right]
       =\frac{1}{d}\norm{A}_{\mathsf F}^2=O(d),
    \qquad
    \mathrm{Lip}\!\left(g\mapsto d^{-1/2}\norm{g^\top A}_2\right)
       \le \frac{\norm{A}_{\mathrm{op}}}{\sqrt d}=O(1).
\end{equation}
Thus $\norm{p^\top G_{x,y}\bar V^\top}_2$ has conditional mean
$O(\sqrt d)$ and Gaussian deviations on an $O(1)$ scale.  The same
argument with $B:=\bar U G_{x,y}$ controls
$\norm{\bar U G_{x,y}r}_2$.  Finally, conditional on
$(\bar U,\bar V,r)$,
\[
    p^\top G_{x,y}r\sim
    \mathcal N\!\left(0,\frac{1}{d}\norm{G_{x,y}r}_2^2\right),
    \qquad
    \frac{1}{d}\norm{G_{x,y}r}_2^2\le
       \frac{C}{d}\norm{r}_2^2=O(1)
\]
on $\mathcal E_d$.  Applying Gaussian concentration to the three
terms in \eqref{eq:ssa_projected_three} therefore gives, uniformly over
deterministic $x,y$,
\begin{equation}\label{eq:ssa_fixed_xy_tail}
    \P\left(
      \mathcal E_d\cap
      \left\{\norm{\Pi_{x,y}M}_{\mathsf F}>K\sqrt d\right\}
    \right)
    \le Ce^{-cK^2d}
\end{equation}
for all sufficiently large fixed $K$.  Here the use of
$\mathcal E_d$ causes no issue under the conditioning: we only invoke
its residual consequences above, which are measurable with respect to
$(\bar U,\bar V)$, together with the bound on $\norm{r}_2$ for the
scalar term.

Let $\mathcal N$ be a deterministic $1/4$-net of $S^{d-1}$ with
$|\mathcal N|\le9^d$.  Taking a union bound in
\eqref{eq:ssa_fixed_xy_tail},
\begin{align*}
&\P\left(
    \max_{x,y\in\mathcal N}
    \norm{\Pi_{x,y}[UG(C_{x,y})V^\top]}_{\mathsf F}>K\sqrt d
  \right)\\
&\qquad\le e^{-c_0d}+81^d C e^{-cK^2d}.
\end{align*}
Choosing $K$ sufficiently large makes the right-hand side
exponentially small in $d$, and hence
\begin{equation}
    \max_{x,y\in\mathcal N}
    \norm{\Pi_{x,y}[UG(C_{x,y})V^\top]}_{\mathsf F}
    =O_\prec(\sqrt d).
\end{equation}
Combining this with \eqref{eq:ssa_profile_ineq},
$\gamma_d^{-1}=O(1)$, and the standard net inequality
$\norm{W}_{\mathrm{op}}\le2\max_{x,y\in\mathcal N}|x^\top Wy|$
proves
\begin{equation}\label{eq:ssa_A4_conclusion}
    \norm{\Wh}_{\mathrm{op}}=O_\prec(\sqrt d).
\end{equation}
This establishes \ref{post:spectral-op}.  The same proof applies after
removing any fixed number of samples: the Gaussian residual
decomposition and deterministic bound on $G$ are unchanged, while
$n-O(1)\asymp d^2$.

\section{A coupling lemma for bilinear Gaussian forms}
\label{app:bilinear_coupling}

This appendix establishes a coupling estimate for the joint law of bilinear forms $u_j^\top W v_0$ that share a common right-side Gaussian vector $v_0$. The estimate is the technical input behind the joint Gaussian limit for competitor scores derived in \sref{tam_selfconsistent}, where it is applied with $W = \WLOO[i, j_1, \ldots, j_m]$ and $v_0 = v_i$.

We first prove the estimate for a deterministic matrix $W$, then transfer it to the random case via concentration of $\|W\|_\mathsf{F}^2/d^2$.

\subsection{Deterministic matrix}
\label{app:bilinear_coupling_det}

\begin{lemma}[Bilinear Gaussian coupling, deterministic matrix]\label{lem:bilinear_gauss_coupling_det}
Let $W \in \R^{d \times d}$ be a deterministic matrix with $V_0 := \|W\|_\mathsf{F}^2/d^2 > 0$, and let $u_1, \ldots, u_m, v_0 \in \R^d$ be independent Gaussian vectors with $u_k, v_0 \sim \mathcal{N}(0, I_d/d)$. Set
\begin{equation}\label{eq:Xj_def}
    X_j \;:=\; u_j^\top W v_0, \qquad j = 1, \ldots, m.
\end{equation}
On a common probability space, there exist i.i.d. random variables $Y_1, \ldots, Y_m \sim \mathcal{N}(0,V_0)$ such that
\begin{equation}\label{eq:bilinear_coupling}
    \bigl\|(X_1, \ldots, X_m) - (Y_1, \ldots, Y_m)\bigr\|_2 \;=\; \Op\!\bigl(\|W^\top W\|_\mathsf{F}/(d^2\sqrt{V_0})\bigr).
\end{equation}
\end{lemma}

\begin{remark}[Interpretation]
The deviation $\|W^\top W\|_\mathsf{F}/d^2$ measures the spread of the eigenvalues of $W^\top W/d^2$: when the eigenvalues are concentrated, $\|W^\top W\|_\mathsf{F}/d^2$ is small relative to $\|W\|_\mathsf{F}^2/d^2 = \tr(W^\top W)/d^2$. In the proof below it is exactly the Hanson--Wright scale for the random conditional variance $\|Wv_0\|^2/d$.
\end{remark}

\begin{proof}
The argument exploits the conditional Gaussian structure of $\{X_j\}$ given $v_0$.

\subparagraph*{Step 1: Conditional Gaussian.}
Set $w := W v_0 \in \R^d$. Conditional on $v_0$, $w$ is deterministic, and $X_j = u_j^\top w$ is a Gaussian linear functional of $u_j$. Since $u_1, \ldots, u_m \sim \mathcal{N}(0, I_d/d)$ are independent across $j$,
\begin{equation}\label{eq:cond_gaussian_lemma}
    (X_1, \ldots, X_m)\,\big|\,v_0 \;\sim\; \mathcal{N}\!\bigl(0,\, V\cdot I_m\bigr),
    \qquad
    V \;:=\; \frac{\|W v_0\|^2}{d}.
\end{equation}

\subparagraph*{Step 2: Hanson--Wright on the conditional variance.}
We have
\begin{equation}\label{eq:EV_lemma}
    \E[V] \;=\; \frac{1}{d}\,\E[v_0^\top W^\top W v_0] \;=\; \frac{1}{d^2}\,\tr(W^\top W) \;=\; \frac{\|W\|_\mathsf{F}^2}{d^2} \;=:\; V_0,
\end{equation}
using $\E[v_0 v_0^\top] = I_d/d$. The Hanson--Wright inequality applied to the centered quadratic form $V - V_0 = (1/d)(v_0^\top W^\top W v_0 - \|W\|_\mathsf{F}^2/d)$ in the Gaussian vector $v_0$ gives
\begin{equation}\label{eq:HW_lemma}
    V - V_0 \;=\; \Op\!\bigl(\|W^\top W\|_\mathsf{F}/d^2\bigr),
    \qquad
    \var V \;=\; \frac{2}{d^4}\,\|W^\top W\|_\mathsf{F}^2.
\end{equation}

\subparagraph*{Step 3: Coupling.}
Define $Y_j := X_j \cdot \sqrt{V_0/V}$ for each $j$. Conditional on $v_0$, $\{Y_j\}$ inherit the Gaussianity and independence of $\{X_j\}$, with conditional variance $V \cdot (V_0/V) = V_0$. Since the conditional law $\mathcal{N}(0, V_0\cdot I_m)$ is independent of $v_0$, the unconditional joint law of $(Y_1, \ldots, Y_m)$ is i.i.d. $\mathcal{N}(0, V_0)$ as required.

The deviation factors as
\begin{equation}\label{eq:Xj_minus_Yj}
    X_j - Y_j \;=\; X_j\bigl(1 - \sqrt{V_0/V}\bigr) \;=\; X_j \cdot \frac{\sqrt V - \sqrt{V_0}}{\sqrt V}.
\end{equation}
Two facts: conditional on $v_0$, the standardized form $X_j/\sqrt V \sim \mathcal{N}(0, 1)$, so $|X_j/\sqrt V| = \Op(1)$ uniformly in $j$; and $|\sqrt V - \sqrt{V_0}| \le |V - V_0|/\sqrt{V_0}$, by $\sqrt V + \sqrt{V_0} \ge \sqrt{V_0}$. Combined with \eqref{eq:HW_lemma},
\begin{equation}\label{eq:Xj_Yj_bound_with_V0}
    |X_j - Y_j| \;=\; |X_j/\sqrt V| \cdot |\sqrt V - \sqrt{V_0}| \;=\; \Op\!\bigl(\|W^\top W\|_\mathsf{F}/(d^2\sqrt{V_0})\bigr),
\end{equation}
uniformly in $j$. For fixed $m$, we have
\begin{equation}
\|(X) - (Y)\|_2^2 = \sum_{j = 1}^{m}|X_j - Y_j|^2 = \Op\!\bigl((\|W^\top W\|_\mathsf{F}/(d^2\sqrt{V_0}))^2\bigr),
\end{equation}
which gives \eqref{eq:bilinear_coupling}.
\end{proof}

\subsection{Random matrix with concentrated Frobenius norm}
\label{app:bilinear_coupling_rand}

\begin{corollary}[Bilinear Gaussian coupling, random matrix]\label{cor:bilinear_gauss_coupling_rand}
Let $W$ be a random $d \times d$ matrix with $\|W\|_\mathsf{F} > 0$, independent of $u_1, \ldots, u_m, v_0$ which are i.i.d. $\mathcal{N}(0, I_d/d)$. Let $\nu_d^2 > 0$ be a deterministic constant. Suppose $\rho_f, \rho_\mathrm{op} > 0$ are deterministic rates such that
\begin{equation}\label{eq:cor_hyp}
    \frac{\|W\|_\mathsf{F}^2}{d^2} - \nu_d^2 \;=\; \Op(\rho_f),
    \qquad
    \frac{\|W^\top W\|_\mathsf{F}}{d^2} \;=\; \Op(\rho_\mathrm{op}).
\end{equation}
Set $X_j := u_j^\top W v_0$ for $j = 1, \ldots, m$. On a common probability space, there exist i.i.d. random variables $Y_1, \ldots, Y_m \sim \mathcal{N}(0,\nu_d^2)$ such that
\begin{equation}\label{eq:cor_coupling}
    \bigl\|(X_1, \ldots, X_m) - (Y_1, \ldots, Y_m)\bigr\|_2 \;=\; \Op\!\bigl((\rho_\mathrm{op} + \rho_f)/\nu_d\bigr).
\end{equation}
\end{corollary}

\begin{proof}
The key observation is that, conditional on $(W, v_0)$, $X_j/\sqrt V \overset{\text{iid}}{\sim} \mathcal{N}(0, 1)$ across $j$, where $V := \|W v_0\|^2/d$ is the conditional variance.

Define
\begin{equation}\label{eq:cor_Yj_def}
    Y_j \;:=\; X_j\,\sqrt{\nu_d^2/V}, \qquad j = 1, \ldots, m.
\end{equation}
Conditional on $(W, v_0)$, $\{Y_j\}$ inherit the Gaussianity and independence of $\{X_j\}$ with variance $V \cdot \nu_d^2/V = \nu_d^2$, so the unconditional joint law of $(Y_1, \ldots, Y_m)$ is i.i.d. $\mathcal{N}(0, \nu_d^2)$ as required. The deviation factors as $X_j - Y_j = X_j(1 - \sqrt{\nu_d^2/V}) = X_j\,(\sqrt V - \sqrt{\nu_d^2})/\sqrt V$, and using $|X_j/\sqrt V| = \Op(1)$ together with $|\sqrt V - \sqrt{\nu_d^2}| \le |V - \nu_d^2|/\sqrt{\nu_d^2}$,
\begin{equation}\label{eq:cor_a_pertermbound}
    |X_j - Y_j| \;=\; |X_j/\sqrt V|\cdot|\sqrt V - \sqrt{\nu_d^2}| \;=\; \Op\!\bigl(|V - \nu_d^2|/\sqrt{\nu_d^2}\bigr).
\end{equation}
Decompose
\begin{equation}\label{eq:cor_a_decomp}
    V - \nu_d^2 \;=\; (V - V_0(W)) \;+\; (V_0(W) - \nu_d^2),
    \qquad V_0(W) := \|W\|_\mathsf{F}^2/d^2.
\end{equation}
Hanson--Wright applied to $V - V_0(W)$ conditional on $W$ gives
\begin{equation}
    V - V_0(W)
    =
    \Op\!\left(\|W^\top W\|_\mathsf{F}/d^2\right)
    =
    \Op(\rho_\mathrm{op}),
\end{equation}
using \eqref{eq:HW_lemma}. The second piece is $\Op(\rho_f)$ by \eqref{eq:cor_hyp}. Combining with \eqref{eq:cor_a_pertermbound} and summing over $j$ yields \eqref{eq:cor_coupling}.
\end{proof}

\section{A trace estimate for the inverse Hessian}
\label{app:chi_trace}

This appendix records several elementary estimates that will be used in the characterization of the susceptibility parameter $\chi$. The first says that a fresh rank-one Gaussian direction probes the normalized trace of a bounded linear operator on $\R^{d \times d}$. The next two show that the normalized inverse trace is stable under the perturbations needed to simplify the TAM Hessian.

We will use the following moment notation. For a sequence of random variables $X_d$ and deterministic positive numbers $a_d$, we write
\begin{equation}
    X_d = O_{\mathcal L}(a_d)
\end{equation}
if, for every fixed $p < \infty$, there is a constant $C_p < \infty$ such that $\|X_d\|_{L^p} \le C_p a_d$ for all sufficiently large $d$.

\begin{lemma}[Fresh rank-one trace estimate]\label{lem:fresh_rank_one_trace}
Let $u, v \in \R^d$ be independent Gaussian vectors with distribution $\mathcal N(0, I_d/d)$. Let $T : \R^{d \times d} \to \R^{d \times d}$ be a random symmetric linear operator, independent of $(u, v)$, and assume that
\begin{equation}
    \|T\|_{\mathrm{op}} = O_{\mathcal L}(1).
\end{equation}
Then
\begin{equation}\label{eq:fresh_rank_one_trace}
    \biggl\langle uv^\top,\, T[uv^\top]\biggr\rangle
    =
    \frac{\tr(T)}{d^2}
    +
    O_{\mathcal L}(d^{-1/2}).
\end{equation}
Equivalently, for every fixed $p < \infty$,
\begin{equation}\label{eq:fresh_rank_one_trace_Lp}
    \left\|
    \biggl\langle uv^\top,\, T[uv^\top]\biggr\rangle
    -
    \frac{\tr(T)}{d^2}
    \right\|_{L^p}
    \le
    \frac{C_p}{\sqrt d}.
\end{equation}
\end{lemma}

\begin{proof}
Write $u = g/\sqrt d$ and $v = h/\sqrt d$, where $g,h \in \R^d$ are independent standard Gaussian vectors. Define
\begin{equation}
    F(g,h,T)
    \;:=\;
    \biggl\langle uv^\top,\,T[uv^\top]\biggr\rangle
    =
    \frac{1}{d^2}
    \bigl\langle g \otimes h,\, T[g \otimes h]\bigr\rangle .
\end{equation}
Since $T$ is independent of $(g,h)$, we have
\begin{equation}\label{eq:F_cond_mean}
    \E\!\left[F(g,h,T) \mid T\right]
    =
    \frac{\tr(T)}{d^2}.
\end{equation}

We first bound the fluctuation in $L^2$. By the Gaussian Poincare inequality applied to the variables $(g,h)$, with $T$ kept fixed,
\begin{equation}\label{eq:poincare_F}
    \var\!\left(F(g,h,T) \mid T\right)
    \le
    \E\!\left[\|\nabla_g F\|^2 + \|\nabla_h F\|^2 \mid T\right].
\end{equation}
For fixed $h$, define the compressed operator
\begin{equation}
    T_h
    :=
    (I \otimes h)^\ast T(I \otimes h)
    \qquad\text{on }\R^d,
\end{equation}
where $(I \otimes h)g = g \otimes h$. Then
\begin{equation}
    F(g,h,T)
    =
    \frac{1}{d^2} g^\top T_h g,
    \qquad
    \nabla_g F
    =
    \frac{2}{d^2}T_h g,
\end{equation}
using the symmetry of $T$. It follows that
\begin{equation}\label{eq:grad_g_bound}
    \E\!\left[\|\nabla_g F\|^2 \mid T\right]
    =
    \frac{4}{d^4}\,
    \E\!\left[\tr(T_h^2) \mid T\right].
\end{equation}
Moreover,
\begin{equation}
    \|T_h\|_{\mathrm{op}}
    \le
    \|T\|_{\mathrm{op}}\,\|h\|^2,
\end{equation}
and hence
\begin{equation}
    \tr(T_h^2)
    \le
    d\,\|T_h\|_{\mathrm{op}}^2
    \le
    d\,\|T\|_{\mathrm{op}}^2\,\|h\|^4.
\end{equation}
Since $\E\|h\|^4 = d^2 + 2d$, \eqref{eq:grad_g_bound} gives
\begin{equation}\label{eq:grad_g_final}
    \E\!\left[\|\nabla_g F\|^2 \mid T\right]
    \le
    \frac{C\|T\|_{\mathrm{op}}^2}{d}.
\end{equation}
The same argument, with the roles of $g$ and $h$ interchanged, gives
\begin{equation}\label{eq:grad_h_final}
    \E\!\left[\|\nabla_h F\|^2 \mid T\right]
    \le
    \frac{C\|T\|_{\mathrm{op}}^2}{d}.
\end{equation}
Combining \eqref{eq:poincare_F}, \eqref{eq:grad_g_final}, and \eqref{eq:grad_h_final}, we obtain
\begin{equation}\label{eq:var_F_given_T}
    \var\!\left(F(g,h,T) \mid T\right)
    \le
    \frac{C\|T\|_{\mathrm{op}}^2}{d}.
\end{equation}

We now upgrade this $L^2$ bound to fixed moments. The centered random variable
\begin{equation}
    F(g,h,T) - \frac{\tr(T)}{d^2}
\end{equation}
is, for fixed $T$, a Gaussian polynomial of degree at most $4$ in $(g,h)$. By Gaussian hypercontractivity, for every fixed $p < \infty$,
\begin{equation}\label{eq:hypercontractive_F}
    \left\|
    F(g,h,T) - \frac{\tr(T)}{d^2}
    \right\|_{L^p(g,h \mid T)}
    \le
    C_p
    \left\|
    F(g,h,T) - \frac{\tr(T)}{d^2}
    \right\|_{L^2(g,h \mid T)}
    \le
    \frac{C_p\|T\|_{\mathrm{op}}}{\sqrt d},
\end{equation}
where the last step uses \eqref{eq:var_F_given_T}. Taking expectation over $T$ and using $\|T\|_{\mathrm{op}} = O_{\mathcal L}(1)$ yields \eqref{eq:fresh_rank_one_trace_Lp}. This proves the lemma.
\end{proof}

We next record two deterministic perturbation estimates for the Hessian operators in \eqref{eq:H_LOO_def}--\eqref{eq:M_i_def}. We identify $\R^{d\times d}$ with $\R^{d^2}$ throughout this discussion. If $A,B:\R^{d\times d}\to\R^{d\times d}$ are positive-definite self-adjoint operators, equivalently $d^2\times d^2$ symmetric matrices after vectorization, and $E:=A-B$, then the resolvent identity gives the following bound: whenever $\|A^{-1}\|_{\mathrm{op}},\|B^{-1}\|_{\mathrm{op}}\le \tau^{-1}$,
\begin{equation}\label{eq:resolvent_trace_bound}
    \left|
    \frac{1}{d^2}\tr(A^{-1})
    -
    \frac{1}{d^2}\tr(B^{-1})
    \right|
    \le
    \frac{1}{\tau^2 d^2}\,\|E\|_\ast
    =
    \frac{1}{\tau^2 d^2}\,\|A-B\|_\ast ,
\end{equation}
where $\|\cdot\|_\ast$ denotes the nuclear norm on this $d^2$-dimensional operator space.

For a fixed cavity index $h$, let $H^{\setminus h}$ be the Hessian in \eqref{eq:H_LOO_def}. Define the simplified Hessian $\widetilde H^{\setminus h}$ by replacing the block $M_i^{\setminus h}$ in \eqref{eq:M_i_def}, for every $i\ne h$, by
\begin{equation}\label{eq:M_i_simplified_def}
    \widetilde M_i^{\setminus h}
    :=
    b_i^{\setminus h}e_i e_i^\top
    +
    a_i^{\setminus h}\diag\bigl(\widetilde R_{:,i}^{\setminus h}\bigr),
\end{equation}
and keeping the $h$-th block equal to zero. Equivalently, $\widetilde H^{\setminus h}$ is obtained from $H^{\setminus h}$ by replacing $b_i^{\setminus h}Q_{:,i}^{\setminus h}(Q_{:,i}^{\setminus h})^\top$ by $b_i^{\setminus h}e_i e_i^\top$ in each active block.

\begin{lemma}[Trace reduction to the simplified Hessian]\label{lem:chi_rank_one_block_reduction}
Assume $\lambda n/d^2\ge \tau>0$ and $\max_i b_i^{\setminus h}\le B$. For $i\ne h$, write
\begin{equation}
    Q_{:,i}^{\setminus h}=-e_i+\rho_i^{\setminus h},
    \qquad
    p_i^{\setminus h}:=U\rho_i^{\setminus h}.
\end{equation}
Then
\begin{equation}\label{eq:rank_one_reduction_trace_bound}
    \left|
    \frac{1}{d^2}\tr\!\left((H^{\setminus h})^{-1}\right)
    -
    \frac{1}{d^2}\tr\!\left((\widetilde H^{\setminus h})^{-1}\right)
    \right|
    \le
    \frac{B}{\tau^2 d^2}
    \sum_{i\ne h}
    \|v_i\|^2
    \left(2\|u_i\|\|p_i^{\setminus h}\|+\|p_i^{\setminus h}\|^2\right).
\end{equation}
Consequently, if $n/d^2=O(1)$, $\max_i\{\|u_i\|,\|v_i\|\}=O(1)$, and $\max_{i\ne h}\|p_i^{\setminus h}\|\le \eta$, then the normalized inverse trace of $H^{\setminus h}$ differs from that of the simplified Hessian $\widetilde H^{\setminus h}$ by $O(\eta)$.
\end{lemma}

\begin{proof}
Since
\begin{equation}
    Q_{:,i}^{\setminus h}(Q_{:,i}^{\setminus h})^\top-e_i e_i^\top
    =
    -e_i(\rho_i^{\setminus h})^\top
    -
    \rho_i^{\setminus h}e_i^\top
    +
    \rho_i^{\setminus h}(\rho_i^{\setminus h})^\top,
\end{equation}
the $i$th summand in $H^{\setminus h}-\widetilde H^{\setminus h}$ is
\begin{equation}
    \Delta
    \mapsto
    b_i^{\setminus h}\left[
        -u_i(p_i^{\setminus h})^\top
        -
        p_i^{\setminus h}u_i^\top
        +
        p_i^{\setminus h}(p_i^{\setminus h})^\top
    \right]\Delta\,v_i v_i^\top .
\end{equation}
Each of the three displayed terms is rank one as an operator on $\R^{d\times d}$. For example,
\begin{equation}
    u_i(p_i^{\setminus h})^\top \Delta\,v_i v_i^\top
    =
    \bigl\langle p_i^{\setminus h} v_i^\top,\Delta\bigr\rangle\,u_i v_i^\top,
\end{equation}
and hence its nuclear norm is $\|u_i v_i^\top\|_\mathsf{F}\|p_i^{\setminus h} v_i^\top\|_\mathsf{F}=\|u_i\|\|p_i^{\setminus h}\|\|v_i\|^2$. The same bound holds for the transposed cross term, while the $p_i^{\setminus h}(p_i^{\setminus h})^\top$ term contributes $\|p_i^{\setminus h}\|^2\|v_i\|^2$. Thus
\begin{equation}
    \|H^{\setminus h}-\widetilde H^{\setminus h}\|_\ast
    \le
    B\sum_{i\ne h}
    \|v_i\|^2
    \left(2\|u_i\|\|p_i^{\setminus h}\|+\|p_i^{\setminus h}\|^2\right).
\end{equation}
Both $H^{\setminus h}$ and $\widetilde H^{\setminus h}$ are bounded below by the ridge term $(\lambda n/d^2)I$, so their inverse operator norms are at most $\tau^{-1}$. Applying \eqref{eq:resolvent_trace_bound} gives \eqref{eq:rank_one_reduction_trace_bound}.
\end{proof}

\begin{lemma}[Stability at the natural coefficient scale]\label{lem:chi_coefficient_stability}
Let $\widetilde H^{\setminus h}$ be the simplified Hessian above. Let $\bar H^{\setminus h}$ be the simplified Hessian obtained by replacing, for $i\ne h$, the coefficients $a_i^{\setminus h}$ and $\widetilde R_{j,i}^{\setminus h}$ by $\bar a_i$ and $\bar R_{j,i}$, while keeping the coefficients $b_i^{\setminus h}$ unchanged. Assume
\begin{equation}
    \frac{\lambda n}{d^2}\ge \tau>0,
    \qquad
    0\le a_i^{\setminus h},\bar a_i\le B,
    \qquad
    \bar R_{j,i}\ge 0,
\end{equation}
and
\begin{equation}\label{eq:coef_perturb_assumptions}
    \max_{i\ne h} |\bar a_i-a_i^{\setminus h}|\le \delta_a,
    \qquad
    \max_{i\ne h}\sum_{j\ne i,h}\widetilde R_{j,i}^{\setminus h}\le C_R,
    \qquad
    \max_{j\ne i,h}|\bar R_{j,i}-\widetilde R_{j,i}^{\setminus h}|
    \le
    \frac{\delta_R}{k}.
\end{equation}
Then
\begin{equation}\label{eq:coef_stability_trace_bound}
\begin{aligned}
    &\left|
    \frac{1}{d^2}\tr\!\left((\bar H^{\setminus h})^{-1}\right)
    -
    \frac{1}{d^2}\tr\!\left((\widetilde H^{\setminus h})^{-1}\right)
    \right|
    \\
    &\hspace{1cm}\le
    \frac{1}{\tau^2 d^2}
    \sum_{i\ne h}\|v_i\|^2
    \sum_{j\ne i,h}
    \left(
        \delta_a\widetilde R_{j,i}^{\setminus h}
        +
        \frac{B\delta_R}{k}
    \right)\|u_j\|^2 .
\end{aligned}
\end{equation}
Consequently, if $n/d^2=O(1)$, $n/k=O(1)$, $\max_i\{\|u_i\|,\|v_i\|\}=O(1)$, and $C_R=O(1)$, then the normalized inverse trace changes by
\begin{equation}
    O(\delta_a+\delta_R).
\end{equation}
\end{lemma}

\begin{proof}
Each elementary map
\begin{equation}
    \Delta
    \mapsto
    u_j u_j^\top\Delta\,v_i v_i^\top
    =
    \bigl\langle u_j v_i^\top,\Delta\bigr\rangle\,u_j v_i^\top
\end{equation}
has rank one and nuclear norm $\|u_j\|^2\|v_i\|^2$. Therefore
\begin{equation}\label{eq:coef_stability_nuclear}
    \|\bar H^{\setminus h}-\widetilde H^{\setminus h}\|_\ast
    \le
    \sum_{i\ne h}\sum_{j\ne i,h}
    |\bar a_i\bar R_{j,i}-a_i^{\setminus h}\widetilde R_{j,i}^{\setminus h}|\,
    \|u_j\|^2\|v_i\|^2 .
\end{equation}
The coefficient difference satisfies
\begin{equation}
    |\bar a_i\bar R_{j,i}-a_i^{\setminus h}\widetilde R_{j,i}^{\setminus h}|
    \le
    |\bar a_i-a_i^{\setminus h}|\,\widetilde R_{j,i}^{\setminus h}
    +
    \bar a_i\,|\bar R_{j,i}-\widetilde R_{j,i}^{\setminus h}|
    \le
    \delta_a\widetilde R_{j,i}^{\setminus h} + \frac{B\delta_R}{k}.
\end{equation}
Substituting this into \eqref{eq:coef_stability_nuclear}, and using that both Hessians are bounded below by the ridge term $(\lambda n/d^2)I$, \eqref{eq:resolvent_trace_bound} proves \eqref{eq:coef_stability_trace_bound}.
\end{proof}

\subsection*{Quadratic forms of the simplified Hessian}

We now use the simplified Hessian to extract the scalar equation for $\chi$. In this subsection we suppress the leave-one-out superscript and write
\begin{equation}
    f_{j i}:=\vecop(u_j v_i^\top),
    \qquad
    r_{j,i}:=\beta q_{j,i}(1-q_{j,i}),
    \qquad
    \widetilde R_{j,i}=\frac{r_{j,i}}{k}.
\end{equation}
The full simplified Hessian is
\begin{equation}\label{eq:full_simplified_hessian}
    \widetilde H
    =
    \frac{\lambda n}{d^2}I
    +
    \sum_{i=1}^n b_i f_{i i}f_{i i}^\top
    +
    \sum_{i=1}^n\sum_{j\ne i}
    \frac{a_i r_{j,i}}{k} f_{j i}f_{j i}^\top .
\end{equation}
Here $a_i,b_i,q_{j,i}$ are evaluated at the full minimizer. For the $i$-cavity problem, superscript $(i)$ denotes the corresponding quantities after removing sample $i$ from both tensor legs.

\paragraph{Feature spectral bounds.}
We use the following spectral estimates for the feature matrices: uniformly in $i$ and $\ell$,
\begin{equation}
\label{eq:feature_spectral_bound}
    \left\|\bigl[f_{\ell\ell}\bigr]_{\ell\ne i}\right\|_{\mathrm{op}}
    =
    \Op(1),
    \qquad
    \left\|\bigl[f_{j\ell}\bigr]_{j\ne \ell,i}\right\|_{\mathrm{op}}
    =
    \Op(\sqrt d).
\end{equation}

\begin{lemma}[LOO coefficient stability]\label{lem:loo_coefficient_stability_chi}
Assume \ref{post:loo-mu}--\ref{post:loo-bounded}. Then, uniformly in $i$,
\begin{equation}\label{eq:b_loo_stability_chi}
    \max_{\ell\ne i}|b_\ell-b_\ell^{(i)}|
    =
    \Op(d^{-1}),
\end{equation}
and
\begin{equation}\label{eq:ar_loo_stability_chi}
    \max_{\ell\ne i}\max_{j\ne \ell,i}
    \left|
    a_\ell r_{j,\ell}
    -
    a_\ell^{(i)}r_{j,\ell}^{(i)}
    \right|
    =
    \Op(d^{-1}).
\end{equation}
\end{lemma}

\begin{proof}
The functions $a=\sigma(-m)$, $b=a(1-a)$, and $r=\beta q(1-q)$ are uniformly Lipschitz on the bounded-score event supplied by \ref{post:loo-bounded}; the corresponding LOO scores are bounded by \ref{post:loo-score}. The LOO perturbations of the relevant scores and thresholds are $\Op(d^{-1})$ by \ref{post:loo-mu}--\ref{post:loo-score}, since $n\asymp d^2$ and the $\mu$ perturbation is smaller. Applying the Lipschitz bounds entrywise gives \eqref{eq:b_loo_stability_chi}--\eqref{eq:ar_loo_stability_chi}.
\end{proof}

\begin{lemma}[Diagonal quadratic form]\label{lem:diag_quadratic_chi}
Assume \ref{post:loo-mu}--\ref{post:trace-concentration} and the feature spectral bounds \eqref{eq:feature_spectral_bound}. Let $\chi_d$ be as in \eqref{eq:chi_d_def}. Then, uniformly in $i$,
\begin{equation}\label{eq:diag_quadratic_chi}
    f_{i i}^\top \widetilde H^{-1} f_{i i}
    =
    \frac{\chi_d}{1+b_i\chi_d}
    +
    \Op(d^{-1/2}).
\end{equation}
\end{lemma}

\begin{proof}
Let $H_i$ be the simplified $i$-cavity Hessian that keeps only feature directions $f_{j\ell}$ with $j,\ell\ne i$, and uses the $i$-cavity coefficients:
\begin{equation}
    H_i
    =
    \frac{\lambda n}{d^2}I
    +
    \sum_{\ell\ne i} b_\ell^{(i)} f_{\ell\ell}f_{\ell\ell}^\top
    +
    \sum_{\ell\ne i}\sum_{j\ne \ell,i}
    \frac{a_\ell^{(i)} r_{j,\ell}^{(i)}}{k}
    f_{j\ell}f_{j\ell}^\top .
\end{equation}
This operator is independent of $(u_i,v_i)$. Set
\begin{equation}
    A_i:=H_i+b_i f_{ii}f_{ii}^\top .
\end{equation}
We first compare $\widetilde H$ with $A_i$. Write
\begin{equation}
    \widetilde H=A_i+C_i^{\mathrm{feat}}+C_i^{\mathrm{coef}},
\end{equation}
where
\begin{equation}
    C_i^{\mathrm{feat}}
    =
    \sum_{\ell\ne i}\frac{a_\ell r_{i,\ell}}{k}f_{i\ell}f_{i\ell}^\top
    +
    \sum_{j\ne i}\frac{a_i r_{j,i}}{k}f_{j i}f_{j i}^\top
\end{equation}
contains the feature directions involving $u_i$ or $v_i$, except for the diagonal rank-one term, and
\begin{equation}
\begin{aligned}
    C_i^{\mathrm{coef}}
    &=
    \sum_{\ell\ne i}(b_\ell-b_\ell^{(i)})f_{\ell\ell}f_{\ell\ell}^\top
    \\
    &\quad+
    \sum_{\ell\ne i}\sum_{j\ne \ell,i}
    \frac{
    a_\ell r_{j,\ell}
    -
    a_\ell^{(i)}r_{j,\ell}^{(i)}
    }{k}
    f_{j\ell}f_{j\ell}^\top .
\end{aligned}
\end{equation}
The resolvent identity gives
\begin{equation}
    f_{ii}^\top \widetilde H^{-1}f_{ii}
    -
    f_{ii}^\top A_i^{-1}f_{ii}
    =
    -
    f_{ii}^\top \widetilde H^{-1}
    (C_i^{\mathrm{feat}}+C_i^{\mathrm{coef}})
    A_i^{-1}f_{ii}.
\end{equation}
By Woodbury,
\begin{equation}
    A_i^{-1}f_{ii}
    =
    \frac{H_i^{-1}f_{ii}}
    {1+b_i f_{ii}^\top H_i^{-1}f_{ii}},
\end{equation}
and the denominator is at least one. Thus it suffices to bound
\begin{equation}\label{eq:diag_cavity_error_target}
    f_{ii}^\top \widetilde H^{-1}
    (C_i^{\mathrm{feat}}+C_i^{\mathrm{coef}})
    H_i^{-1}f_{ii}.
\end{equation}
Let $\xi_i:=\widetilde H^{-1}f_{ii}$. The ridge lower bound gives $\|\xi_i\|=\Op(1)$. For the two terms in $C_i^{\mathrm{feat}}$, each summand has a factor of the form
\begin{equation}
    f_{i\ell}^\top H_i^{-1}f_{ii}
    \quad\text{or}\quad
    f_{j i}^\top H_i^{-1}f_{ii}.
\end{equation}
In the first expression $v_i$ is an unmatched Gaussian vector, and in the second expression $u_i$ is unmatched. Equivalently, if $M=\operatorname{mat}(H_i^{-1}f_{i\ell})$ or $M=\operatorname{mat}(H_i^{-1}f_{j i})$, then $\|M\|_\mathsf F=\Op(1)$ and the expression is $u_i^\top Mv_i$ with one Gaussian leg independent of $M$ and the other leg. Hence each such factor is $\Op(d^{-1/2})$. Since the coefficients have total scale $n/k=O(1)$, the contribution of $C_i^{\mathrm{feat}}$ to \eqref{eq:diag_cavity_error_target} is $\Op(d^{-1/2})$.

We next consider $C_i^{\mathrm{coef}}$. The diagonal coefficient part can be written as
\begin{equation}
    \bigl[f_{\ell\ell}\bigr]_{\ell\ne i}
    \diag(b_\ell-b_\ell^{(i)})_{\ell\ne i}
    \bigl[f_{\ell\ell}\bigr]_{\ell\ne i}^\top .
\end{equation}
By Lemma~\ref{lem:loo_coefficient_stability_chi} and \eqref{eq:feature_spectral_bound}, its operator norm is $\Op(d^{-1})$, so its contribution to \eqref{eq:diag_cavity_error_target} is $\Op(d^{-1})$.

For the off-diagonal coefficient part, fix $\ell\ne i$ and write
\begin{equation}
    F_\ell^{(i)}
    :=
    \bigl[f_{j\ell}\bigr]_{j\ne \ell,i},
    \qquad
    D_\ell^{(i)}
    :=
    \diag\!\left(a_\ell r_{j,\ell}-a_\ell^{(i)}r_{j,\ell}^{(i)}\right)_{j\ne \ell,i}.
\end{equation}
The corresponding block is $k^{-1}F_\ell^{(i)}D_\ell^{(i)}(F_\ell^{(i)})^\top$. Lemma~\ref{lem:loo_coefficient_stability_chi} gives $\|D_\ell^{(i)}\|_{\mathrm{op}}=\Op(d^{-1})$, while \eqref{eq:feature_spectral_bound} gives $\|F_\ell^{(i)}\|_{\mathrm{op}}=\Op(\sqrt d)$. Moreover,
\begin{equation}
    \left\|(F_\ell^{(i)})^\top H_i^{-1}f_{ii}\right\|
    =
    \Op(1),
\end{equation}
because $H_i$ and $F_\ell^{(i)}$ are independent of $f_{ii}$ and
\begin{equation}
    \frac{1}{d^2}\tr\!\left(
    H_i^{-1}F_\ell^{(i)}(F_\ell^{(i)})^\top H_i^{-1}
    \right)
    =
    \Op(1).
\end{equation}
Therefore the fixed-$\ell$ contribution is $\Op(1/(k\sqrt d))$, and summing over $\ell$ gives $\Op(d^{-1/2})$. We have proved
\begin{equation}\label{eq:diag_cavity_replacement}
    f_{ii}^\top \widetilde H^{-1}f_{ii}
    =
    f_{ii}^\top A_i^{-1}f_{ii}
    +
    \Op(d^{-1/2}).
\end{equation}

It remains to evaluate the right-hand side. Woodbury gives
\begin{equation}
    f_{ii}^\top A_i^{-1}f_{ii}
    =
    \frac{x_i}{1+b_i x_i},
    \qquad
    x_i:=f_{ii}^\top H_i^{-1}f_{ii}.
\end{equation}
Since $H_i$ is independent of $f_{ii}$, Lemma~\ref{lem:fresh_rank_one_trace}, \ref{post:trace-concentration}, and the finite-removal trace stability transferred to the simplified Hessian by the perturbation estimates above give
\begin{equation}
    x_i
    =
    \chi_d+\Op(d^{-1/2}).
\end{equation}
The map $x\mapsto x/(1+b_i x)$ is uniformly Lipschitz for $b_i\ge 0$, and \eqref{eq:diag_quadratic_chi} follows from \eqref{eq:diag_cavity_replacement}.
\end{proof}

\begin{lemma}[Off-diagonal quadratic form]\label{lem:offdiag_quadratic_chi}
Under the assumptions of Lemma~\ref{lem:diag_quadratic_chi}, uniformly for $j\ne i$,
\begin{equation}\label{eq:offdiag_quadratic_chi}
    f_{j i}^\top \widetilde H^{-1}f_{j i}
    =
    \chi_d+\Op(d^{-1/2}).
\end{equation}
\end{lemma}

\begin{proof}
The proof is the same as that of Lemma~\ref{lem:diag_quadratic_chi}. One removes all terms involving $u_j$ or $v_i$ to create a cavity independent of $f_{j i}$. The direct Woodbury correction along $f_{j i}$ has coefficient $a_i r_{j,i}/k=O(k^{-1})$ and is therefore absorbed into the $\Op(d^{-1/2})$ error. The feature-removal and coefficient-removal errors are controlled by the same unmatched-Gaussian and feature-matrix estimates used above.
\end{proof}

\begin{remark}[Use for $\chi$]\label{rem:chi_use}
In the leave-one-out analysis, the operator $T$ will be a Hessian inverse, for example $T = (H^{\setminus i})^{-1}$, evaluated on a sigma-field that is independent of a fresh pair $(u_i,v_i)$. Lemma~\ref{lem:fresh_rank_one_trace} then gives
\begin{equation}
    \bigl\langle u_i v_i^\top,\,(H^{\setminus i})^{-1}[u_i v_i^\top]\bigr\rangle
    =
    \frac{1}{d^2}\tr\!\left((H^{\setminus i})^{-1}\right)
    +
    O_{\mathcal L}(d^{-1/2}),
\end{equation}
provided $\|(H^{\setminus i})^{-1}\|_{\mathrm{op}} = O_{\mathcal L}(1)$. Thus the susceptibility can be characterized through the normalized trace of the inverse cavity Hessian. Lemmas~\ref{lem:chi_rank_one_block_reduction} and~\ref{lem:chi_coefficient_stability} then justify replacing the original Hessian by simpler approximations without changing this normalized trace at leading order.
\end{remark}

\section{Variational formulation of the scalar equations}
\label{app:scalar_variational}

This appendix verifies the variational representation used in \sref{tam_selfconsistent}. Throughout, $Z$ and $G$ denote independent standard normal variables, and $\nu,\chi>0$. 

\subsection{The scalar TAM functional}

Recall
\begin{equation}
    c_r(\nu)
    =
    \min_{\mu\in\R}
    \left\{
    \mu+\frac{1}{r}\Eb{Z}{\phi_\beta(\nu Z-\mu)}
    \right\}.
\end{equation}
The minimizer, denoted $\mu_r(\nu)$, is characterized by
\begin{equation}\label{eq:app_mu_r_nu}
    r
    =
    \Eb{Z}{\sigma\!\bigl(\beta(\nu Z-\mu_r(\nu))\bigr)}.
\end{equation}
\begin{equation}\label{eq:c_r_derivative_raw}
    c_r'(\nu)
    =
    \frac{1}{r}
    \Eb{Z}{Z\,\phi_\beta'\!\bigl(\nu Z-\mu_r(\nu)\bigr)}.
\end{equation}
Stein's identity gives
\begin{equation}\label{eq:c_r_derivative}
    c_r'(\nu)
    =
    \frac{\nu}{r}
    \Eb{Z}{\phi_\beta''\!\bigl(\nu Z-\mu_r(\nu)\bigr)}.
\end{equation}
We will also use that $c_r$ is convex in $\nu$. Indeed, the function
\begin{equation}
    (\nu,\mu)\mapsto
    \mu+\frac{1}{r}\Eb{Z}{\phi_\beta(\nu Z-\mu)}
\end{equation}
is jointly convex, and partial minimization over $\mu$ preserves convexity.

\subsection{Moreau envelope identities}

Let $S=S_{\nu,\chi}(G)$ be the minimizer in \eqref{eq:scalar_moreau_def}, and set
\begin{equation}\label{eq:app_A_B_def}
    A
    :=
    \sigma\!\bigl(-S+c_r(\nu)\bigr),
    \qquad
    B:=A(1-A).
\end{equation}
The first-order condition in $s$ is
\begin{equation}\label{eq:app_scalar_channel}
    S
    =
    \nu G+\chi A.
\end{equation}
Envelope differentiation in $\chi$ gives
\begin{equation}\label{eq:M_chi_derivative}
    \partial_\chi M_r(\nu,\chi)
    =
    -\frac{1}{2}\E[A^2].
\end{equation}
For the derivative in $\nu$, the envelope theorem gives
\begin{equation}\label{eq:M_nu_derivative_raw}
    \partial_\nu M_r(\nu,\chi)
    =
    -\E[AG]+\E[A]c_r'(\nu).
\end{equation}
Differentiating \eqref{eq:app_scalar_channel} with respect to the standard Gaussian coordinate $G$ yields
\begin{equation}
    \frac{dA}{dG}
    =
    -\frac{\nu B}{1+\chi B}.
\end{equation}
Therefore Stein's identity gives
\begin{equation}\label{eq:AG_stein}
    \E[AG]
    =
    -\nu\,\E\!\left[\frac{B}{1+\chi B}\right].
\end{equation}
Combining \eqref{eq:c_r_derivative}, \eqref{eq:M_nu_derivative_raw}, and \eqref{eq:AG_stein},
\begin{equation}\label{eq:M_nu_derivative}
    \partial_\nu M_r(\nu,\chi)
    =
    \nu\,\E\!\left[\frac{B}{1+\chi B}\right]
    +
    \frac{\nu}{r}\E[A]\,
    \Eb{Z}{\phi_\beta''\!\bigl(\nu Z-\mu_r(\nu)\bigr)}.
\end{equation}

\subsection{Saddle-point equations}

The saddle potential is
\begin{equation}
    \mathcal P_{\alpha,r}(\nu,\chi)
    =
    \frac{\lambda}{2}\nu^2
    -
    \frac{\nu^2}{2\alpha\chi}
    +
    M_r(\nu,\chi).
\end{equation}
Using \eqref{eq:M_chi_derivative},
\begin{equation}\label{eq:P_chi_derivative}
    \partial_\chi \mathcal P_{\alpha,r}(\nu,\chi)
    =
    \frac{\nu^2}{2\alpha\chi^2}
    -
    \frac{1}{2}\E[A^2].
\end{equation}
Hence the $\chi$ saddle-point equation is
\begin{equation}\label{eq:app_chi_saddle_first}
    \nu^2
    =
    \alpha\chi^2\E[A^2].
\end{equation}
Using \eqref{eq:M_nu_derivative},
\begin{equation}\label{eq:P_nu_derivative}
    \partial_\nu \mathcal P_{\alpha,r}(\nu,\chi)
    =
    \lambda\nu
    -
    \frac{\nu}{\alpha\chi}
    +
    \nu\,\E\!\left[\frac{B}{1+\chi B}\right]
    +
    \frac{\nu}{r}\E[A]\,
    \Eb{Z}{\phi_\beta''\!\bigl(\nu Z-\mu_r(\nu)\bigr)}.
\end{equation}
Since $\nu>0$, the $\nu$ saddle-point equation is
\begin{equation}\label{eq:app_nu_saddle_chi_fp}
    \frac{1}{\alpha}
    =
    \lambda\chi
    +
    \chi\,\E\!\left[\frac{B}{1+\chi B}\right]
    +
    \frac{\chi}{r}\E[A]\,
    \Eb{Z}{\phi_\beta''\!\bigl(\nu Z-\mu_r(\nu)\bigr)}.
\end{equation}
This is exactly the $\chi$-equation obtained from the zero-residual limit of \eqref{eq:finite_d_map_residual}.

It remains to connect \eqref{eq:app_chi_saddle_first} with the variance equation. From \eqref{eq:app_scalar_channel},
\begin{equation}
    \E[AS]
    =
    \nu\E[AG]+\chi\E[A^2].
\end{equation}
Using \eqref{eq:AG_stein} and \eqref{eq:app_chi_saddle_first},
\begin{equation}
    \E[AS]
    =
    -\nu^2\E\!\left[\frac{B}{1+\chi B}\right]
    +
    \frac{\nu^2}{\alpha\chi}.
\end{equation}
Substituting \eqref{eq:app_nu_saddle_chi_fp} into the last display gives
\begin{equation}
    \E[AS]
    =
    \nu^2
    \left[
    \lambda+\frac{1}{r}\E[A]\,
    \Eb{Z}{\phi_\beta''\!\bigl(\nu Z-\mu_r(\nu)\bigr)}
    \right],
\end{equation}
which is the corresponding $\nu$-equation.

\subsection{Maximization in \texorpdfstring{$\chi$}{chi} and convexity in \texorpdfstring{$\nu$}{nu}}

For fixed $\nu$, the inner problem is concave after the change of variables $t=1/\chi$. Indeed,
\begin{equation}
    \mathcal P_{\alpha,r}(\nu,1/t)
    =
    \frac{\lambda}{2}\nu^2
    -
    \frac{\nu^2}{2\alpha}t
    +
    \E
    \min_{s\in\R}
    \left\{
    \frac{t}{2}(s-\nu G)^2
    +
    \ell(s-c_r(\nu))
    \right\}.
\end{equation}
The last term is the infimum of affine functions of $t$, and hence is concave in $t$. Thus the inner extremization may equivalently be viewed as a one-dimensional concave maximization in $1/\chi$.

For fixed $\nu$, differentiating $A=\sigma(-S+c_r(\nu))$ with respect to $\chi$ gives
\begin{equation}
    \partial_\chi A
    =
    -\frac{AB}{1+\chi B}.
\end{equation}
Consequently,
\begin{equation}\label{eq:chi_square_A_monotone}
    \frac{d}{d\chi}
    \left(\chi^2\E[A^2]\right)
    =
    2\chi\,\E\!\left[\frac{A^2}{1+\chi B}\right]
    >
    0.
\end{equation}
Moreover $\chi^2\E[A^2]\to0$ as $\chi\downarrow0$, while $\chi^2\E[A^2]\to\infty$ as $\chi\to\infty$. By \eqref{eq:P_chi_derivative}, $\partial_\chi\mathcal P$ changes sign exactly once, from positive to negative. Hence $\chi\mapsto\mathcal P_{\alpha,r}(\nu,\chi)$ has a unique global maximizer.

Finally we prove strong convexity of the reduced potential $\mathcal V_{\alpha,r}$ defined in \eqref{eq:reduced_potential_def}. The logistic loss has the Fenchel representation
\begin{equation}\label{eq:logistic_fenchel}
    \ell(t)
    =
    \sup_{a\in[0,1]}
    \left\{
    -at-I(a)
    \right\},
    \qquad
    I(a):=a\log a+(1-a)\log(1-a),
\end{equation}
with the convention $0\log0=0$. Applying this pointwise in the Moreau envelope and minimizing over $s$ gives
\begin{equation}\label{eq:moreau_dual}
    M_r(\nu,\chi)
    =
    \sup_{a(\cdot)\in[0,1]}
    \E\left[
    a(G)\bigl(c_r(\nu)-\nu G\bigr)
    -
    I(a(G))
    -
    \frac{\chi}{2}a(G)^2
    \right].
\end{equation}
Taking the supremum over $\chi>0$ yields
\begin{equation}\label{eq:reduced_potential_dual}
    \mathcal V_{\alpha,r}(\nu)
    =
    \frac{\lambda}{2}\nu^2
    +
    \sup_{a(\cdot)\in[0,1]}
    \left\{
    \E\left[
    a(G)\bigl(c_r(\nu)-\nu G\bigr)-I(a(G))
    \right]
    -
    \frac{\nu}{\sqrt{\alpha}}
    \left(\E[a(G)^2]\right)^{1/2}
    \right\}.
\end{equation}
For each fixed measurable $a(\cdot)$, the expression inside the supremum is a convex function of $\nu$, because $c_r$ is convex and $\E[a(G)]\ge0$; adding $\lambda\nu^2/2$ makes it $\lambda$-strongly convex. The supremum of $\lambda$-strongly convex functions is $\lambda$-strongly convex. This proves the convexity and uniqueness assertions in Proposition~\ref{prop:scalar_variational}.

\begin{remark}[Meaning of the dual optimizer]
The optimizer $a(G)$ in \eqref{eq:moreau_dual} is the scalar limit of the sample-wise logistic weights. Indeed, equality in the Fenchel representation \eqref{eq:logistic_fenchel} is attained at
\begin{equation}
    a=\sigma(-t).
\end{equation}
Taking $t=S-c_r(\nu)$, where $S=S_{\nu,\chi}(G)$ is the proximal minimizer, gives
\begin{equation}
    a_\ast(G)
    =
    \sigma\!\bigl(-S_{\nu,\chi}(G)+c_r(\nu)\bigr).
\end{equation}
Equivalently, the proximal first-order condition gives
\begin{equation}
    S_{\nu,\chi}(G)
    =
    \nu G+\chi a_\ast(G).
\end{equation}
Thus the dual optimizer is the limiting counterpart of the finite-dimensional coefficients $a_i=\sigma(-m_i)$.
\end{remark}

\end{document}